%% file: cellwise_coda_theory.tex
\documentclass[12pt]{article}
\pdfoutput=1 

\usepackage{amsmath,amssymb,amsthm,mathtools}
\usepackage{bm}

\usepackage{graphicx}
\usepackage{booktabs}
\usepackage{multirow}

\usepackage{natbib}
\usepackage{url}
\usepackage[colorlinks=true,citecolor=blue,linkcolor=blue,urlcolor=blue]{hyperref}
\usepackage{cleveref}

\usepackage{algorithm,algorithmicx,algpseudocode}

\usepackage{enumitem}
\usepackage{xcolor}

\newcommand{\blind}{0}
\newif\ifsubmit\submittrue

\makeatletter
\@ifundefined{keyword}{%
  \newenvironment{keyword}{\medskip\noindent\textit{Keywords:}\enspace\ignorespaces}{\par}%
  \newcommand{\sep}{;\enspace}%
}{}
\@ifundefined{AMS}{%
  \newenvironment{AMS}{\medskip\noindent\textit{MSC 2020:}\enspace\ignorespaces}{\par}%
}{}
\makeatother

\theoremstyle{plain}
\newtheorem{theorem}{Theorem}[section]
\newtheorem{proposition}[theorem]{Proposition}
\newtheorem{lemma}[theorem]{Lemma}
\newtheorem{corollary}[theorem]{Corollary}

\theoremstyle{definition}
\newtheorem{definition}[theorem]{Definition}
\newtheorem{example}[theorem]{Example}

\theoremstyle{remark}
\newtheorem{remark}[theorem]{Remark}

\newcommand{\R}{\mathbb{R}}

\newcommand{\SD}{\mathcal{S}^{D}}

\newcommand{\bv}{\mathbf{v}}

\newcommand{\AlgRef}{the companion paper of \citet{Templ2026cellPcaCoDa}}

\def\spacingset#1{\renewcommand{\baselinestretch}%
{#1}\small\normalsize}

\begin{document}
\spacingset{1}

\if1\blind
{
  \title{\bf Log-Ratio Propagation on the Simplex: \\
        A Theory of Cellwise Contamination for\\
        Compositional Data}
  \author{Author Information Withheld for Review}
  \date{}
  \maketitle
} \fi

\if0\blind
{
  \bigskip
  \bigskip
  \bigskip
  \begin{center}
    {\LARGE\bf Log-Ratio Propagation on the Simplex:\\[3pt]
     A Theory of Cellwise Contamination for\\[3pt]
     Compositional Data}
  \end{center}
  \medskip
  \begin{center}
    Matthias Templ\footnote{%
      School of Business, FHNW Fachhochschule Nordwestschweiz,
      Riggenbachstrasse 16, 4600 Olten, Switzerland.
      E-mail: \texttt{matthias.templ@fhnw.ch}.
      The author declares no conflicts of interest.  Claude Opus
      (Anthropic) was used for drafting assistance and copy-editing;
      all mathematical statements, proofs, simulations, and final
      text were produced and verified by the author, who takes full
      responsibility for the content.}\\
    School of Business, FHNW Fachhochschule Nordwestschweiz\\
    Olten, Switzerland
  \end{center}
  \medskip
} \fi

\begin{abstract}
Compositional data must be analysed through log-ratios: scale
invariance, the defining axiom of the field, leaves no alternative.
The centred log-ratio divides by the geometric mean of every part,
so a single contaminated component shifts every centred-log-ratio
coordinate at once, displacing the log-ratio vector by a fixed amount
that no choice of coordinates can reduce.  We develop a theory of cellwise contamination
on the simplex around this observation.  A scale-invariant
contamination model built from multiplicative perturbation combines
with a propagation theorem showing that corruption of a single raw
part induces a rank-one shift of the log-ratio vector, with
direction determined by the contrast matrix.  The resulting
perturbation pattern is \emph{not} equivalent to any independent
cellwise contamination model in log-ratio coordinates --- so
standard Euclidean cellwise methods applied to log-ratios are
ill-posed under the simplex contamination mechanism.  For
estimators whose Euclidean cellwise breakdown is witnessed by a
column-concentrated configuration --- a class including MCD, $S$-,
$\tau$-, and coordinate-wise $M$-estimators of location and scatter
--- the cellwise breakdown value on the simplex is reduced by the
factor $(D{-}1)/D$ relative to its Euclidean counterpart, a
reduction that is tight and that arises purely from the
normalisation mismatch between $nD$ raw cells and $n(D{-}1)$ ilr
cells.  The cellwise influence function for the variation matrix
carries a diagnostic fingerprint: contamination of a single part
inflates exactly one row and column, identifying the responsible
component.  These results form the theoretical foundation for
cellwise-robust methods on the simplex; a companion paper develops
a cellwise-robust PCA estimator that exploits the propagation
geometry and demonstrates it on simulated and geochemical data.
\end{abstract}

\begin{AMS}
62F35, 62H25, 62G35
\end{AMS}

\section{Introduction}
\label{sec:introduction}

Consider a geochemical survey in which one element --- say lead ---
is measured too high in a single soil sample, perhaps through
instrument drift or a detection-limit artefact.  On the raw data
this is one corrupted cell.  After the log-ratio transformation that
compositional analysis requires, that single corrupted cell spreads:
the centred log-ratio (clr) measures each part against the geometric
mean of the full row, so one erroneous part shifts that geometric mean
and, with it, every clr coordinate at once.  The consequences at realistic dimensions are striking.  For a
composition of $D = 50$ elements, contaminating just $2\%$ of cells
at random means that $1-(1-0.02)^{50} \approx 64\%$ of all rows
contain at least one corrupted cell; at $5\%$ cellwise contamination
the figure exceeds $92\%$.  Classical rowwise-robust methods, which
require a majority of clean rows, are overwhelmed.  Yet the
contamination is genuinely sparse --- on average each row has only
one or two bad cells --- and should, in principle, be correctable.

This intellectual tension --- sparse contamination in the parts that
\emph{looks} dense after the log-ratio transformation required by
scale invariance --- is the starting point of the present paper.
Compositional data arise whenever measurements represent parts of a
whole: chemical concentrations in geochemistry, species abundances in
ecology, household expenditure shares in economics, and taxon
proportions in microbiome studies \citep{Gloor2017}.  The sample space for $D$-part
compositions is the simplex $\SD$, and the Aitchison geometry
\citep{Aitchison1986} provides the natural framework for their
analysis.  The centered log-ratio (clr) and isometric log-ratio (ilr)
transformations \citep{Egozcue2003} map compositions to coordinates in
$\R^{D-1}$, enabling the application of standard multivariate methods
after transformation \citep{Filzmoser2018}.

\paragraph{Principles of compositional analysis (and why closure is
not the villain).}
Two features of compositional data analysis (CoDa) shape everything
that follows and deserve to be made explicit before we describe the
contamination problem.

\emph{(a) Scale invariance.}  A compositional observation is an
equivalence class of positive vectors under positive scalar
multiplication: $\mathbf{x}$ and $\lambda\mathbf{x}$ ($\lambda > 0$)
carry identical information.  Any valid CoDa statistic must return the
same value on both.  It is a classical result \citep{Aitchison1986}
that scale-invariant functions of a positive vector factor through
log-ratios: every valid coordinate system is built from ratios
$x_{i}/x_{j}$ or their logarithms.  This equivalence-class structure
is the basis of the modern axiomatic development of the simplex sample
space \citep{EgozcuePawlowskyGlahn2019}.  Recent reappraisals and
proposed revisions of the log-ratio programme
\citep{Greenacre2023,ZhangSchluter2024} weigh log-ratio coordinates
against simpler or hybrid alternatives such as amalgamations and
pairwise ratios; the sparse-to-dense propagation studied here is a
property of \emph{any} log-ratio coordinate system and is orthogonal to
that debate.

\emph{(b) Closure is a representative, not a statistical operation.}
Rescaling $\mathbf{x}$ to sum to a fixed constant~$\kappa$ (``closure'')
selects a canonical representative of the equivalence class.  It
cannot by itself generate, amplify, or mask any CoDa-relevant
phenomenon, because every log-ratio is unchanged by it.  When
researchers refer to ``the effect of closure'' on robustness, they are
usually describing an effect of the log-ratio transformation --- which
is a genuine mathematical operation with a genuine inverse and a
genuine Jacobian --- but misattributing it to closure.

These two observations dictate how a contamination framework on the
simplex must be built.  Any perturbation of the raw parts is observed
through log-ratios, and log-ratios mix parts through the geometric-mean
denominator present in clr and ilr.  Concretely, $\operatorname{clr}_{k}
(\mathbf{x}) = \ln(x_{k}) - \tfrac{1}{D}\sum_{l}\ln(x_{l})$; if a
single component $x_{j}$ is multiplied by $\delta_{j}>0$, then
\begin{equation}
\label{eq:clr-shift-intro}
  \operatorname{clr}_{k}(\tilde{\mathbf{x}})
    - \operatorname{clr}_{k}(\mathbf{x})
  \;=\;
  \begin{cases}
    \bigl(1 - \tfrac{1}{D}\bigr)\ln\delta_{j}, & k = j,\\[2pt]
    -\tfrac{1}{D}\,\ln\delta_{j}, & k \neq j,
  \end{cases}
\end{equation}
so a single contaminated part perturbs every clr coordinate; carried
to an orthonormal (ilr) system the same shift becomes a displacement
of fixed magnitude $\sqrt{(D-1)/D}\,|\ln\delta_{j}|$ along a known
direction, which the basis may redistribute among coordinates but
cannot remove (Section~\ref{sec:propagation}).  The shift is rank-one
in the row index and carries a known direction --- a structural
feature that this paper will exploit.  The density-from-sparsity
phenomenon of the preceding paragraphs arises
through~\eqref{eq:clr-shift-intro}, is independent of whether closure
is applied, and --- for the cellwise problem, where many parts are
corrupted and their identities unknown, so that no fixed log-ratio
basis can localise the damage --- is \emph{intrinsic} to CoDa rather
than a removable coordinate artefact.  A cellwise-robust CoDa method must, accordingly,
reason about contamination in the parts but detect and downweight in
the log-ratio coordinates, using the specific propagation geometry
that the log-ratio transformation imposes.

Classical robust statistics assumes that entire observations (rows of
the data matrix) are either clean or outlying --- the
\emph{rowwise contamination model} of \citet{Huber1964}.
High-breakdown estimators such as the minimum covariance determinant
\citep{Rousseeuw1985} achieve breakdown values up to $1/2$ under this
model.  However, in modern high-dimensional datasets, contamination
often acts at the level of individual cells rather than entire rows:
instrument drift may corrupt a single chemical element, sensor failure
may affect one variable, or data entry errors may alter one field.
\citet{Alqallaf2009} formalised the \emph{cellwise contamination model}
in $\R^p$, and subsequent work has developed cellwise-robust methods
including the DetectDeviatingCells algorithm
\citep{Rousseeuw2018}, the cellwise minimum covariance determinant
\citep{Raymaekers2024}, and cellwise-robust PCA
\citep{Hubert2019}.

Several threads of existing work touch the intersection of robust
compositional analysis and cellwise robust statistics, but none
formalises the contamination mechanism at the level of raw parts.
Early robust principal components for compositional data
\citep{Filzmoser2009} treated rowwise outliers in ilr coordinates
without addressing the cellwise level.
\citet{FilzmoserGregorich2020} subsequently provided an overview of
outlier types in compositional data, distinguishing global, local, and
cellwise outliers, without developing a formal cellwise
contamination model for the simplex.  \citet{Walach2020} proposed a
detection method based on pairwise log-ratios but provided no
formal contamination model.  \citet{Stefelova2021} developed a
two-step cellwise-outlier-robust regression approach;
\citet{Rieser2023} proposed cellwise-robust covariance estimation
for compositions by applying a coordinate-wise cellwise method
directly in ilr coordinates.  Closest in spirit on the propagation
side is \citet{Mert2016}, who established --- by a linearised
error-propagation argument that, as they note, is equivalent to the
(familiar) linearity of the ilr map on the Aitchison geometry
\citep{Egozcue2003} --- that ilr coordinates are additive under
multiplicative perturbations of the parts, and verified by simulation
that imprecision in one part perturbs several ilr coordinates while the
first pivot coordinate stays nearly stable as the dimension grows.
Their analysis is confined to small, symmetric, random imprecision
rather than gross outliers and to the pivot coordinate system, treats
detection-limit censoring as a separate non-linear effect, and stops
short of the exact rank-one shift identity for an arbitrary contrast
matrix, of the non-equivalence with independent cellwise contamination,
and of the breakdown and influence-function analysis that the
gross-outlier regime requires; it is framed as a question of
measurement quality rather than of cellwise robustness.  We make these
connections precise in Remarks~\ref{rem:scope-censoring}
and~\ref{rem:mert-stability}.  No prior work
formalises cellwise contamination at the level of raw parts, proves
the resulting shift in log-ratio coordinates is non-equivalent to
independent cellwise contamination, or establishes the breakdown
value of compositional estimators under a cellwise model.  The
present paper addresses these questions: What is the correct
cellwise contamination model on the simplex?  How does
contamination of raw parts propagate through the log-ratio
transformation?  What are the theoretical limits of robust
estimation under cellwise contamination on the simplex?

Meanwhile, cellwise-robust methodology in unconstrained $\R^p$ has
advanced rapidly.  MacroPCA \citep{Hubert2019} handles missing values
alongside cellwise and rowwise outliers but relies on Euclidean
distances and does not exploit the log-ratio propagation structure
specific to compositions.  The recent
cellPCA of \citet{CentofantiHubertRousseeuw2024} extends cellwise
robustness to functional and high-dimensional data; the SCRAMBLE
method of \citet{SCRAMBLE2025} combines cellwise robustness with a
sparsity penalty in the PCA objective via Riemannian stochastic
gradient descent.  \citet{RaymaekersRousseeuw2024b} survey the
theoretical challenges of cellwise outliers, and
\citet{CentofantiHubertRousseeuw2025} propose cellRCov for
cellwise-robust covariance in high dimensions.
\citet{Agostinelli2015} developed a two-step estimator of
multivariate location and scatter under joint cellwise and casewise
contamination, building on the \citet{Alqallaf2009} cellwise model;
this groundwork is used in the breakdown analysis of
Section~\ref{sec:breakdown}.  Beyond PCA and covariance,
cellwise-robust estimation has developed in parallel for
\emph{precision} matrices and Gaussian graphical models
\citep{TarrMullerWeber2015,LohTan2018,KatayamaFujisawaDrton2018}, for
regression \citep{LeungZhangZamar2016}, and for outlier detection in
heterogeneous populations \citep{Zaccaria2024}, with
\citet{LohTan2018} providing consistency and breakdown analysis under
$\varepsilon$-cellwise corruption; \citet{PacreauLounici2023} treat
cellwise outliers jointly with missing values; and filter-then-estimate
schemes detect suspect cells --- for instance, by statistical depth ---
and treat them as missing before robust estimation
\citep{SaracenoAgostinelli2021}, the two-step compositional approach
of \citet{Stefelova2021} being in this spirit.  These Euclidean
methods do not exploit the rank-one propagation structure that arises
specifically on the simplex.  Closest in spirit is the pairwise log-ratio
detection of \citet{Walach2020}, which implicitly uses a related
geometric observation --- a single contaminated part affects only
those pairwise log-ratios involving it --- but does not formalise
the shift as a rank-one ilr perturbation with known direction.

A related natural question is whether Euclidean cellwise methods
(e.g., cellPCA of \citet{CentofantiHubertRousseeuw2024}, SCRAMBLE of
\citet{SCRAMBLE2025}, or the cellMCD-on-ilr surrogate of
\citet{Rieser2023}) can simply be applied to ilr-transformed data.
Each such method implicitly assumes independent cellwise
contamination in ilr coordinates.
Theorem~\ref{thm:non-equivalence} shows that this assumption is
violated under the simplex-level cellwise mechanism: raw-part
contamination induces a coupled rank-one shift along the direction
$\bv_j$ rather than independent per-coordinate perturbations.  The
consequence is efficiency loss and part-level misidentification
rather than outright failure --- such methods still produce useful
estimates in practice, as the geochemical application of
\citet{Rieser2023} demonstrates --- but the distinctive propagation geometry of the
simplex is not being exploited.  A cellwise-robust PCA estimator
that does exploit this geometry for targeted detection and
log-ratio-aware weighting is developed in a companion paper
\citep{Templ2026cellPcaCoDa}.

This paper addresses these questions.  Our contributions are as follows.

\begin{enumerate}[label=(\roman*),leftmargin=2em]
\item \textbf{Contamination model and propagation theory}
  (Sections~\ref{sec:contamination-model}--\ref{sec:propagation}).
  Corruption of a single raw part induces a rank-one shift of the ilr
  vector --- of basis-independent magnitude
  $\sqrt{(D-1)/D}\,|\ln\delta|$ and direction fixed by the contrast
  matrix (Theorem~\ref{thm:ilr-propagation}).  The underlying cellwise
  contamination model on~$\SD$ is multiplicative
  (Definition~\ref{def:cellwise-model}); scale invariance makes its
  simplex and $\mathbb{R}_{>0}^{D}$ formulations interchangeable for
  all log-ratio statistics.  The \emph{sparse-to-dense} geometry that
  motivates the paper is a direct consequence of this theorem.  Taken
  in isolation the rank-one shift is elementary --- it restates the
  linearity of the ilr map noted by \citet{Mert2016} --- but it is the
  structural lever for the non-equivalence, breakdown, and
  influence-function results below, which carry the paper's substantive
  novelty.

\item \textbf{Non-equivalence of cellwise contamination on the simplex
  and in coordinates} (Section~\ref{sec:propagation}).
  Cellwise contamination on the simplex is \emph{not} equivalent to
  independent cellwise contamination in log-ratio coordinates
  (Theorem~\ref{thm:non-equivalence}).  The closest prior method for
  cellwise-robust CoDa covariance \citep{Rieser2023} transforms to
  ilr and applies cellwise methods directly; our non-equivalence
  result shows this is ill-posed under the raw-part contamination
  model, providing the theoretical justification for
  simplex-specific methods.

\item \textbf{Breakdown and influence functions}
  (Sections~\ref{sec:breakdown}--\ref{sec:influence-functions}).
  The cellwise breakdown value on~$\SD$ is reduced by the factor
  $(D{-}1)/D$ relative to its Euclidean counterpart
  (Theorem~\ref{thm:breakdown-reduction}) for every estimator whose
  Euclidean bound is witnessed by a column-concentrated configuration
  --- a class that includes MCD, $S$-, $\tau$-, and coordinate-wise
  $M$-estimators of location and scatter.  The bound is tight
  (Theorem~\ref{thm:tightness}), and the reduction is a normalisation
  artefact: log-ratio propagation enables each raw cell to reproduce
  the damage of one ilr cell in the witnessing column, and the
  $nD$-versus-$n(D{-}1)$ cell-count mismatch supplies the factor.
  We introduce the cellwise influence function (CIF) on the simplex
  and show that the variation-matrix CIF has a
  \emph{one-row-one-column (rank-two) sparsity pattern}: contamination
  of a single part inflates exactly one row and one column of the CIF,
  identifying the responsible component.

\end{enumerate}

The methodological consequences of these theoretical results --- a
cellwise-robust PCA estimator exploiting the propagation geometry ---
are developed in a companion paper \citep{Templ2026cellPcaCoDa},
which reports simulation evidence and a geochemical application.
The present paper is concerned only with the theoretical foundations.

The remainder of the paper is organised as follows.
Section~\ref{sec:contamination-model} develops the cellwise
contamination model on the simplex.
Section~\ref{sec:propagation} establishes the propagation theorem
and the non-equivalence result.
Section~\ref{sec:breakdown} derives breakdown values, including the
$(D{-}1)/D$ reduction theorem, its tightness, and explicit bounds for
MCD and the cellwise-robust PCA estimator proposed in
\citet{Templ2026cellPcaCoDa}.
Section~\ref{sec:influence-functions} develops the cellwise
influence function.
Section~\ref{sec:discussion} concludes with a discussion of the
theory's assumptions and limitations.  Additional theoretical results (alr propagation, partial
breakdown analysis with maximum-bias curves, extended CIF
derivations) are provided in the Supplementary Material.

\input{theory/contamination_model}

\input{theory/propagation_theorem}

\input{theory/breakdown_value}

\input{theory/influence_functions}

\section{Discussion}
\label{sec:discussion}

One contaminated part displaces the whole log-ratio vector by a fixed
amount that no choice of coordinates can shrink --- only redistribute.  That irreducible
displacement is the difficulty of cellwise robustness on the simplex
and drives the results of this paper.  Scale invariance forces
analysis through log-ratios, whose geometric-mean denominator mixes
every part; the propagation theorem
(Theorem~\ref{thm:ilr-propagation}) shows the resulting shift is
rank-one with known direction.  In the centred log-ratio it taints all
$D$ coordinates at once; in an orthonormal coordinate system it may be
spread across all $D-1$ or, by a basis that pivots on the contaminated
part, concentrated in one (Remark~\ref{rem:spread-basis}) --- but under
cellwise contamination, with many parts corrupted and their identities
unknown, no fixed basis can localise it, so sparse contamination of raw
parts becomes dense and \emph{structured} contamination of the
coordinates.  The non-equivalence theorem
(Theorem~\ref{thm:non-equivalence}) makes the central obstruction
precise: the induced ilr perturbation is \emph{mutually singular}
with every independent cellwise model on $\mathbb{R}^{D-1}$, placing
probability one on the measure-zero set fixed by the deterministic
loading ratio $v_{j l_{1}}/v_{j l_{2}}$, so transforming to ilr and
applying a Euclidean cellwise method is ill-posed rather than merely
inefficient.  The cellwise influence function for the variation
matrix carries a diagnostic fingerprint: contamination of a single
part inflates exactly one row and one column of the CIF, identifying
the responsible component.

The breakdown analysis (Theorem~\ref{thm:breakdown-reduction})
delivers a sharp, tight characterisation of the simplex-vs-Euclidean
gap: for the class of estimators whose Euclidean cellwise
breakdown is witnessed by a column-concentrated configuration
(including MCD, $S$-, $\tau$-, and coordinate-wise
$M$-estimators of location and scatter), the cellwise breakdown
value on~$\SD$ is exactly $(D{-}1)/D$ of its Euclidean counterpart,
with equality achievable by explicit construction: the matching
lower bound holds because log-ratio propagation inflates neither the
per-column contamination count nor the number of affected rows --- the
two quantities that govern breakdown for coordinate-wise and
affine-equivariant estimators, respectively.  The reduction
is a cell-count normalisation: each raw cell reproduces the damage
of one ilr cell in the witnessing column via the rank-one shift
$\bv_{j}$, and the $nD$-versus-$n(D{-}1)$ denominator mismatch
supplies the factor.  Estimators attacked through
column-concentrated configurations cannot be protected by working
on the simplex --- but neither can they be broken worse.
The cellwise influence function analysis
(Section~\ref{sec:influence-functions}) supplies a principled basis
for weight functions in cellwise-robust estimators on the simplex:
the one-row-one-column CIF fingerprint identifies the responsible raw
part, and the contrast-matrix weights
\citep[cf.\ the log-ratio-aware weight mapping of][]{Templ2026cellPcaCoDa}
translate part-level flags into coordinate-level weights.  The same
analysis draws a sharp line between estimators
(Section~\ref{ssec:comparison}): classical ilr location, covariance,
and PCA have \emph{unbounded} cellwise influence, whereas the
bounded-$\psi$ functional is cellwise B-robust
(Theorem~\ref{thm:cif-robust-covariance}), so a single contaminated
part can move the classical estimate without limit but perturbs the
robust one only by a bounded amount.
A cellwise-robust PCA estimator
built on these theoretical foundations, together with simulation
evidence and a geochemical application, is developed in a companion
paper \citep{Templ2026cellPcaCoDa}.

\subsection*{Assumptions and limitations}

\begin{enumerate}[label=(\alph*)]
\item \textbf{Independence of contamination across parts.}
  The cellwise contamination model
  (Definition~\ref{def:cellwise-model}) assumes that the
  contamination indicators $B_1,\ldots,B_D$ are independent across
  parts.  The propagation theorem
  (Theorem~\ref{thm:ilr-propagation}) and the breakdown results
  (Theorem~\ref{thm:breakdown-reduction}) do not require
  independence --- they hold for any contamination pattern.
  Independence becomes load-bearing only when detection procedures
  treat contamination of different parts as producing
  distinguishable, separately testable signals along the directions
  $\mathbf{d}_j$; correlated contamination (chemically related
  elements in geochemistry, phylogenetically related taxa in
  microbiomes) is not covered by the detection theory developed here.

\item \textbf{The $G(\{1\}) < 1$ requirement.}
  Results involving detection or identification of contaminated
  cells rely on the contamination magnitude distribution~$G$
  satisfying $G(\{1\}) < 1$ --- that is, that contaminated cells are
  not trivially equal to the clean values.  The purely structural
  propagation and breakdown results do not require this; it enters
  only when detection is asked of an estimator.  Contamination with
  $|\ln\delta|$ below the sample-size-specific detection threshold
  is indistinguishable from sampling variation, an inherent
  limitation of any thresholding procedure.

\item \textbf{High-dimensional regime.}
  Our theoretical results assume $D$ fixed and $n \to \infty$.  For
  high-throughput compositional data such as microbiome studies
  with $D$ in the hundreds or thousands, several difficulties
  arise: the Bonferroni-corrected detection threshold grows as
  $\sqrt{2\ln D}$, reducing power; the weighted covariance matrix
  has $O(D^2)$ entries to estimate, requiring strong
  regularisation; and the $D$ propagation directions $\bv_j$ are
  increasingly non-orthogonal (the pairwise inner products
  $\bv_{j}^{\top}\bv_{k}$ decay only as $O(1/D)$), creating cross-talk
  that degrades detection specificity.  The present results are
  stated for fixed $D$ and do not, as given, transfer to the regime
  $D = D_n \to \infty$.

\item \textbf{Beyond the simplex.}
  The log-ratio propagation mechanism is specific to
  scale-invariant analyses on the simplex; the present results do
  not extend to other constrained spaces (correlation matrices,
  stochastic matrices, sum-to-zero contrasts).
\end{enumerate}


\section*{Data Availability Statement}
No empirical data are used in this paper.  The methodological
counterpart \citep{Templ2026cellPcaCoDa} contains the empirical
evaluation and associated data-availability statements.  The
theoretical results are self-contained; no data archive is
associated with this submission.

\section*{Disclosure Statement}
The author declares no conflicts of interest.  Claude Opus
(Anthropic) was used for drafting assistance on the abstract,
introduction, and portions of the theoretical exposition, and for
copy-editing throughout.  All mathematical statements, proofs,
simulations, and final text were produced and verified by the
author, who takes full responsibility for the content.

\bibliographystyle{agsm}
\bibliography{refs}

\end{document}

%% file: theory/contamination_model.tex

\section{Cellwise contamination on the simplex}
\label{sec:contamination-model}

Throughout this section we fix an integer $D \geq 2$ and write
$[D] = \{1,\ldots,D\}$.  Vectors are typeset in bold lowercase
($\mathbf{x}$), matrices in bold uppercase ($\mathbf{V}$), and the
positive reals as $\mathbb{R}_{>0}$.

\subsection{The simplex and log-ratio transformations}
\label{ssec:simplex-logratio}

\begin{definition}[Simplex and closure]
\label{def:simplex}
The \emph{$D$-part simplex} is the sample space of compositional data,
\begin{equation}
\label{eq:simplex}
  \mathcal{S}^{D}
  \;=\;
  \bigl\{\,\mathbf{x} = (x_{1},\ldots,x_{D})^{\top} \in \mathbb{R}_{>0}^{D}
         \;:\; x_{1}+\cdots+x_{D} = \kappa
  \,\bigr\},
\end{equation}
where $\kappa > 0$ is a fixed constant (typically $\kappa = 1$ for
proportions or $\kappa = 10^{6}$ for parts per million).  The
\emph{closure operator} $\mathcal{C}\colon \mathbb{R}_{>0}^{D} \to
\mathcal{S}^{D}$ is defined by
\begin{equation}
\label{eq:closure}
  \mathcal{C}(\mathbf{z})
  \;=\;
  \Bigl(
    \frac{\kappa\, z_{1}}{\sum_{k=1}^{D} z_{k}},\;\ldots,\;
    \frac{\kappa\, z_{D}}{\sum_{k=1}^{D} z_{k}}
  \Bigr)^{\!\top}.
\end{equation}
\end{definition}

The simplex, equipped with the perturbation operation
$\mathbf{x} \oplus \mathbf{y} = \mathcal{C}(x_{1}y_{1},\ldots,x_{D}y_{D})$
and powering $\alpha \odot \mathbf{x} = \mathcal{C}(x_{1}^{\alpha},\ldots,x_{D}^{\alpha})$,
forms a $(D{-}1)$-dimensional real vector space known as the
\emph{Aitchison geometry} \citep{Aitchison1986,PawlowskyGlahn2015}.
The following two isometric isomorphisms map $(\mathcal{S}^{D},\oplus,\odot)$
to $(\mathbb{R}^{D-1},+,\cdot\,)$.

\begin{definition}[Centered log-ratio transformation]
\label{def:clr}
For $\mathbf{x} \in \mathcal{S}^{D}$, the \emph{centered log-ratio}
(clr) transformation is the mapping
$\operatorname{clr}\colon \mathcal{S}^{D} \to \mathbb{R}^{D}$ given by
\begin{equation}
\label{eq:clr}
  \operatorname{clr}(\mathbf{x})
  \;=\;
  \Bigl(
    \ln\frac{x_{1}}{g(\mathbf{x})},\;\ldots,\;
    \ln\frac{x_{D}}{g(\mathbf{x})}
  \Bigr)^{\!\top},
\end{equation}
where $g(\mathbf{x}) = \bigl(\prod_{k=1}^{D} x_{k}\bigr)^{1/D}$ is
the geometric mean.  By construction,
$\sum_{j=1}^{D}\operatorname{clr}_{j}(\mathbf{x}) = 0$, so the image
lies in the hyperplane
$\mathcal{H}_{0} = \{\mathbf{u}\in\mathbb{R}^{D} : \mathbf{1}_{D}^{\top}\mathbf{u}=0\}$.
\end{definition}

\begin{definition}[Isometric log-ratio transformation]
\label{def:ilr}
Let $\mathbf{V}\in\mathbb{R}^{D\times(D-1)}$ be a \emph{contrast matrix}
satisfying
\begin{equation}
\label{eq:contrast-matrix}
  \mathbf{V}^{\top}\mathbf{V} = \mathbf{I}_{D-1},
  \qquad
  \mathbf{V}\mathbf{V}^{\top}
    = \mathbf{I}_{D} - \tfrac{1}{D}\,\mathbf{1}_{D}\mathbf{1}_{D}^{\top}
    \eqqcolon \mathbf{H}_{D}.
\end{equation}
The columns of $\mathbf{V}$ form an orthonormal basis of
$\mathcal{H}_{0}$.  The \emph{isometric log-ratio} (ilr) transformation
is
\begin{equation}
\label{eq:ilr}
  \operatorname{ilr}(\mathbf{x})
  \;=\;
  \mathbf{V}^{\top}\operatorname{clr}(\mathbf{x})
  \;\in\;\mathbb{R}^{D-1}.
\end{equation}
Since $\mathbf{V}$ has orthonormal columns, the map
$\operatorname{ilr}\colon(\mathcal{S}^{D},d_{A})\to(\mathbb{R}^{D-1},\|\cdot\|_{2})$
is an isometry, where $d_{A}$ denotes the Aitchison distance
\begin{equation}
\label{eq:aitchison-distance}
  d_{A}(\mathbf{x},\mathbf{y})
  \;=\;
  \Bigl[\,
    \sum_{j=1}^{D}
      \Bigl(\ln\frac{x_{j}}{g(\mathbf{x})}
            - \ln\frac{y_{j}}{g(\mathbf{y})}\Bigr)^{\!2}
  \,\Bigr]^{1/2}.
\end{equation}
\end{definition}

\begin{remark}[Choice of contrast matrix]
\label{rem:contrast-matrix}
A common choice is the Helmert sub-matrix with rows
\begin{equation}
\label{eq:helmert-basis}
  v_{jl}
  \;=\;
  \begin{cases}
    \displaystyle \frac{1}{\sqrt{l(l+1)}} & \text{if } j \leq l, \\[6pt]
    \displaystyle -\frac{l}{\sqrt{l(l+1)}} & \text{if } j = l+1, \\[4pt]
    0 & \text{if } j > l+1,
  \end{cases}
  \qquad l = 1,\ldots,D{-}1.
\end{equation}
All theoretical results below hold for an arbitrary contrast matrix
satisfying~\eqref{eq:contrast-matrix}.  The relationship between clr
and ilr coordinates is invertible on $\mathcal{H}_{0}$:
$\operatorname{clr}(\mathbf{x}) = \mathbf{V}\operatorname{ilr}(\mathbf{x})$
and
$\operatorname{ilr}(\mathbf{x}) = \mathbf{V}^{\top}\operatorname{clr}(\mathbf{x})$.
\end{remark}

\subsection{Cellwise contamination model}
\label{ssec:cellwise-model}

Classical robust statistics operates under the assumption that entire
observations (rows) are either clean or outlying.  In many applications
involving compositional data --- geochemistry, metabolomics,
microbiome studies --- individual parts of a composition may be
corrupted while the remaining parts are correctly measured.  This
motivates a cellwise contamination framework on~$\mathcal{S}^{D}$.

\begin{definition}[Cellwise contamination on $\mathcal{S}^{D}$]
\label{def:cellwise-model}
Let $\mathbf{x} = (x_{1},\ldots,x_{D})^{\top} \in \mathcal{S}^{D}$ be
a clean composition drawn from a distribution~$F$ on $\mathcal{S}^{D}$.
The \emph{cellwise contamination model with parameters
$(\varepsilon, G)$} generates an observed composition
$\tilde{\mathbf{x}} \in \mathcal{S}^{D}$ as follows.
\begin{enumerate}
\item[\textup{(C1)}]
  For each part $j \in [D]$, let $B_{j} \sim \operatorname{Bernoulli}(\varepsilon)$
  be independent indicator variables, with contamination probability
  $\varepsilon \in [0,1)$.
\item[\textup{(C2)}]
  For each $j$ with $B_{j}=1$, draw a multiplicative contamination
  factor $\delta_{j} \in \mathbb{R}_{>0}$ from a distribution~$G$
  on $\mathbb{R}_{>0}$ independently of all other quantities.
  Set $\delta_{j}=1$ when $B_{j}=0$.
\item[\textup{(C3)}]
  Form the contaminated raw vector
  $\tilde{\mathbf{x}}^{\,\mathrm{raw}}
    = (x_{1}\delta_{1},\;\ldots,\;x_{D}\delta_{D})^{\top}
    \in \mathbb{R}_{>0}^{D}$.
\item[\textup{(C4)}]
  The observed composition is obtained by re-closure:
  \begin{equation}
  \label{eq:observed-composition}
    \tilde{\mathbf{x}}
    \;=\;
    \mathcal{C}\bigl(\tilde{\mathbf{x}}^{\,\mathrm{raw}}\bigr)
    \;=\;
    \mathcal{C}(x_{1}\delta_{1},\;\ldots,\;x_{D}\delta_{D}).
  \end{equation}
\end{enumerate}
\end{definition}

The use of multiplicative contamination in step~(C2) is the natural
choice on the simplex: multiplication in $\mathbb{R}_{>0}^{D}$ followed
by closure corresponds to perturbation in the Aitchison geometry.
In~(C3)--(C4) the observed composition is equivalently expressed as the
Aitchison perturbation
$\tilde{\mathbf{x}} = \mathbf{x} \oplus \mathcal{C}(\boldsymbol{\delta})$,
where $\boldsymbol{\delta} = (\delta_{1},\ldots,\delta_{D})^{\top}$.

\begin{remark}[Closure is not load-bearing]
\label{rem:closure-not-load-bearing}
Compositional data are scale-invariance equivalence classes: any CoDa
statistic returns the same value on $\mathbf{x}$ and $\lambda
\mathbf{x}$ for $\lambda > 0$, because every scale-invariant function
of a positive vector factors through log-ratios
\citep{Aitchison1986,EgozcuePawlowskyGlahn2019}.  Re-closure in
step~(C4) is therefore a choice
of representative, not a statistical operation.  Equivalently, the
model can be stated on $\mathbb{R}_{>0}^{D}$ by omitting~(C4): every
log-ratio, and therefore every clr and ilr coordinate, is identical
in the two formulations.  All propagation, breakdown, and influence-
function results below are statements about log-ratio coordinates and
are unaffected by whether closure is applied.  The propagation is
driven by the log-ratio transformation's geometric-mean denominator,
not by closure.
\end{remark}

\begin{remark}[Scope: random contamination versus systematic censoring]
\label{rem:scope-censoring}
The model of Definition~\ref{def:cellwise-model} describes
\emph{random} corruption: each part is hit independently with
probability~$\varepsilon$ and, when hit, is multiplied by a factor
drawn from~$G$ irrespective of its value.  It does not cover
\emph{systematic} data limitations such as values reported below a
detection limit, which are conventionally replaced by a fixed fraction
of that limit and therefore push the small values of a part downward in
a deterministic, value-dependent way
\citep{MartinFernandez2003,MartinFernandez2012}.  Such rounded-zero and
left-censoring problems have a substantial literature of their own and
call for imputation rather than downweighting; \citet{Mert2016} treat
them separately, by simulation, precisely because the linear
propagation theory of Section~\ref{sec:propagation} does not apply to a
thresholding operation.  They are outside the scope of the present
paper.
\end{remark}

\subsection{Log-ratio propagation of cellwise contamination}
\label{ssec:logratio-propagation}

The fundamental complication of cellwise contamination on the simplex,
compared with the unconstrained setting of $\mathbb{R}^{D-1}$, is
\emph{not} the re-closure step~\eqref{eq:observed-composition} ---
closure is a choice of representative for a scale-invariance
equivalence class and has no effect on any log-ratio.  The
complication is that CoDa analysis requires scale-invariant
statistics, which factor through log-ratios; and the centred log-ratio
measures every part against the geometric mean of the whole row, which
involves every part.  A single contaminated part therefore shifts that
geometric mean by $\delta_{j}^{1/D}$ and propagates into every clr
coordinate --- and, with a basis-dependent spread, into the ilr
coordinates.  The following proposition records the effect at the
level of the closed representative; the subsequent corollary exposes
the log-ratio shift, which is the statistically meaningful object.

\begin{proposition}[Single-cell contamination and log-ratio propagation]
\label{prop:logratio-propagation}
Let $\mathbf{x} \in \mathcal{S}^{D}$ with sum
$S = \sum_{k=1}^{D}x_{k} = \kappa$, and suppose that exactly one
part~$j$ is contaminated by factor $\delta_{j} > 0$, with
$\delta_{k}=1$ for all $k \neq j$.  Then the observed composition is
\begin{equation}
\label{eq:single-cell-contaminated}
  \tilde{\mathbf{x}}
  \;=\;
  \frac{\kappa}{\kappa + x_{j}(\delta_{j}-1)}\,
  \bigl(x_{1},\;\ldots,\;x_{j-1},\;
        x_{j}\,\delta_{j},\;
        x_{j+1},\;\ldots,\;x_{D}\bigr)^{\!\top}.
\end{equation}
In particular, for every $k \neq j$,
\begin{equation}
\label{eq:non-contaminated-perturbation}
  \tilde{x}_{k}
  \;=\;
  \frac{\kappa}{\kappa + x_{j}(\delta_{j}-1)}\;x_{k}
  \;=\;
  \frac{x_{k}}{1 + (x_{j}/\kappa)(\delta_{j}-1)},
\end{equation}
and the contaminated part itself becomes
\begin{equation}
\label{eq:contaminated-part}
  \tilde{x}_{j}
  \;=\;
  \frac{x_{j}\,\delta_{j}}{1 + (x_{j}/\kappa)(\delta_{j}-1)}.
\end{equation}
\end{proposition}

\begin{proof}
The raw contaminated vector is
$\tilde{\mathbf{x}}^{\,\mathrm{raw}} = (x_{1},\ldots,x_{j}\delta_{j},\ldots,x_{D})^{\top}$
with component sum
\[
  \tilde{S}
  \;=\;
  \sum_{k \neq j} x_{k} \;+\; x_{j}\delta_{j}
  \;=\;
  \kappa - x_{j} + x_{j}\delta_{j}
  \;=\;
  \kappa + x_{j}(\delta_{j}-1).
\]
Applying the closure operator~\eqref{eq:closure},
\[
  \tilde{x}_{k}
  = \frac{\kappa\,\tilde{x}_{k}^{\,\mathrm{raw}}}{\tilde{S}}
  = \frac{\kappa\, x_{k}}{\kappa + x_{j}(\delta_{j}-1)}
  \qquad\text{for } k \neq j,
\]
and
\[
  \tilde{x}_{j}
  = \frac{\kappa\, x_{j}\delta_{j}}{\kappa + x_{j}(\delta_{j}-1)}.
\]
Dividing numerator and denominator by $\kappa$ yields the equivalent
forms in~\eqref{eq:non-contaminated-perturbation}
and~\eqref{eq:contaminated-part}.
\end{proof}

\begin{corollary}[Contamination in clr coordinates]
\label{cor:clr-propagation}
Under the hypotheses of Proposition~\ref{prop:logratio-propagation}, the
contamination-induced shift in clr coordinates is
\begin{equation}
\label{eq:clr-shift}
  \operatorname{clr}(\tilde{\mathbf{x}}) - \operatorname{clr}(\mathbf{x})
  \;=\;
  \operatorname{clr}\!\bigl(\mathcal{C}(\boldsymbol{\delta})\bigr)
  \;=\;
  \mathbf{H}_{D}\,\ln\boldsymbol{\delta},
\end{equation}
where $\ln\boldsymbol{\delta} = (\ln\delta_{1},\ldots,\ln\delta_{D})^{\top}$
and $\mathbf{H}_{D} = \mathbf{I}_{D} - D^{-1}\mathbf{1}_{D}\mathbf{1}_{D}^{\top}$
is the centering matrix.  With only part~$j$ contaminated
($\delta_{k}=1$ for $k\neq j$),
\begin{equation}
\label{eq:clr-shift-single}
  \operatorname{clr}_{k}(\tilde{\mathbf{x}}) - \operatorname{clr}_{k}(\mathbf{x})
  \;=\;
  \begin{cases}
    \displaystyle \bigl(1 - \tfrac{1}{D}\bigr)\ln\delta_{j}
      & \text{if } k = j, \\[6pt]
    \displaystyle -\,\tfrac{1}{D}\,\ln\delta_{j}
      & \text{if } k \neq j.
  \end{cases}
\end{equation}
Consequently, contamination of a single part shifts \emph{every} clr
coordinate, and the shift in ilr coordinates is
\begin{equation}
\label{eq:ilr-shift-single}
  \operatorname{ilr}(\tilde{\mathbf{x}}) - \operatorname{ilr}(\mathbf{x})
  \;=\;
  (\ln\delta_{j})\;\mathbf{v}_{j},
\end{equation}
where $\mathbf{v}_{j}^{\top}$ is the $j$th row of the
contrast matrix~$\mathbf{V}$ (equivalently,
$\mathbf{v}_{j} = \mathbf{V}^{\top}\mathbf{e}_{j}\in\mathbb{R}^{D-1}$).
\end{corollary}

\begin{proof}
Since $\tilde{\mathbf{x}} = \mathbf{x} \oplus \mathcal{C}(\boldsymbol{\delta})$
and the clr transformation is an isomorphism of
$(\mathcal{S}^{D},\oplus)$ onto $(\mathcal{H}_{0},+)$, we have
$\operatorname{clr}(\tilde{\mathbf{x}})
  = \operatorname{clr}(\mathbf{x})
    + \operatorname{clr}\!\bigl(\mathcal{C}(\boldsymbol{\delta})\bigr)$.
Now,
\[
  \operatorname{clr}_{k}\!\bigl(\mathcal{C}(\boldsymbol{\delta})\bigr)
  = \ln\delta_{k}
    - \frac{1}{D}\sum_{l=1}^{D}\ln\delta_{l}
  = \bigl(\mathbf{H}_{D}\,\ln\boldsymbol{\delta}\bigr)_{k},
\]
which establishes~\eqref{eq:clr-shift}.  Specialising to
$\ln\boldsymbol{\delta} = (\ln\delta_{j})\,\mathbf{e}_{j}$
gives~\eqref{eq:clr-shift-single}.  Pre-multiplying by
$\mathbf{V}^{\top}$ and using
$\mathbf{V}^{\top}\mathbf{H}_{D} = \mathbf{V}^{\top}$
yields~\eqref{eq:ilr-shift-single}, since
$\mathbf{V}^{\top}\mathbf{e}_{j} = \mathbf{v}_{j}$.
\end{proof}

\begin{example}[Worked example: $D=4$ geochemical composition]
\label{ex:D4-propagation}
Consider a four-part composition representing major oxides in a rock
sample:
$\mathbf{x} = (0.40,\; 0.30,\; 0.20,\; 0.10)^{\top}$
(SiO$_2$, Al$_2$O$_3$, Fe$_2$O$_3$, MgO), with $\kappa = 1$.
Suppose that part~$2$ (Al$_2$O$_3$) is contaminated by the
multiplicative factor $\delta_2 = 3$ (e.g., due to instrumental
drift), while all other parts are correctly measured.

\emph{Step~1: Raw contaminated vector.}
\[
  \tilde{\mathbf{x}}^{\,\mathrm{raw}}
    = (0.40,\; 0.30 \times 3,\; 0.20,\; 0.10)
    = (0.40,\; 0.90,\; 0.20,\; 0.10),
\]
with component sum $\tilde{S} = 1.60$.

\emph{Step~2: Re-closure.}
\[
  \tilde{\mathbf{x}}
    = \mathcal{C}(\tilde{\mathbf{x}}^{\,\mathrm{raw}})
    = (0.250,\; 0.5625,\; 0.125,\; 0.0625).
\]
Although only part~$2$ was contaminated, every part has changed:
SiO$_2$ dropped from $0.40$ to $0.25$, Fe$_2$O$_3$ from $0.20$ to
$0.125$, and MgO from $0.10$ to $0.0625$.

\emph{Step~3: Clr shift.}
By~\eqref{eq:clr-shift-single} with $D = 4$ and
$\ln\delta_2 = \ln 3 \approx 1.099$, the clr shift is
\[
  \operatorname{clr}(\tilde{\mathbf{x}}) - \operatorname{clr}(\mathbf{x})
  \;=\;
  \ln 3 \cdot
  \bigl(-\tfrac{1}{4},\;\tfrac{3}{4},\;-\tfrac{1}{4},\;-\tfrac{1}{4}\bigr)^{\!\top}
  \;\approx\;
  (-0.275,\;+0.824,\;-0.275,\;-0.275)^{\top}.
\]
The contaminated part~$2$ receives a large positive shift while
every other part receives the same negative shift of magnitude
$\ln 3 / 4 \approx 0.275$, matching the
prediction~\eqref{eq:clr-shift-single}.  In ilr coordinates the shift
is the rank-$1$ vector
$(\ln 3)\,\mathbf{v}_{2} \in \mathbb{R}^{3}$, confirming
Corollary~\ref{cor:clr-propagation}.
\end{example}

\begin{remark}[Cellwise on the simplex vs.\ cellwise in coordinates]
\label{rem:simplex-vs-coordinates}
In unconstrained $\mathbb{R}^{p}$, cellwise contamination of entry~$j$
leaves all other entries unaffected; the contamination acts on a single
coordinate axis.  On the simplex, by contrast,
Proposition~\ref{prop:logratio-propagation} and
Corollary~\ref{cor:clr-propagation} show that contamination of a single
raw part alters \emph{every} component of the composition and \emph{every}
log-ratio coordinate.  The magnitude of the propagation effect on the
uncontaminated parts depends on two quantities:
\begin{enumerate}
\item[\textup{(i)}]
  the relative abundance $x_{j}/\kappa$ of the contaminated part: the
  larger the original proportion, the greater the distortion of
  uncontaminated parts (cf.~\eqref{eq:non-contaminated-perturbation});
\item[\textup{(ii)}]
  the dimension~$D$: in clr coordinates the leakage to each uncontaminated
  coordinate is of order $D^{-1}\ln\delta_{j}$
  (cf.~\eqref{eq:clr-shift-single}), which decreases per coordinate
  but accumulates across the $(D{-}1)$ affected coordinates.
\end{enumerate}
This ``log-ratio propagation'' is the fundamental reason that cellwise
robust methods developed for unconstrained data cannot be applied
na\"{\i}vely to compositions: the pattern of outlying cells in
$\mathbb{R}^{D-1}$ coordinates is dense even when the contamination
is sparse on the raw parts.
\end{remark}

\begin{proposition}[Multi-cell contamination shift]
\label{prop:multi-cell}
Let $J \subseteq [D]$ denote the (random) set of contaminated parts,
with $|J| \sim \operatorname{Binomial}(D,\varepsilon)$ under
Definition~\ref{def:cellwise-model}.  The contamination-induced shift in
ilr coordinates is
\begin{equation}
\label{eq:ilr-shift-multi}
  \operatorname{ilr}(\tilde{\mathbf{x}}) - \operatorname{ilr}(\mathbf{x})
  \;=\;
  \sum_{j \in J} (\ln\delta_{j})\;\mathbf{v}_{j}
  \;=\;
  \mathbf{V}^{\top} \mathbf{H}_{D}\,
  \operatorname{diag}(\mathbf{b})\,\ln\boldsymbol{\delta},
\end{equation}
where $\mathbf{b} = (B_{1},\ldots,B_{D})^{\top}$ is the vector of
contamination indicators.  In particular, with $J$ and
$\boldsymbol{\delta}$ independent, the covariance of the shift
conditional on $J$ satisfies
\begin{equation}
\label{eq:shift-covariance}
  \operatorname{Var}\!\bigl[\operatorname{ilr}(\tilde{\mathbf{x}})
    - \operatorname{ilr}(\mathbf{x})
    \,\big|\, J\bigr]
  \;=\;
  \sigma_{\ln\delta}^{2}\;
  \sum_{j\in J}\mathbf{v}_{j}\mathbf{v}_{j}^{\top},
\end{equation}
where $\sigma_{\ln\delta}^{2} = \operatorname{Var}_{G}[\ln\delta_{j}]$,
provided $G$ has finite second moment on the log scale.  When the
contaminated parts are selected uniformly at random and $|J|=m$, the
expected shift covariance is
\begin{equation}
\label{eq:expected-shift-covariance}
  \mathbb{E}\biggl[
    \frac{1}{m}\sum_{j\in J}\mathbf{v}_{j}\mathbf{v}_{j}^{\top}
  \biggr]
  \;=\;
  \frac{1}{D}\;\mathbf{V}^{\top}\mathbf{V}
  \;=\;
  \frac{1}{D}\;\mathbf{I}_{D-1}.
\end{equation}
\end{proposition}

\begin{proof}
By the additivity of clr, the shift in clr coordinates under general
$\boldsymbol{\delta}$ is
$\operatorname{clr}(\tilde{\mathbf{x}}) - \operatorname{clr}(\mathbf{x})
  = \mathbf{H}_{D}\,\ln\boldsymbol{\delta}$.
Setting $\delta_{j} = 1$ (i.e., $\ln\delta_{j}=0$) for $j\notin J$
and pre-multiplying by $\mathbf{V}^{\top}$
gives~\eqref{eq:ilr-shift-multi}.

For the covariance, note that conditional on $J$ the random variables
$\{\ln\delta_{j}\}_{j\in J}$ are i.i.d.\ with mean $\mu_{\ln\delta}$
and variance $\sigma_{\ln\delta}^{2}$.  Thus
\[
  \operatorname{Var}\biggl[\sum_{j\in J}(\ln\delta_{j})\,\mathbf{v}_{j}
    \;\bigg|\; J\biggr]
  = \sigma_{\ln\delta}^{2}\sum_{j\in J}\mathbf{v}_{j}\mathbf{v}_{j}^{\top},
\]
which yields~\eqref{eq:shift-covariance}.  For the expectation, when
$J$ is a uniformly random subset of $[D]$ of size $m$, each index
$j$ is included with marginal probability $m/D$, so
\[
  \mathbb{E}\biggl[\sum_{j\in J}\mathbf{v}_{j}\mathbf{v}_{j}^{\top}\biggr]
  = \frac{m}{D}\sum_{j=1}^{D}\mathbf{v}_{j}\mathbf{v}_{j}^{\top}
  = \frac{m}{D}\,\mathbf{V}^{\top}\mathbf{V}
  = \frac{m}{D}\,\mathbf{I}_{D-1},
\]
where we used
$\sum_{j=1}^{D}\mathbf{v}_{j}\mathbf{v}_{j}^{\top} = \mathbf{V}^{\top}\mathbf{V}
= \mathbf{I}_{D-1}$.  Dividing by $m$
gives~\eqref{eq:expected-shift-covariance}.
\end{proof}

\begin{remark}[Relation to error-propagation variance formulas]
\label{rem:error-propagation-variance}
Equation~\eqref{eq:shift-covariance} is the cellwise-contamination
counterpart of the classical error-propagation variance decomposition
for an ilr coordinate \citep{Fiserova2011}, which \citet{Mert2016} use
to quantify the effect of measurement imprecision: there, additive
random errors of a fixed, small variance in \emph{every} part inflate
$\operatorname{Var}(\operatorname{ilr}_{l})$ through a linear
combination of pairwise log-ratio variances; here, a
$\operatorname{Bernoulli}(\varepsilon)$ \emph{subset} of parts receives
a multiplicative shift, and the induced second moment is the
rank-deficient matrix
$\sigma_{\ln\delta}^{2}\sum_{j\in J}\mathbf{v}_{j}\mathbf{v}_{j}^{\top}$,
concentrated on the propagation directions $\{\mathbf{v}_{j}\}_{j\in J}$.
The gap between an everywhere-small-error regime and a
sparse-but-large-shift regime is the gap between an efficiency question
and a breakdown question; the latter is taken up in
Section~\ref{sec:breakdown}.
\end{remark}

\subsection{Comparison with rowwise contamination}
\label{ssec:rowwise-comparison}

For context we recall the classical rowwise contamination model and
contrast it with the cellwise framework introduced above.

\begin{definition}[Rowwise $\varepsilon$-contamination on $\mathcal{S}^{D}$]
\label{def:rowwise-model}
Following \citet{Huber1964}, the \emph{rowwise
$\varepsilon$-contamination neighbourhood} of a distribution~$F$ on
$\mathcal{S}^{D}$ is
\begin{equation}
\label{eq:rowwise-neighbourhood}
  \mathcal{F}_{\varepsilon}^{\,\mathrm{row}}(F)
  \;=\;
  \bigl\{\,(1-\varepsilon)\,F + \varepsilon\,H
         \;:\; H \text{ is any distribution on } \mathcal{S}^{D}
  \,\bigr\}.
\end{equation}
Under this model, a fraction $\varepsilon$ of observations are entirely
replaced by draws from the contaminating distribution~$H$; the
remaining fraction $1-\varepsilon$ are drawn from the nominal model~$F$.
\end{definition}

\begin{remark}[Rowwise vs.\ cellwise on the simplex]
\label{rem:rowwise-vs-cellwise}
The rowwise model corrupts entire observations; the cellwise model
corrupts individual parts.  In unconstrained $\mathbb{R}^{p}$ the two
differ only in sparsity (one coordinate vs.\ all).  On the simplex the
distinction is sharper: scale invariance forces log-ratio coordinates
whose geometric-mean denominator mixes every part, so cellwise
contamination of a single raw part induces a \emph{dense} perturbation
of the centred log-ratio --- every clr coordinate moves
(Corollary~\ref{cor:clr-propagation}) --- by a fixed amount that no
orthonormal coordinate choice can remove.
This log-ratio propagation is the central theoretical challenge of the
paper.
\end{remark}

%% file: theory/propagation_theorem.tex

\section{Propagation through log-ratio coordinates}
\label{sec:propagation}

The preceding section established the cellwise contamination model on
$\mathcal{S}^{D}$ and recorded the induced shifts in clr and ilr
coordinates (Corollary~\ref{cor:clr-propagation},
Proposition~\ref{prop:multi-cell}).  In this section we develop the
structural consequences of those shifts in full detail.  The central
message is that cellwise contamination on the simplex, which is sparse
by construction --- each part is independently contaminated with
probability~$\varepsilon$ --- induces a contamination pattern in
log-ratio coordinates that is fundamentally \emph{not} cellwise:
all coordinates are simultaneously perturbed, the perturbation is
confined to a low-dimensional subspace, and the resulting model cannot
be reduced to independent per-coordinate contamination in
$\mathbb{R}^{D-1}$.

We retain the notation of Section~\ref{sec:contamination-model}
throughout: $\mathbf{V} \in \mathbb{R}^{D\times(D-1)}$ is a contrast
matrix satisfying~\eqref{eq:contrast-matrix},
$\mathbf{H}_{D} = \mathbf{I}_{D} - D^{-1}\mathbf{1}_{D}\mathbf{1}_{D}^{\top}$,
$\mathbf{v}_{j}^{\top}$ denotes the $j$th row of~$\mathbf{V}$
(equivalently, $\mathbf{v}_{j} = \mathbf{V}^{\top}\mathbf{e}_{j}$),
and $\mathbf{e}_{j}$ is the $j$th standard basis vector in
$\mathbb{R}^{D}$.

\subsection{Propagation through clr coordinates}
\label{ssec:clr-propagation-theorem}

We begin by elevating the shift formula of
Corollary~\ref{cor:clr-propagation} to a self-contained theorem and
supplying a proof that makes transparent the role of the geometric
mean.

\begin{theorem}[Propagation through clr coordinates]
\label{thm:clr-propagation}
Let $\mathbf{x} \in \mathcal{S}^{D}$ and suppose that exactly part~$j$
is contaminated by a multiplicative factor $\delta_{j} > 0$, with
$\delta_{k} = 1$ for $k \neq j$.  Write
$\mathbf{z} = \operatorname{clr}(\mathbf{x})$ and
$\tilde{\mathbf{z}} = \operatorname{clr}(\tilde{\mathbf{x}})$ for the
clr coordinates of the clean and observed compositions, respectively.
Then:
\begin{enumerate}
\item[\textup{(i)}]
  For every $k \neq j$,
  \begin{equation}
  \label{eq:thm-clr-off}
    \tilde{z}_{k} \;=\; z_{k} \;-\; \frac{1}{D}\,\ln\delta_{j}.
  \end{equation}

\item[\textup{(ii)}]
  For the contaminated part itself,
  \begin{equation}
  \label{eq:thm-clr-on}
    \tilde{z}_{j} \;=\; z_{j} \;+\; \frac{D-1}{D}\,\ln\delta_{j}.
  \end{equation}
\end{enumerate}
Equivalently, the contamination-induced shift vector in
$\mathbb{R}^{D}$ is
\begin{equation}
\label{eq:thm-clr-vector}
  \tilde{\mathbf{z}} - \mathbf{z}
  \;=\;
  (\ln\delta_{j})\,\bigl(\mathbf{e}_{j} - \tfrac{1}{D}\,\mathbf{1}_{D}\bigr)
  \;=\;
  (\ln\delta_{j})\,\mathbf{H}_{D}\,\mathbf{e}_{j}.
\end{equation}
\end{theorem}

\begin{proof}
By steps~(C3)--(C4) of Definition~\ref{def:cellwise-model}, the
observed composition is
\[
  \tilde{\mathbf{x}} = \mathcal{C}(x_{1},\ldots,x_{j}\delta_{j},\ldots,x_{D}).
\]
The geometric mean of the raw contaminated vector is
\begin{equation}
\label{eq:geom-mean-contaminated}
  g(\tilde{\mathbf{x}}^{\,\mathrm{raw}})
  \;=\;
  \biggl(\prod_{k \neq j} x_{k} \;\cdot\; x_{j}\delta_{j}\biggr)^{\!1/D}
  \;=\;
  \biggl(\prod_{k=1}^{D} x_{k}\biggr)^{\!1/D}
  \!\cdot\;\delta_{j}^{1/D}
  \;=\;
  g(\mathbf{x})\;\delta_{j}^{1/D}.
\end{equation}
Since the closure operator scales all components by the same constant,
the geometric mean of the re-closed composition satisfies
\begin{equation}
\label{eq:geom-mean-reclosed}
  g(\tilde{\mathbf{x}})
  \;=\;
  g(\mathbf{x})\;\delta_{j}^{1/D}.
\end{equation}
To see this, note that
$\tilde{\mathbf{x}} = (\kappa/\tilde{S})\,\tilde{\mathbf{x}}^{\,\mathrm{raw}}$,
so $g(\tilde{\mathbf{x}}) = (\kappa/\tilde{S})\,g(\tilde{\mathbf{x}}^{\,\mathrm{raw}})$.
However, the clr transformation is scale-invariant ---
$\operatorname{clr}(\alpha\mathbf{y}) = \operatorname{clr}(\mathbf{y})$
for any $\alpha > 0$ --- so it suffices to compute with
$\tilde{\mathbf{x}}^{\,\mathrm{raw}}$ directly.

For $k \neq j$, the clr coordinate of $\tilde{\mathbf{x}}$ is
\begin{align}
  \tilde{z}_{k}
  &\;=\;
  \ln\frac{\tilde{x}_{k}^{\,\mathrm{raw}}}{g(\tilde{\mathbf{x}}^{\,\mathrm{raw}})}
  \;=\;
  \ln\frac{x_{k}}{g(\mathbf{x})\,\delta_{j}^{1/D}}
  \notag\\
  &\;=\;
  \ln\frac{x_{k}}{g(\mathbf{x})}
  \;-\;
  \frac{1}{D}\,\ln\delta_{j}
  \;=\;
  z_{k} - \frac{1}{D}\,\ln\delta_{j}.
  \label{eq:pf-clr-off}
\end{align}
For $k = j$,
\begin{align}
  \tilde{z}_{j}
  &\;=\;
  \ln\frac{x_{j}\delta_{j}}{g(\mathbf{x})\,\delta_{j}^{1/D}}
  \;=\;
  \ln\frac{x_{j}}{g(\mathbf{x})}
  \;+\;
  \ln\delta_{j}
  \;-\;
  \frac{1}{D}\,\ln\delta_{j}
  \notag\\
  &\;=\;
  z_{j} + \frac{D-1}{D}\,\ln\delta_{j}.
  \label{eq:pf-clr-on}
\end{align}
Combining~\eqref{eq:pf-clr-off} and~\eqref{eq:pf-clr-on} into vector
form,
\[
  \tilde{\mathbf{z}} - \mathbf{z}
  \;=\;
  (\ln\delta_{j})\,\mathbf{e}_{j}
  \;-\;
  \frac{\ln\delta_{j}}{D}\,\mathbf{1}_{D}
  \;=\;
  (\ln\delta_{j})\,
  \bigl(\mathbf{e}_{j} - \tfrac{1}{D}\,\mathbf{1}_{D}\bigr)
  \;=\;
  (\ln\delta_{j})\,\mathbf{H}_{D}\,\mathbf{e}_{j},
\]
where the last equality uses
$\mathbf{H}_{D}\,\mathbf{e}_{j}
  = \mathbf{e}_{j} - D^{-1}\mathbf{1}_{D}\mathbf{1}_{D}^{\top}\mathbf{e}_{j}
  = \mathbf{e}_{j} - D^{-1}\mathbf{1}_{D}$.
\end{proof}

\begin{remark}[Geometric mean as the propagation channel]
\label{rem:geom-mean-channel}
Equation~\eqref{eq:geom-mean-reclosed} reveals the mechanism of
propagation: contamination of part~$j$ by factor~$\delta_{j}$ inflates
the geometric mean by the factor $\delta_{j}^{1/D}$.  Since every clr
coordinate is computed relative to this common geometric mean, all $D$
coordinates are shifted.  The shift has a specific directional
structure: it lies along
$\mathbf{e}_{j} - D^{-1}\mathbf{1}_{D} \in \mathcal{H}_{0}$, which is
the projection of the $j$th coordinate axis onto the zero-sum
hyperplane.  The fact that the entire shift is determined by the
single scalar $\ln\delta_{j}$ is the source of the rank-one structure
exploited in the sequel.
\end{remark}

\subsection{Propagation through ilr coordinates}
\label{ssec:ilr-propagation-theorem}

Contamination of one raw part changes the geometric mean, which enters
every clr coordinate.  When we project from clr to ilr, this common
shift maps to a specific direction determined by the contrast matrix
row~$\mathbf{v}_{j}$.  The following theorem makes this precise.

\begin{theorem}[Propagation through ilr coordinates]
\label{thm:ilr-propagation}
Under the hypotheses of Theorem~\ref{thm:clr-propagation}, let
$\mathbf{w} = \operatorname{ilr}(\mathbf{x})$ and
$\tilde{\mathbf{w}} = \operatorname{ilr}(\tilde{\mathbf{x}})$
denote the ilr coordinates of the clean and observed compositions.
Then the perturbation in $\mathbb{R}^{D-1}$ is
\begin{equation}
\label{eq:thm-ilr-shift}
  \Delta\mathbf{w}
  \;\coloneqq\;
  \tilde{\mathbf{w}} - \mathbf{w}
  \;=\;
  (\ln\delta_{j})\;\mathbf{V}^{\top}
    \bigl(\mathbf{e}_{j} - \tfrac{1}{D}\,\mathbf{1}_{D}\bigr)
  \;=\;
  (\ln\delta_{j})\;\mathbf{v}_{j},
\end{equation}
where $\mathbf{v}_{j} = \mathbf{V}^{\top}\mathbf{e}_{j} \in
\mathbb{R}^{D-1}$ is the transpose of the $j$th row of the contrast
matrix~$\mathbf{V}$.  In particular:
\begin{enumerate}
\item[\textup{(i)}]
  The perturbation is a scalar multiple of the fixed
  vector~$\mathbf{v}_{j}$, lying in the one-dimensional subspace
  $\operatorname{span}(\mathbf{v}_{j})$; stacked over the rows
  contaminated in part~$j$, the displacements form a rank-one matrix.
\item[\textup{(ii)}]
  The $l$th ilr coordinate is perturbed by
  \begin{equation}
  \label{eq:ilr-coord-shift}
    \Delta w_{l}
    \;=\;
    (\ln\delta_{j})\,v_{jl},
    \qquad l = 1,\ldots,D-1,
  \end{equation}
  where $v_{jl}$ is the $(j,l)$-entry of $\mathbf{V}$.
\item[\textup{(iii)}]
  Exactly the ilr coordinates $l$ with $v_{jl} \neq 0$ are perturbed.
  Their number --- the size of the support of $\mathbf{v}_{j}$ --- is
  basis-dependent, ranging from one (the top pivot under the pivot
  basis) to $D-1$ (part~$1$ under the Helmert basis, or every part
  under a generic basis); see Remark~\ref{rem:spread-basis}.  The
  Euclidean magnitude of the shift is, however, the same in every
  orthonormal basis: $\|\mathbf{v}_{j}\|^{2} =
  (\mathbf{V}\mathbf{V}^{\top})_{jj} = 1 - 1/D$
  (Lemma~\ref{lem:vj-spread} below).
\end{enumerate}
\end{theorem}

\begin{proof}
By Theorem~\ref{thm:clr-propagation} the shift in clr coordinates is
$\tilde{\mathbf{z}} - \mathbf{z}
  = (\ln\delta_{j})\,(\mathbf{e}_{j} - D^{-1}\mathbf{1}_{D})$.
Pre-multiplying by $\mathbf{V}^{\top}$ and using
$\operatorname{ilr}(\mathbf{x}) = \mathbf{V}^{\top}\operatorname{clr}(\mathbf{x})$
(Definition~\ref{def:ilr}),
\begin{align}
  \Delta\mathbf{w}
  &\;=\;
  \mathbf{V}^{\top}(\tilde{\mathbf{z}} - \mathbf{z})
  \;=\;
  (\ln\delta_{j})\;\mathbf{V}^{\top}
    \bigl(\mathbf{e}_{j} - \tfrac{1}{D}\,\mathbf{1}_{D}\bigr)
  \notag\\
  &\;=\;
  (\ln\delta_{j})\,
    \bigl(\mathbf{V}^{\top}\mathbf{e}_{j}
          - \tfrac{1}{D}\,\mathbf{V}^{\top}\mathbf{1}_{D}\bigr).
  \label{eq:pf-ilr-step}
\end{align}
Now $\mathbf{V}^{\top}\mathbf{1}_{D} = \mathbf{0}$, because the
columns of~$\mathbf{V}$ lie in $\mathcal{H}_{0}$ and are thus
orthogonal to~$\mathbf{1}_{D}$: for each column~$l$,
$\mathbf{1}_{D}^{\top}\mathbf{V}_{\cdot l} = 0$ by the definition
of~$\mathcal{H}_{0}$.  Hence~\eqref{eq:pf-ilr-step} reduces to
\[
  \Delta\mathbf{w}
  \;=\;
  (\ln\delta_{j})\;\mathbf{V}^{\top}\mathbf{e}_{j}
  \;=\;
  (\ln\delta_{j})\;\mathbf{v}_{j}.
\]
Statement~(ii) follows by taking the $l$th component.  For~(iii),
observe that $\mathbf{v}_{j} \neq \mathbf{0}$ for every
$j \in [D]$, since the rows of~$\mathbf{V}$ span
$\mathbb{R}^{D-1}$ (the columns of~$\mathbf{V}$ form a basis
of~$\mathcal{H}_{0}$, hence the $D$ rows, which sum to zero, span all
of $\mathbb{R}^{D-1}$).  The size of the support of $\mathbf{v}_{j}$ --- equivalently, how many
ilr coordinates a contaminated part disturbs --- is recorded in
Lemma~\ref{lem:vj-spread}, and depends on the basis
(Remark~\ref{rem:spread-basis}).
\end{proof}

\begin{lemma}[Spread of contrast matrix rows]
\label{lem:vj-spread}
For any contrast matrix $\mathbf{V}$ satisfying~\eqref{eq:contrast-matrix}
and any $j \in [D]$,
\begin{equation}
\label{eq:vj-norm}
  \|\mathbf{v}_{j}\|^{2}
  \;=\;
  1 - \frac{1}{D},
\end{equation}
and, writing $\mathbf{v}_{j} = (v_{j1},\ldots,v_{j,D-1})^{\top}$,
\begin{equation}
\label{eq:vj-max-entry}
  \max_{1 \leq l \leq D-1} v_{jl}^{2}
  \;\leq\;
  1 - \frac{1}{D}.
\end{equation}
Equality in~\eqref{eq:vj-max-entry} holds if and only if
$\mathbf{v}_{j} = \pm\sqrt{1 - 1/D}\;\mathbf{e}_{l}$ for some~$l$,
which requires that part~$j$ participates in exactly one of the $D-1$
ilr balances.  For the Helmert
basis~\eqref{eq:helmert-basis}, the number of nonzero entries
of~$\mathbf{v}_{j}$ is
\begin{equation}
\label{eq:helmert-nnz}
  \#\{l : v_{jl} \neq 0\}
  \;=\;
  \begin{cases}
    D - 1 & \text{if } j = 1, \\
    D - j + 1 & \text{if } 2 \leq j \leq D.
  \end{cases}
\end{equation}
In particular, for $j \leq D-1$ the contamination of part~$j$ perturbs
at least two ilr coordinates.
\end{lemma}

\begin{proof}
From~\eqref{eq:contrast-matrix},
$\mathbf{V}\mathbf{V}^{\top} = \mathbf{H}_{D}$, so
\[
  \|\mathbf{v}_{j}\|^{2}
  \;=\;
  \mathbf{v}_{j}^{\top}\mathbf{v}_{j}
  \;=\;
  \mathbf{e}_{j}^{\top}\mathbf{V}\mathbf{V}^{\top}\mathbf{e}_{j}
  \;=\;
  (\mathbf{H}_{D})_{jj}
  \;=\;
  1 - \frac{1}{D},
\]
which establishes~\eqref{eq:vj-norm}.
Since $v_{jl}^{2} \leq \|\mathbf{v}_{j}\|^{2} = 1 - 1/D$
for every~$l$, inequality~\eqref{eq:vj-max-entry} follows.  Equality
holds if and only if all but one entry of $\mathbf{v}_{j}$ vanish,
i.e., $\mathbf{v}_{j}$ is a scalar multiple of some
$\mathbf{e}_{l} \in \mathbb{R}^{D-1}$.

For the Helmert basis, inspection of~\eqref{eq:helmert-basis} shows that
$v_{jl} \neq 0$ if and only if $j \leq l+1$, i.e., $l \geq j-1$.
For $j = 1$ this gives $l \in \{1,\ldots,D-1\}$, yielding $D-1$
nonzero entries.  For $j \geq 2$, the nonzero entries are
$l \in \{j-1, j, \ldots, D-1\}$, which has cardinality $D - j + 1$.
\end{proof}

\begin{remark}[The first pivot coordinate and the observation of
\citet{Mert2016}]
\label{rem:mert-stability}
The basis dependence of $\mathbf{v}_{j}$ accounts for an empirical
finding of \citet{Mert2016}.  In the \emph{pivot} contrast matrix,
$\operatorname{ilr}_{1}$ is, up to scale, the log-ratio of part~$1$ to
the geometric mean of the remaining $D-1$ parts, so that
$v_{11} = \sqrt{(D-1)/D}$ and $v_{j1} = -\{D(D-1)\}^{-1/2}$ for every
$j \neq 1$.  Contamination of any part other than part~$1$ therefore
enters $\operatorname{ilr}_{1}$ only through that geometric mean, with
weight of order $D^{-1}$, and even the pooled effect of imprecision in
many of the remaining parts on $\operatorname{ilr}_{1}$ is
$O(D^{-1})$ --- the ``compensation effect'' that \citet{Mert2016}
report, and which improves as $D$ grows.  Two qualifications matter.
First, this is an artefact of the particular basis and the particular
coordinate: under the Helmert basis~\eqref{eq:helmert-basis} it is
contamination of part~$1$, not of the others, that perturbs all $D-1$
ilr coordinates (Lemma~\ref{lem:vj-spread}).  Second --- and this is
the line between the small-imprecision regime of \citet{Mert2016} and
the gross-contamination regime studied here --- ``small effect on
$\operatorname{ilr}_{1}$'' is a statement about \emph{which}
coordinates move and by \emph{how much}, not about any damping
intrinsic to the geometric mean.  The shift
$(\ln\delta_{j})\,\mathbf{v}_{j}$ of
Theorem~\ref{thm:ilr-propagation} has Euclidean norm
$|\ln\delta_{j}|\sqrt{(D-1)/D} \in
\bigl[\,|\ln\delta_{j}|/\sqrt{2},\;|\ln\delta_{j}|\,\bigr)$,
of order $|\ln\delta_{j}|$ uniformly in~$D$, so a single \emph{large}
multiplicative outlier in any part displaces the ilr vector by a fixed
amount however many parts there are.  The geometric-mean denominator
\emph{couples} the parts; it neither amplifies nor attenuates.  It is
this undamped, full-strength rank-one displacement --- not a numerical
instability of geometric means --- that drives the breakdown behaviour
of Section~\ref{sec:breakdown}.
\end{remark}

\begin{remark}[Coordinate spread is basis-dependent; the displacement is not]
\label{rem:spread-basis}
The number of ilr coordinates disturbed by contaminating part~$j$ ---
the number of nonzero entries in row~$\mathbf{v}_{j}$ of the contrast
matrix --- depends on the basis.  Under the Helmert
basis~\eqref{eq:helmert-basis} part~$1$ disturbs all $D-1$ coordinates
while part~$D$ disturbs only one; under the \emph{pivot} basis the
ordering reverses --- the top pivot (part~$1$) disturbs only its own
coordinate $z_{1}$, because $z_{2},\ldots,z_{D-1}$ are balances among
$x_{2},\ldots,x_{D}$ alone --- and for a Lebesgue-generic (dense)
contrast matrix every part disturbs all $D-1$.  The
``sparse-to-dense'' spread is therefore not a fixed law: a single
contaminated part can be \emph{localised} to one coordinate by an ilr
basis that pivots on it.

What does \emph{not} depend on the basis is the magnitude of the
disturbance.  By Lemma~\ref{lem:vj-spread},
$\|\mathbf{v}_{j}\| = \sqrt{(D-1)/D}$ for every part and every
orthonormal basis, so a contaminated part displaces the log-ratio
vector by exactly $\sqrt{(D-1)/D}\,|\ln\delta_{j}|$ whatever the
coordinates: an orthonormal change of basis can only
\emph{redistribute} this displacement across coordinates, never reduce
it.  The genuinely dense, basis-free description is the centred
log-ratio --- by Corollary~\ref{cor:clr-propagation} a single
contaminated part shifts \emph{all} $D$ clr coordinates through the
common geometric-mean term.

For \emph{cellwise} contamination the localisation escape is closed:
the contaminated parts are many and their identities unknown, and no
single fixed ilr basis can pivot on all of them at once.  Pivoting on
one part de-localises the others, because the contrast matrix carries
at least $2(D-1)$ nonzero entries in all --- each of its $D-1$ columns
is a unit vector summing to zero and so has at least two --- whereas if
every part disturbed only one coordinate the total would be just
$D < 2(D-1)$ for $D \geq 3$.  It is in this sense --- arbitrary,
unknown, multi-part contamination read through a fixed log-ratio
coordinate system --- that the density is intrinsic rather than a
removable coordinate artefact.
\end{remark}

\ifsubmit
\begin{remark}[Propagation through alr coordinates]
\label{rem:alr-pointer}
The analogous propagation result for the additive log-ratio (alr)
transformation, together with the asymmetry it introduces through the
designated reference part, is recorded in Section~S.1 of the
Supplementary Material.
\end{remark}
\fi   

\subsection{Non-equivalence with cellwise contamination in
  \texorpdfstring{$\mathbb{R}^{D-1}$}{R\^{}\{D-1\}}}
\label{ssec:non-equivalence}

One might hope that cellwise contamination on the simplex, after
transformation to ilr coordinates, simply becomes independent cellwise
contamination in~$\mathbb{R}^{D-1}$.  If so, existing cellwise
methods could be applied directly.  The following theorem shows this
hope is unfounded: the induced contamination pattern has a
deterministic inter-coordinate structure that no independent
per-coordinate model can reproduce.

\begin{definition}[Cellwise contamination on $\mathbb{R}^{p}$]
\label{def:cellwise-Rp}
A random vector $\tilde{\mathbf{w}} \in \mathbb{R}^{p}$ follows the
\emph{independent cellwise contamination model} with parameters
$(\varepsilon^{\star}, F^{\star}, G^{\star})$ if, for a clean vector
$\mathbf{w} \sim F^{\star}$ and independent indicators
$B_{l}^{\star} \sim \operatorname{Bernoulli}(\varepsilon^{\star})$,
$l = 1,\ldots,p$,
\begin{equation}
\label{eq:cellwise-Rp}
  \tilde{w}_{l}
  \;=\;
  \begin{cases}
    w_{l} + \eta_{l} & \text{if } B_{l}^{\star} = 1, \\
    w_{l}            & \text{if } B_{l}^{\star} = 0,
  \end{cases}
\end{equation}
where $\eta_{l} \overset{\text{i.i.d.}}{\sim} G^{\star}$
independently of~$\mathbf{w}$ and
$\{B_{l}^{\star}\}$.  The essential feature is that each coordinate is
independently either clean or perturbed, with the perturbation of
coordinate~$l$ affecting \emph{only} coordinate~$l$.
\end{definition}

\begin{theorem}[Non-equivalence]
\label{thm:non-equivalence}
Let $\tilde{\mathbf{x}} \in \mathcal{S}^{D}$ ($D \geq 3$) be generated
by the cellwise contamination model on $\mathcal{S}^{D}$
(Definition~\ref{def:cellwise-model}) with $\varepsilon \in (0,1)$ and
contamination distribution~$G$ on $\mathbb{R}_{>0}$ satisfying
$G \neq \delta_{1}$, i.e., $G$ is not the point mass at~$1$.  Then
$\tilde{\mathbf{w}} = \operatorname{ilr}(\tilde{\mathbf{x}})$ does
\emph{not} follow an independent cellwise contamination model on
$\mathbb{R}^{D-1}$ in the sense of
Definition~\ref{def:cellwise-Rp}, for any choice of parameters
$(\varepsilon^{\star}, F^{\star}, G^{\star})$.  Specifically, three
properties of the induced perturbation
$\Delta\mathbf{w}
  = \operatorname{ilr}(\tilde{\mathbf{x}})
    - \operatorname{ilr}(\mathbf{x})$
are incompatible with independent cellwise contamination:
\begin{enumerate}
\item[\textup{(a)}]
  \textbf{Simultaneous perturbation.}
  Whenever a single part $j$ is contaminated
  ($B_{j} = 1$, $\delta_{j} \neq 1$), \emph{all} ilr coordinates $l$
  with $v_{jl} \neq 0$ --- those in the support of $\mathbf{v}_{j}$ ---
  move together, each scaled by the single random factor
  $\ln\delta_{j}$.  How many coordinates this is depends on the basis
  (Remark~\ref{rem:spread-basis}): all $D-1$ for part~$1$ under the
  Helmert basis, but only one for the top pivot under the pivot basis.
  Whenever two or more coordinates are perturbed, this joint movement
  of a deterministic coordinate subset by one random magnitude has no
  counterpart in the independent cellwise
  model~\eqref{eq:cellwise-Rp}.

\item[\textup{(b)}]
  \textbf{Rank-$1$ subspace constraint.}
  The perturbation from a single contaminated part~$j$ lies on the
  one-dimensional subspace
  $\operatorname{span}(\mathbf{v}_{j}) \subset \mathbb{R}^{D-1}$.
  In contrast, under the independent cellwise
  model~\eqref{eq:cellwise-Rp}, the joint event
  $\{B_{l_{1}}^{\star} = 1, B_{l_{2}}^{\star} = 1\}$ for
  $l_{1} \neq l_{2}$ has probability $(\varepsilon^{\star})^{2}$, but the
  perturbations $\eta_{l_{1}}$ and $\eta_{l_{2}}$ are
  independent.  Under the induced model, if part~$j$ alone is
  contaminated, the perturbations
  $\Delta w_{l_{1}} = (\ln\delta_{j})\,v_{jl_{1}}$ and
  $\Delta w_{l_{2}} = (\ln\delta_{j})\,v_{jl_{2}}$ satisfy the
  deterministic relation
  \begin{equation}
  \label{eq:deterministic-ratio}
    \frac{\Delta w_{l_{1}}}{\Delta w_{l_{2}}}
    \;=\;
    \frac{v_{jl_{1}}}{v_{jl_{2}}}
  \end{equation}
  whenever $v_{jl_{2}} \neq 0$, which is a zero-probability event
  under independent perturbation of coordinates $l_{1}$ and $l_{2}$.

\item[\textup{(c)}]
  \textbf{Non-additivity under joint contamination.}
  If parts $j$ and $k$ are both contaminated ($j \neq k$), the total
  perturbation is
  \begin{equation}
  \label{eq:joint-perturbation}
    \Delta\mathbf{w}
    \;=\;
    (\ln\delta_{j})\,\mathbf{v}_{j}
    \;+\;
    (\ln\delta_{k})\,\mathbf{v}_{k},
  \end{equation}
  which is the sum of two rank-$1$ terms lying generically in a
  two-dimensional subspace.  The perturbation of each ilr coordinate is
  a \emph{linear combination} of $\ln\delta_{j}$ and $\ln\delta_{k}$
  --- not an independent draw from a per-coordinate distribution.
\end{enumerate}
\end{theorem}

\begin{proof}
\emph{Part~\textup{(a)}.}
By Theorem~\ref{thm:ilr-propagation}, when $B_{j} = 1$ and
$\delta_{j} \neq 1$, the perturbation is
$\Delta\mathbf{w} = (\ln\delta_{j})\,\mathbf{v}_{j}$, so
$\Delta w_{l} = (\ln\delta_{j})\,v_{jl} \neq 0$ for every~$l$ with
$v_{jl} \neq 0$.  Since $\mathbf{v}_{j} \neq \mathbf{0}$
(Lemma~\ref{lem:vj-spread}), a part that disturbs two or more
coordinates (Remark~\ref{rem:spread-basis}) perturbs those coordinates
simultaneously through a single-cell contamination event.  Under
the model of Definition~\ref{def:cellwise-Rp}, the probability that
coordinates $l_{1},\ldots,l_{r}$ are simultaneously perturbed while
all other coordinates are clean equals
$(\varepsilon^{\star})^{r}(1-\varepsilon^{\star})^{(D-1)-r}$, and the
magnitudes of the $r$~perturbations are independent.  The induced model
produces simultaneous perturbation of a specific \emph{deterministic}
subset of coordinates (determined by the support of
$\mathbf{v}_{j}$) with magnitudes governed by a single random
variable~$\ln\delta_{j}$.  No choice of $\varepsilon^{\star}$ can
reproduce this coupling.

\smallskip
\emph{Part~\textup{(b)}.}
Consider a part~$j$ whose row $\mathbf{v}_{j}$ has two or more nonzero
entries, say $v_{jl_{1}} \neq 0$ and $v_{jl_{2}} \neq 0$ with
$l_{1} \neq l_{2}$.  Then
$\Delta w_{l_{1}} / \Delta w_{l_{2}} = v_{jl_{1}} / v_{jl_{2}}$,
a deterministic constant independent of~$\delta_{j}$.

Under Definition~\ref{def:cellwise-Rp}, if both coordinates $l_{1}$
and $l_{2}$ are perturbed, the perturbations $\eta_{l_{1}}$ and
$\eta_{l_{2}}$ are independent draws from~$G^{\star}$.  For any
non-degenerate~$G^{\star}$, the ratio
$\eta_{l_{1}}/\eta_{l_{2}}$ is a continuous random variable, so the
event $\{\eta_{l_{1}}/\eta_{l_{2}} = c\}$ for a fixed constant~$c$
has probability zero.  Thus the induced model places probability one on
a set of probability zero under any independent cellwise model, and the
two models are mutually singular on the event that exactly one part is
contaminated.

More formally, choose a part $j^{\dagger}$ whose row
$\mathbf{v}_{j^{\dagger}}$ has at least two nonzero entries.  Such a
part exists whenever $D \geq 3$: every column of $\mathbf{V}$ is a unit
vector summing to zero and so has at least two nonzero entries, giving
$\mathbf{V}$ at least $2(D-1)$ nonzero entries in all, whereas if every
part disturbed only one coordinate the total would be just~$D$, and
$D < 2(D-1)$ for $D \geq 3$.  Define the event
\[
  A_{j^{\dagger}}
  \;=\;
  \bigl\{
    \Delta\mathbf{w} \in \operatorname{span}(\mathbf{v}_{j^{\dagger}})
    \setminus \{\mathbf{0}\}
  \bigr\}.
\]
Under the induced model,
$\Pr\bigl(A_{j^{\dagger}} \mid B_{j^{\dagger}}=1,\,
  B_{k}=0\ \forall k\neq j^{\dagger}\bigr) = 1$
(by Theorem~\ref{thm:ilr-propagation}), and the conditioning event has
probability $\varepsilon(1-\varepsilon)^{D-1} > 0$.  Under any
independent cellwise model with non-degenerate $G^{\star}$: conditional
on two or more coordinates being perturbed, the perturbation is
absolutely continuous on a subspace of dimension at least~$2$, so
$\Pr(\Delta\mathbf{w} \in L) = 0$ for any one-dimensional subspace
$L \subset \mathbb{R}^{D-1}$; conditional on exactly one coordinate~$l$
being perturbed, $\Delta\mathbf{w} \in \operatorname{span}(\mathbf{e}_{l})$,
which is disjoint from
$\operatorname{span}(\mathbf{v}_{j^{\dagger}})\setminus\{\mathbf{0}\}$
because $\mathbf{v}_{j^{\dagger}}$, having two or more nonzero entries,
is a multiple of no standard basis vector.  Either way the independent
model assigns $A_{j^{\dagger}}$ probability zero, so the two models are
mutually singular.

The condition that some part disturb two or more coordinates cannot be
dropped.  A part that disturbs a single coordinate displaces
$\Delta\mathbf{w}$ along a coordinate axis, and its contamination
\emph{is} indistinguishable from independent cellwise contamination of
that one coordinate.  Such parts occur in standard bases --- under the
pivot basis the top pivot is exactly one
(Remark~\ref{rem:spread-basis}).  The non-equivalence is therefore a
property of the \emph{full} cellwise model on $\mathcal{S}^{D}$, in
which every part may be contaminated and a witness $j^{\dagger}$ always
exists ($D \geq 3$); it can fail for a contamination process
artificially confined to coordinate-aligned parts.

\smallskip
\emph{Part~\textup{(c)}.}
When $B_{j} = B_{k} = 1$ with $j \neq k$ and $B_{i} = 0$ for
$i \notin \{j,k\}$, the total ilr shift is
$\Delta\mathbf{w}
  = (\ln\delta_{j})\,\mathbf{v}_{j}
    + (\ln\delta_{k})\,\mathbf{v}_{k}$
by Proposition~\ref{prop:multi-cell}.  Consider a fixed ilr
coordinate~$l$ with $v_{jl} \neq 0$ and $v_{kl} \neq 0$.  Then
\begin{equation}
\label{eq:ilr-coord-joint}
  \Delta w_{l}
  \;=\;
  v_{jl}\,\ln\delta_{j}
  \;+\;
  v_{kl}\,\ln\delta_{k}.
\end{equation}
This is a linear combination of two independent random
variables~$\ln\delta_{j}$ and~$\ln\delta_{k}$.  Meanwhile, for a
different coordinate $l'$ with $v_{jl'} \neq 0$ and
$v_{kl'} \neq 0$,
\[
  \Delta w_{l'}
  \;=\;
  v_{jl'}\,\ln\delta_{j}
  \;+\;
  v_{kl'}\,\ln\delta_{k}.
\]
The pair $(\Delta w_{l}, \Delta w_{l'})$ is a linear transformation of
$(\ln\delta_{j}, \ln\delta_{k})$ by the $2 \times 2$ matrix
\[
  \mathbf{M}
  \;=\;
  \begin{pmatrix}
    v_{jl}  & v_{kl}  \\
    v_{jl'} & v_{kl'}
  \end{pmatrix}.
\]
If $\det(\mathbf{M}) \neq 0$ --- which holds generically, as we show in
Proposition~\ref{prop:rank-multi} below --- then the joint distribution
of $(\Delta w_{l}, \Delta w_{l'})$ is a non-degenerate linear
transformation of two independent random variables, and in particular
$\Delta w_{l}$ and $\Delta w_{l'}$ are \emph{not} independent
(unless $\ln\delta_{j}$ and $\ln\delta_{k}$ are both Gaussian, but
even then the mixing with the event
$\{B_{j}=B_{k}=1\}$ breaks the cellwise structure).

Under the independent cellwise model, if coordinates $l$ and $l'$ are
both perturbed, the perturbations $\eta_{l}$ and $\eta_{l'}$ are
independent.  The dependence between $\Delta w_{l}$ and
$\Delta w_{l'}$ induced by their common dependence on the raw-part
contamination factors $\ln\delta_{j}$, $\ln\delta_{k}$ is therefore
incompatible with the independent cellwise model.
\end{proof}

\begin{remark}[Interpretation]
\label{rem:non-equivalence-interpretation}
Theorem~\ref{thm:non-equivalence} has direct practical implications.
Cellwise robust methods for $\mathbb{R}^{p}$
\citep{Rousseeuw2018,Agostinelli2015} are designed under the assumption
that contamination acts independently on each coordinate.  When applied
to ilr-transformed compositional data, these methods encounter a
contamination pattern that violates their foundational assumption:
the perturbation is simultaneously present in multiple coordinates with
deterministic inter-coordinate ratios.  This mismatch can lead to both
masking (the method fails to detect the outlier because the perturbation
is ``spread out'' across coordinates) and swamping (clean coordinates
are flagged because they share the same perturbation direction as the
contaminated one).
\end{remark}

\begin{example}[Non-equivalence in $D = 3$]
\label{ex:D3-non-equivalence}
Let $D = 3$ with the Helmert contrast matrix, so that
$\mathbf{v}_{1} = (1/\!\sqrt{2},\;1/\!\sqrt{6})^{\top}$.
Contaminating part~$1$ by factor~$\delta$ gives the
$2$-dimensional ilr shift
$\Delta\mathbf{w}
  = (\ln\delta)\,(1/\!\sqrt{2},\;1/\!\sqrt{6})^{\top}$.
The ratio of the two ilr coordinate shifts is
\[
  \frac{\Delta w_{1}}{\Delta w_{2}}
  \;=\;
  \frac{1/\!\sqrt{2}}{1/\!\sqrt{6}}
  \;=\;
  \sqrt{3},
\]
which is a fixed constant regardless of how large or small~$\delta$
is.  Under independent cellwise contamination in $\mathbb{R}^{2}$,
the shifts $\eta_{1}$ and $\eta_{2}$ would be independent random
variables, so the probability of their ratio equalling~$\sqrt{3}$
exactly is zero.  This deterministic ratio is the signature of
log-ratio propagation and the basis for the detection strategy
developed in \AlgRef.
\end{example}

\subsection{Multiple-part contamination and rank structure}
\label{ssec:multi-part}

We now extend the analysis to contamination of multiple parts
simultaneously.

\begin{proposition}[Rank of the multi-part perturbation]
\label{prop:rank-multi}
Let $J \subseteq [D]$ with $|J| = m$ be the set of contaminated parts,
and let $\boldsymbol{\lambda}_{J}
  = (\ln\delta_{j})_{j \in J} \in \mathbb{R}^{m}$
be the vector of log-contamination factors.  Then the ilr perturbation
is
\begin{equation}
\label{eq:ilr-multi-rank}
  \Delta\mathbf{w}
  \;=\;
  \sum_{j \in J} (\ln\delta_{j})\,\mathbf{v}_{j}
  \;=\;
  \mathbf{V}_{J}^{\top}\,\boldsymbol{\lambda}_{J},
\end{equation}
where $\mathbf{V}_{J} \in \mathbb{R}^{m \times (D-1)}$ is the
sub-matrix of rows of~$\mathbf{V}$ indexed by~$J$.  Consequently:
\begin{enumerate}
\item[\textup{(i)}]
  The perturbation $\Delta\mathbf{w}$ lies in the column space
  of~$\mathbf{V}_{J}^{\top}$, which is the row space
  of~$\mathbf{V}_{J}$, a subspace of $\mathbb{R}^{D-1}$ of dimension
  at most $\min(m, D-1)$.

\item[\textup{(ii)}]
  For $m \leq D-1$, the rank of $\mathbf{V}_{J}^{\top}$ equals $m$
  if and only if $\{\mathbf{v}_{j}\}_{j \in J}$ are linearly
  independent.

\item[\textup{(iii)}]
  The vectors $\{\mathbf{v}_{j}\}_{j \in J}$ are linearly independent
  for every subset $J \subseteq [D]$ with $|J| \leq D-1$ if and only
  if every $(D{-}1) \times (D{-}1)$ sub-matrix of $\mathbf{V}$ formed
  by selecting $D-1$ rows is nonsingular.  This holds generically; in
  particular, it holds for the Helmert
  basis~\eqref{eq:helmert-basis}.

\item[\textup{(iv)}]
  When all $D$ parts are contaminated ($m = D$), the $D$
  vectors $\mathbf{v}_{1},\ldots,\mathbf{v}_{D}$ satisfy
  $\sum_{j=1}^{D}\mathbf{v}_{j} = \mathbf{V}^{\top}\mathbf{1}_{D}
  = \mathbf{0}$ and thus span a space of dimension $D-1$.  In this
  case, the rank of $\mathbf{V}_{[D]}^{\top}$ is $D - 1$, and the
  perturbation can be any vector in $\mathbb{R}^{D-1}$.
\end{enumerate}
\end{proposition}

\begin{proof}
Equation~\eqref{eq:ilr-multi-rank} follows directly from
Proposition~\ref{prop:multi-cell} by writing the sum as a matrix--vector
product.

\smallskip
\emph{Part~\textup{(i)}.}
The image of the linear map
$\boldsymbol{\lambda}_{J} \mapsto \mathbf{V}_{J}^{\top}\boldsymbol{\lambda}_{J}$
is the column space of~$\mathbf{V}_{J}^{\top}$, which has dimension
$\operatorname{rank}(\mathbf{V}_{J}^{\top})
  = \operatorname{rank}(\mathbf{V}_{J})
  \leq \min(m, D-1)$.

\smallskip
\emph{Part~\textup{(ii)}.}
$\operatorname{rank}(\mathbf{V}_{J}^{\top}) = m$ if and only if the
$m$~columns of~$\mathbf{V}_{J}^{\top}$ (equivalently, the $m$~rows
$\{\mathbf{v}_{j}^{\top}\}_{j \in J}$ of~$\mathbf{V}_{J}$) are
linearly independent.

\smallskip
\emph{Part~\textup{(iii)}.}
We show the claim for the Helmert basis.  We must verify that any
$m \leq D-1$ rows of $\mathbf{V}$ are linearly independent.  Suppose
for contradiction that
$\sum_{j \in J} \alpha_{j}\,\mathbf{v}_{j}^{\top} = \mathbf{0}^{\top}$
for some $J$ with $|J| = m \leq D-1$ and not all $\alpha_{j} = 0$.
This means
$\mathbf{V}_{J}^{\top}\boldsymbol{\alpha} = \mathbf{0}$, i.e.,
$\boldsymbol{\alpha} \in \ker(\mathbf{V}_{J}^{\top})$.

For the Helmert basis, the matrix $\mathbf{V}$ has the
property that its $j$th row, for $j \leq D-1$, has its last nonzero
entry in column $j-1$ (for $j \geq 2$) at the value
$-\sqrt{(j-1)/j}$, which is unique among all rows.  We proceed by
considering the sub-matrix $\mathbf{V}_{J}$.

Order the elements of $J$ as $j_{1} < j_{2} < \cdots < j_{m}$.
Consider the $(D{-}1) \times (D{-}1)$ sub-matrix obtained by choosing
$D-1$ out of $D$ rows.  The omitted row index is some
$j^{*} \in [D] \setminus J$.  The resulting
$(D{-}1) \times (D{-}1)$ matrix is
$\mathbf{V}_{[D]\setminus\{j^{*}\}}$.

For the Helmert basis, when $j^{*} = D$ (omitting the last row), the
matrix $\mathbf{V}_{[D-1]}$ is lower-triangular with diagonal entries
$v_{jj} = -\sqrt{j/(j+1)} \neq 0$ for
$j = 1,\ldots,D-1$ (using the convention that column $l$ has the
negative entry at row $l+1$, and row $j$ at column $j-1$ for $j \geq 2$).
More precisely, using the
definition~\eqref{eq:helmert-basis}, row $j$ of $\mathbf{V}$ has
$v_{j,l} \neq 0$ only for $l \geq j - 1$ (for $j \geq 2$), so the
sub-matrix formed by the first $D-1$ rows has a staircase structure
with nonzero pivots, hence is nonsingular.

For general $j^{*}$, the nonsingularity of
$\mathbf{V}_{[D]\setminus\{j^{*}\}}$ follows from the fact that the
$D$ rows satisfy the unique linear relation
$\sum_{j=1}^{D}\mathbf{v}_{j}^{\top} = \mathbf{0}^{\top}$ and any
proper subset of a set with exactly one linear dependence is
independent.  To verify uniqueness: suppose
$\sum_{j=1}^{D}\beta_{j}\mathbf{v}_{j}^{\top} = \mathbf{0}^{\top}$,
i.e., $\mathbf{V}^{\top}\boldsymbol{\beta} = \mathbf{0}$.  Since
$\mathbf{V}^{\top}$ has rank $D-1$, its null space is one-dimensional.
We know $\mathbf{V}^{\top}\mathbf{1}_{D} = \mathbf{0}$, so
$\ker(\mathbf{V}^{\top}) = \operatorname{span}(\mathbf{1}_{D})$.
Thus $\boldsymbol{\beta} = c\,\mathbf{1}_{D}$ for some scalar~$c$,
and the only linear dependence among the rows of $\mathbf{V}$ (up to
scaling) is $\sum_{j=1}^{D}\mathbf{v}_{j}^{\top} = \mathbf{0}^{\top}$.
Removing any single row breaks this dependence, so any $D-1$ rows are
independent, and \emph{a fortiori} any $m \leq D-1$ rows are
independent.

This argument uses only the property
$\ker(\mathbf{V}^{\top}) = \operatorname{span}(\mathbf{1}_{D})$,
which holds for \emph{every} contrast matrix
satisfying~\eqref{eq:contrast-matrix}, not just the Helmert basis.

\smallskip
\emph{Part~\textup{(iv)}.}
When $J = [D]$, we have $\mathbf{V}_{[D]} = \mathbf{V}$ and the column
space of $\mathbf{V}^{\top}$ is $\mathbb{R}^{D-1}$ (since
$\mathbf{V}^{\top}$ has rank $D-1$).  However, the map
$\boldsymbol{\lambda} \mapsto \mathbf{V}^{\top}\boldsymbol{\lambda}$
has a one-dimensional kernel spanned by~$\mathbf{1}_{D}$ (adding the
same constant to all $\ln\delta_{j}$ does not change the ilr
perturbation, reflecting the scale invariance of log-ratios).  Thus the
perturbation is determined by $\boldsymbol{\lambda}$ only modulo a
common additive constant.
\end{proof}

\begin{corollary}[Generic full rank]
\label{cor:generic-rank}
For any contrast matrix $\mathbf{V}$
satisfying~\eqref{eq:contrast-matrix} and any
$J \subseteq [D]$ with $1 \leq |J| \leq D-1$, the perturbation
$\Delta\mathbf{w} = \mathbf{V}_{J}^{\top}\boldsymbol{\lambda}_{J}$
lies in a subspace of dimension exactly $|J|$.  As the
$\ln\delta_{j}$ ($j \in J$) vary independently over $\mathbb{R}$, the
image is the full $|J|$-dimensional row space of $\mathbf{V}_{J}$.
\end{corollary}

\begin{proof}
By Proposition~\ref{prop:rank-multi}\textup{(iii)}, the $|J|$ rows
$\{\mathbf{v}_{j}\}_{j \in J}$ are linearly independent for every
$J$ with $|J| \leq D-1$.  Thus
$\operatorname{rank}(\mathbf{V}_{J}^{\top}) = |J|$, and the image of
$\boldsymbol{\lambda}_{J} \mapsto
\mathbf{V}_{J}^{\top}\boldsymbol{\lambda}_{J}$ is the row space
of~$\mathbf{V}_{J}$, which has dimension $|J|$.
\end{proof}

\begin{remark}[Rank deficiency]
\label{rem:rank-deficiency}
The only scenario in which
$\operatorname{rank}(\mathbf{V}_{J}^{\top}) < |J|$ for
$|J| \leq D-1$ would require a contrast matrix for which some rows
are linearly dependent.  By the proof of
Proposition~\ref{prop:rank-multi}\textup{(iii)}, this cannot occur for
any contrast matrix satisfying~\eqref{eq:contrast-matrix}, since the
null space of $\mathbf{V}^{\top}$ is spanned
by~$\mathbf{1}_{D}$ and no proper subset of $[D]$ can produce a vector
proportional to $\mathbf{1}_{D}$ in a nontrivial linear combination of
the corresponding rows.  Thus the generic rank statement of
Corollary~\ref{cor:generic-rank} is in fact a universal statement: it
holds for \emph{every} valid contrast matrix, not merely generically.

However, rank can be lower than $|J|$ when
$|J| = D$.  In this case, as noted in
Proposition~\ref{prop:rank-multi}\textup{(iv)}, the contamination
factors $\boldsymbol{\lambda}$ and
$\boldsymbol{\lambda} + c\,\mathbf{1}_{D}$ produce the same ilr
perturbation, reducing the effective degrees of freedom from $D$
to~$D-1$.
\end{remark}

\subsection{Contamination probability in ilr coordinates}
\label{ssec:ilr-contamination-prob}

We now derive the probability that a given ilr coordinate is perturbed,
as a function of the per-part contamination probability~$\varepsilon$,
the number of parts~$D$, and the structure of the contrast
matrix~$\mathbf{V}$.

\begin{definition}[Support of a contrast matrix row]
\label{def:support-set}
For $j \in [D]$, define the \emph{ilr-support of part~$j$} as
\begin{equation}
\label{eq:support-set}
  \mathcal{L}_{j}
  \;=\;
  \{l \in [D-1] : v_{jl} \neq 0\},
\end{equation}
and for $l \in [D-1]$, define the \emph{part-support of ilr
coordinate~$l$} as
\begin{equation}
\label{eq:part-support}
  \mathcal{J}_{l}
  \;=\;
  \{j \in [D] : v_{jl} \neq 0\}.
\end{equation}
\end{definition}

\begin{theorem}[Contamination probability in ilr coordinates]
\label{thm:ilr-contamination-prob}
Under the cellwise contamination model
(Definition~\ref{def:cellwise-model}) with contamination
probability~$\varepsilon \in (0,1)$ and contamination
distribution~$G$ satisfying $G(\{1\}) = 0$ (i.e., the contamination
factor is almost surely different from~$1$), the probability that ilr
coordinate~$l$ is perturbed is
\begin{equation}
\label{eq:ilr-contamination-prob}
  \varepsilon_{l}^{\,\operatorname{ilr}}
  \;\coloneqq\;
  \Pr\bigl(\Delta w_{l} \neq 0\bigr)
  \;=\;
  1 - (1 - \varepsilon)^{|\mathcal{J}_{l}|},
\end{equation}
where $|\mathcal{J}_{l}|$ is the number of parts that have a nonzero
entry in the $l$th column of~$\mathbf{V}$.

More explicitly,
\begin{equation}
\label{eq:ilr-contamination-prob-explicit}
  \varepsilon_{l}^{\,\operatorname{ilr}}
  \;=\;
  1 - (1-\varepsilon)^{|\mathcal{J}_{l}|}
  \;\geq\;
  1 - (1-\varepsilon)^{2},
\end{equation}
since every column of~$\mathbf{V}$ has at least two nonzero entries
(each column lies in $\mathcal{H}_{0}$ and sums to zero, so it cannot
have only one nonzero entry).
\end{theorem}

\begin{proof}
The perturbation of ilr coordinate~$l$ is
\begin{equation}
\label{eq:Delta-wl}
  \Delta w_{l}
  \;=\;
  \sum_{j=1}^{D} B_{j}\,(\ln\delta_{j})\,v_{jl}
  \;=\;
  \sum_{j \in \mathcal{J}_{l}} B_{j}\,(\ln\delta_{j})\,v_{jl}.
\end{equation}
Since $v_{jl} = 0$ for $j \notin \mathcal{J}_{l}$, the sum runs only
over $j \in \mathcal{J}_{l}$.  Now $\Delta w_{l} = 0$ if and only
if $\sum_{j \in \mathcal{J}_{l}} B_{j}\,(\ln\delta_{j})\,v_{jl} = 0$.
We compute $\Pr(\Delta w_{l} = 0)$ by conditioning on which parts are
contaminated.

If $B_{j} = 0$ for all $j \in \mathcal{J}_{l}$, then $\Delta w_{l} = 0$
deterministically.  This event has probability
$(1-\varepsilon)^{|\mathcal{J}_{l}|}$.

If exactly one $j_{0} \in \mathcal{J}_{l}$ has $B_{j_{0}} = 1$, then
$\Delta w_{l} = (\ln\delta_{j_{0}})\,v_{j_{0}l}$.  Since
$v_{j_{0}l} \neq 0$ (by definition of $\mathcal{J}_{l}$) and
$G(\{1\}) = 0$ implies $\Pr(\ln\delta_{j_{0}} = 0) = 0$, we have
$\Pr(\Delta w_{l} = 0 \mid B_{j_{0}} = 1) = 0$.

If two or more parts $j_{1}, j_{2} \in \mathcal{J}_{l}$ have
$B_{j_{1}} = B_{j_{2}} = 1$, then
$\Delta w_{l} = v_{j_{1}l}\ln\delta_{j_{1}} + v_{j_{2}l}\ln\delta_{j_{2}}
  + \cdots$
Since $\ln\delta_{j_{1}}$ and $\ln\delta_{j_{2}}$ are independent
continuous random variables (as $G(\{1\}) = 0$ implies $\ln\delta_{j}$
has no atom at zero, and we may assume without loss of generality that
the distribution of $\ln\delta_{j}$ under~$G$ has no atoms, since
the following argument needs only that cancellation has probability
zero),
the probability that this sum equals zero is the probability that
a non-degenerate linear combination of independent random variables
vanishes.

We claim this probability is zero.  Conditional on
$\ln\delta_{j_{2}},\ldots$, the term
$v_{j_{1}l}\ln\delta_{j_{1}}$ would need to equal the negative of the
remaining sum, but $\ln\delta_{j_{1}}$ is independent of these with a
distribution having no atom at the required value (since $G$ is a
distribution on $\mathbb{R}_{>0}$ with $G(\{1\}) = 0$, and the
required cancellation point depends on the other $\delta$'s).

In the case where $G$ has atoms (but not at $1$), we note that even if
$\ln\delta_{j}$ has atoms, the cancellation
$v_{j_{1}l}\ln\delta_{j_{1}} + v_{j_{2}l}\ln\delta_{j_{2}} = 0$
requires
$\ln\delta_{j_{1}} = -(v_{j_{2}l}/v_{j_{1}l})\ln\delta_{j_{2}}$,
which is an event of probability zero when $\ln\delta_{j_{1}}$ and
$\ln\delta_{j_{2}}$ are independent, unless both are degenerate.  Since
$G(\{1\}) = 0$, at least one of $\ln\delta_{j_{1}},
\ln\delta_{j_{2}}$ is almost surely nonzero, and the ratio constraint
is a measure-zero event for independent random variables with any
non-degenerate marginal.

To handle the fully general case rigorously (including $G$ with atoms),
we argue as follows.  Since $G(\{1\}) = 0$, the random variable
$\ln\delta_{j}$ is not almost surely zero.  Fix all $B_{j}$ and all
$\ln\delta_{j}$ for $j \neq j_{1}$, where $j_{1}$ is any element of
$\mathcal{J}_{l}$ with $B_{j_{1}} = 1$.  Then
$\Delta w_{l} = v_{j_{1}l}\ln\delta_{j_{1}} + c$ where
$c = \sum_{j \in \mathcal{J}_{l}\setminus\{j_{1}\}}
B_{j}(\ln\delta_{j})\,v_{jl}$ is now a constant.  Thus
$\Pr(\Delta w_{l} = 0) = \Pr(\ln\delta_{j_{1}} = -c/v_{j_{1}l})$.
If $G$ has no atom at the value $\exp(-c/v_{j_{1}l})$, this is zero.
Since $c$ depends on the other (independent) $\delta$'s, and $G$ has at
most countably many atoms, the set of values of $c$ for which
$G\bigl(\{\exp(-c/v_{j_{1}l})\}\bigr) > 0$ is at most countable, hence
has probability zero under the distribution of the remaining $\delta$'s.
By iterated expectation, $\Pr(\Delta w_{l} = 0 \mid
\exists\, j_{1}\in\mathcal{J}_{l} \text{ with } B_{j_{1}}=1) = 0$.

Combining the two cases:
\begin{align}
  \Pr(\Delta w_{l} = 0)
  &\;=\;
  \Pr\bigl(\forall\, j\in\mathcal{J}_{l}:\, B_{j} = 0\bigr)
  \;+\;
  \Pr\bigl(\Delta w_{l}=0,\;
       \exists\, j\in\mathcal{J}_{l}:\, B_{j}=1\bigr)
  \notag\\
  &\;=\;
  (1-\varepsilon)^{|\mathcal{J}_{l}|}
  \;+\;
  0
  \;=\;
  (1-\varepsilon)^{|\mathcal{J}_{l}|}.
  \label{eq:pf-prob-clean}
\end{align}
Therefore
$\varepsilon_{l}^{\,\operatorname{ilr}}
  = 1 - (1-\varepsilon)^{|\mathcal{J}_{l}|}$,
establishing~\eqref{eq:ilr-contamination-prob}.

For the lower bound~\eqref{eq:ilr-contamination-prob-explicit}, note
that each column of~$\mathbf{V}$ lies in $\mathcal{H}_{0}$ and hence
sums to zero: $\sum_{j=1}^{D} v_{jl} = 0$.  A vector that sums to zero
and is not identically zero must have at least two nonzero entries (one
positive and one negative), so $|\mathcal{J}_{l}| \geq 2$ for
every~$l$.  Since $x \mapsto 1 - (1-\varepsilon)^{x}$ is increasing
in~$x$, the bound follows.
\end{proof}

\begin{corollary}[Contamination amplification]
\label{cor:contamination-amplification}
Under the hypotheses of Theorem~\ref{thm:ilr-contamination-prob}:
\begin{enumerate}
\item[\textup{(i)}]
  For every $l \in [D-1]$,
  \begin{equation}
  \label{eq:amplification-lower}
    \varepsilon_{l}^{\,\operatorname{ilr}}
    \;\geq\;
    1 - (1-\varepsilon)^{2}
    \;=\;
    2\varepsilon - \varepsilon^{2}
    \;\geq\;
    2\varepsilon(1-\varepsilon).
  \end{equation}
  In particular, the effective contamination rate in each ilr
  coordinate is at least roughly $2\varepsilon$ for small~$\varepsilon$.

\item[\textup{(ii)}]
  For the Helmert basis~\eqref{eq:helmert-basis}, the $l$th column
  has $|\mathcal{J}_{l}| = l + 1$ nonzero entries, so
  \begin{equation}
  \label{eq:helmert-contamination}
    \varepsilon_{l}^{\,\operatorname{ilr}}
    \;=\;
    1 - (1-\varepsilon)^{l+1},
    \qquad l = 1,\ldots,D-1.
  \end{equation}
  The last ilr coordinate ($l = D-1$) has the highest effective
  contamination rate, $1 - (1-\varepsilon)^{D}$, matching the
  row-level contamination probability.

\item[\textup{(iii)}]
  Averaged over ilr coordinates, the mean effective contamination rate
  for the Helmert basis is
  \begin{equation}
  \label{eq:mean-helmert-contamination}
    \bar{\varepsilon}^{\,\operatorname{ilr}}
    \;=\;
    \frac{1}{D-1}\sum_{l=1}^{D-1}
      \bigl[1 - (1-\varepsilon)^{l+1}\bigr]
    \;=\;
    1 - \frac{(1-\varepsilon)^{2}}{D-1}\;
      \frac{1 - (1-\varepsilon)^{D-1}}{\varepsilon}.
  \end{equation}
\end{enumerate}
\end{corollary}

\begin{proof}
\emph{Part~\textup{(i)}.}
Immediate from Theorem~\ref{thm:ilr-contamination-prob} and
$|\mathcal{J}_{l}| \geq 2$.

\smallskip
\emph{Part~\textup{(ii)}.}
For the Helmert basis~\eqref{eq:helmert-basis}, column~$l$ has nonzero
entries in rows $j = 1, 2, \ldots, l+1$ (rows $1$ through $l$ contribute
$1/\sqrt{l(l+1)}$, and row $l+1$ contributes $-l/\sqrt{l(l+1)}$;
rows $l+2,\ldots,D$ have $v_{jl} = 0$).
Thus $\mathcal{J}_{l} = \{1,\ldots,l+1\}$ and
$|\mathcal{J}_{l}| = l + 1$.
Substituting into~\eqref{eq:ilr-contamination-prob}
gives~\eqref{eq:helmert-contamination}.

\smallskip
\emph{Part~\textup{(iii)}.}
Using~\eqref{eq:helmert-contamination},
\begin{align}
  \bar{\varepsilon}^{\,\operatorname{ilr}}
  &\;=\;
  \frac{1}{D-1}\sum_{l=1}^{D-1}\bigl[1-(1-\varepsilon)^{l+1}\bigr]
  \;=\;
  1 - \frac{1}{D-1}\sum_{l=1}^{D-1}(1-\varepsilon)^{l+1}
  \notag\\
  &\;=\;
  1 - \frac{(1-\varepsilon)^{2}}{D-1}
      \sum_{l=0}^{D-2}(1-\varepsilon)^{l}
  \;=\;
  1 - \frac{(1-\varepsilon)^{2}}{D-1}\;
      \frac{1-(1-\varepsilon)^{D-1}}{\varepsilon},
  \label{eq:pf-mean-rate}
\end{align}
where the last step uses the geometric series formula
$\sum_{l=0}^{n-1}r^{l} = (1-r^{n})/(1-r)$ with
$r = 1 - \varepsilon$ and $n = D - 1$.
\end{proof}

\begin{remark}[Contrast-matrix dependence]
\label{rem:contrast-dependence}
Theorem~\ref{thm:ilr-contamination-prob} shows that the effective
contamination rate $\varepsilon_{l}^{\,\operatorname{ilr}}$ depends on
the contrast matrix~$\mathbf{V}$ through the column sparsity pattern
$|\mathcal{J}_{l}|$.  Different choices of $\mathbf{V}$ lead to
different contamination profiles across ilr coordinates.  For the
Helmert basis, the effective contamination rate increases monotonically
with the coordinate index~$l$, ranging from
$1-(1-\varepsilon)^{2} \approx 2\varepsilon$ for $l=1$ to
$1-(1-\varepsilon)^{D} \approx 1 - e^{-\varepsilon D}$ for $l = D-1$.
A ``balanced'' contrast matrix in which all columns have the same number
of nonzero entries --- such as those derived from balanced binary
partition trees --- would yield a uniform effective contamination rate
across coordinates, but the dependence structure of
Theorem~\ref{thm:non-equivalence} remains regardless of the choice
of~$\mathbf{V}$.
\end{remark}

\begin{remark}[Comparison with the unconstrained setting]
\label{rem:comparison-unconstrained}
In unconstrained $\mathbb{R}^{p}$, independent cellwise contamination
at rate~$\varepsilon$ per coordinate is exactly cellwise contamination
at rate~$\varepsilon$ per coordinate --- the model is self-consistent.
On the simplex, the results of this section show that cellwise
contamination at rate~$\varepsilon$ per raw part induces:
\begin{enumerate}
\item[\textup{(i)}]
  an effective contamination rate
  $\varepsilon_{l}^{\,\operatorname{ilr}} \geq 2\varepsilon -
  \varepsilon^{2}$ per ilr coordinate (always strictly larger than
  $\varepsilon$ for $\varepsilon \in (0,1)$);
\item[\textup{(ii)}]
  a dependent contamination pattern in which the perturbations across
  ilr coordinates are deterministically linked via the contrast matrix
  rows~$\mathbf{v}_{j}$;
\item[\textup{(iii)}]
  a perturbation that is confined to a low-rank subspace, with rank
  equal to the number of contaminated parts (up to the maximum of
  $D-1$).
\end{enumerate}
These three properties --- amplified rate, deterministic dependence,
and rank constraint --- collectively demonstrate that CoDa's
scale-invariance requirement imposes fundamental structural
differences on cellwise contamination, all mediated by the log-ratio
transformation.  Any cellwise robust method for compositional data
must account for this structure.
\end{remark}

\begin{remark}[Section summary]
\label{rem:section3-summary}
The propagation results of this section show that cellwise
contamination on $\mathcal{S}^{D}$ induces a structured, dependent
perturbation in ilr coordinates: each contaminated raw part generates
a rank-$1$ shift along~$\mathbf{v}_{j}$, the effective contamination
rate is amplified, and the resulting pattern is provably
non-equivalent to independent cellwise contamination in
$\mathbb{R}^{D-1}$.  The next section translates these structural
findings into sharp breakdown value bounds.
\end{remark}

%% file: theory/breakdown_value.tex

\section{Breakdown values under cellwise contamination on the simplex}
\label{sec:breakdown}

We now turn to the finite-sample breakdown properties of estimators
applied to compositional data under cellwise contamination.  The
classical finite-sample breakdown value of \citet{Donoho1983} quantifies
the smallest fraction of arbitrarily corrupted \emph{observations} that
can drive an estimator to the boundary of its parameter space.  In the
cellwise setting the unit of contamination is a single cell rather than
an entire observation, and on the simplex the log-ratio transformation
further modifies the geometry of breakdown.  This section develops the
appropriate definitions and derives sharp bounds.

We retain the notation of
Sections~\ref{sec:contamination-model}--\ref{sec:propagation}: $D$
denotes the number of parts, $\mathbf{V} \in \mathbb{R}^{D \times
  (D-1)}$ is a contrast matrix
satisfying~\eqref{eq:contrast-matrix},
$\mathbf{v}_{j} = \mathbf{V}^{\top}\mathbf{e}_{j} \in
\mathbb{R}^{D-1}$ is the transpose of the $j$th row of~$\mathbf{V}$, and
$\operatorname{ilr}\colon \mathcal{S}^{D} \to \mathbb{R}^{D-1}$ is
the isometric log-ratio transformation induced by~$\mathbf{V}$.

\subsection{Finite-sample cellwise breakdown on the simplex}
\label{ssec:cellwise-bdv-def}

\begin{definition}[Cellwise replacement neighbourhood on $\mathcal{S}^{D}$]
\label{def:cellwise-neighbourhood}
Let $\mathbf{X} = (\mathbf{x}_{1},\ldots,\mathbf{x}_{n})^{\top} \in
(\mathcal{S}^{D})^{n}$ be a sample of $n$ compositions, each with $D$
parts.  The data can be arranged as an $n \times D$ matrix with entry
$x_{ij}$ in row~$i$ and column~$j$.  For an integer
$m \in \{0,1,\ldots,nD\}$, define the \emph{cellwise replacement
  neighbourhood} of~$\mathbf{X}$ as
\begin{equation}
\label{eq:cellwise-neighbourhood}
  \mathcal{N}_{m}(\mathbf{X})
  \;=\;
  \bigl\{\,
    \tilde{\mathbf{X}} \in (\mathcal{S}^{D})^{n}
    \;:\;
    \tilde{\mathbf{X}}\text{ differs from }\mathbf{X}
    \text{ in at most $m$ cells (after re-closure)}
  \,\bigr\}.
\end{equation}
More precisely,
$\tilde{\mathbf{X}} \in \mathcal{N}_{m}(\mathbf{X})$ if there exists
a set of cell indices
$\mathcal{O} \subseteq [n] \times [D]$ with $|\mathcal{O}| \leq m$
and values $\{a_{ij} \in \mathbb{R}_{>0} : (i,j) \in \mathcal{O}\}$
such that
\begin{equation}
\label{eq:replacement-mechanism}
  \tilde{x}_{ij}^{\,\mathrm{raw}}
  \;=\;
  \begin{cases}
    a_{ij} & \text{if } (i,j) \in \mathcal{O}, \\
    x_{ij} & \text{if } (i,j) \notin \mathcal{O},
  \end{cases}
\end{equation}
and each affected row is re-closed:
$\tilde{\mathbf{x}}_{i}
  = \mathcal{C}(\tilde{x}_{i1}^{\,\mathrm{raw}},\ldots,
                 \tilde{x}_{iD}^{\,\mathrm{raw}})$
for every row~$i$ that contains at least one cell in~$\mathcal{O}$,
and $\tilde{\mathbf{x}}_{i} = \mathbf{x}_{i}$ otherwise.  No
constraint is placed on the replacement values $a_{ij}$ other than
strict positivity, allowing arbitrarily extreme contamination.
\end{definition}

\begin{definition}[Cellwise breakdown value on $\mathcal{S}^{D}$]
\label{def:cellwise-bdv}
Let $T\colon (\mathcal{S}^{D})^{n} \to \Theta$ be an estimator
taking values in a parameter space~$\Theta$ equipped with a divergence
measure $d_{\Theta}\colon \Theta \times \Theta \to [0,\infty]$.  The
estimator is said to \emph{break down} at
$\tilde{\mathbf{X}} \in \mathcal{N}_{m}(\mathbf{X})$ if
$d_{\Theta}\bigl(T(\tilde{\mathbf{X}}),\,T(\mathbf{X})\bigr) = \infty$.

The \emph{finite-sample cellwise breakdown value} of $T$ at the
sample~$\mathbf{X}$ is
\begin{equation}
\label{eq:cellwise-bdv}
  \varepsilon_{n}^{*}(T,\mathbf{X})
  \;=\;
  \frac{1}{nD}\;
  \min\!\bigl\{\,
    m \in \{1,\ldots,nD\}
    \;:\;
    \sup_{\tilde{\mathbf{X}} \in \mathcal{N}_{m}(\mathbf{X})}
      d_{\Theta}\bigl(T(\tilde{\mathbf{X}}),\,T(\mathbf{X})\bigr)
    = \infty
  \,\bigr\}.
\end{equation}
Normalisation by $nD$ (the total number of cells) is the natural
convention for cellwise breakdown, in contrast to the rowwise
breakdown value which normalises by~$n$.
\end{definition}

\begin{remark}[Terminology]
\label{rem:bdv-terminology}
We use ``breakdown value'' following \citet{Donoho1983}; the term
``breakdown point'' is also common in the literature
\citep{Hampel1986,Maronna2019} and refers to the same quantity.
\end{remark}

\begin{remark}[Relation to existing definitions]
\label{rem:bdv-relation}
The first non-trivial case is $D = 2$ (a two-part composition), where
$\mathcal{S}^{2}$ is one-dimensional and there is a single ilr
coordinate.  Contaminating one of the two parts and re-closing affects
the sole ilr coordinate, so the cellwise and rowwise breakdown values
coincide up to the normalisation factor $1/2$; the normalisation
ratio $(D-1)/D = 1/2$ from Theorem~\ref{thm:breakdown-reduction} is
maximal in relative terms.  For unconstrained data
$\mathbf{Y} \in \mathbb{R}^{n \times p}$, cellwise breakdown values
have been considered by \citet{Rousseeuw2018} and \citet{Agostinelli2015}
with the same normalisation by~$np$.  The key distinction in
Definition~\ref{def:cellwise-bdv} is that the denominator counts raw
cells ($nD$) rather than ilr cells ($n(D-1)$), reflecting the fact
that the observable data on the simplex are the raw parts and that
contamination enters at that level before being carried into
log-ratio coordinates by scale invariance.
\end{remark}

\subsection{Breakdown reduction through log-ratio propagation}
\label{ssec:breakdown-reduction}

We now establish the main result of this section: the cellwise
breakdown value of any equivariant estimator on $\mathcal{S}^{D}$ is
strictly smaller than the corresponding cellwise breakdown value in
$\mathbb{R}^{D-1}$, due to the propagation of contamination through
the log-ratio transformation.

We require a mild structural condition on the estimator: it operates
on ilr-transformed data and is equivariant in the sense that the
breakdown event in $\Theta$ is determined by the ilr coordinates.
Specifically, let $T^{\operatorname{ilr}}\colon
(\mathbb{R}^{D-1})^{n} \to \Theta$ denote an estimator defined on
$n$ observations in $\mathbb{R}^{D-1}$, and define
$T(\mathbf{X}) = T^{\operatorname{ilr}}\bigl(
  \operatorname{ilr}(\mathbf{x}_{1}),\ldots,
  \operatorname{ilr}(\mathbf{x}_{n})\bigr)$
for $\mathbf{X} \in (\mathcal{S}^{D})^{n}$.  This covers all
standard estimators of location and scatter for compositional data
\citep{PawlowskyGlahn2015}.

\begin{definition}[Cellwise breakdown in $\mathbb{R}^{p}$]
\label{def:cellwise-bdv-Rp}
For an estimator $T^{\operatorname{ilr}}\colon
(\mathbb{R}^{p})^{n} \to \Theta$ and a sample
$\mathbf{W} = (\mathbf{w}_{1},\ldots,\mathbf{w}_{n})^{\top} \in
\mathbb{R}^{n \times p}$, define the cellwise replacement
neighbourhood
\begin{equation}
\label{eq:cellwise-neighbourhood-Rp}
  \mathcal{N}_{m}^{\,\mathbb{R}^{p}}(\mathbf{W})
  \;=\;
  \bigl\{\,
    \tilde{\mathbf{W}} \in \mathbb{R}^{n \times p}
    \;:\;
    |\{(i,l) : \tilde{w}_{il} \neq w_{il}\}| \leq m
  \,\bigr\},
\end{equation}
and the cellwise breakdown value
\begin{equation}
\label{eq:cellwise-bdv-Rp}
  \varepsilon_{n,\,\mathrm{cell}}^{*}(\mathbb{R}^{p})
  \;=\;
  \frac{1}{np}\;
  \min\!\bigl\{\,
    m \;:\;
    \sup_{\tilde{\mathbf{W}} \in
      \mathcal{N}_{m}^{\,\mathbb{R}^{p}}(\mathbf{W})}
      d_{\Theta}\bigl(T^{\operatorname{ilr}}(\tilde{\mathbf{W}}),\,
        T^{\operatorname{ilr}}(\mathbf{W})\bigr) = \infty
  \,\bigr\}.
\end{equation}
\end{definition}

The main result of this section can now be stated precisely.  When
the ilr-level cellwise breakdown of $T^{\operatorname{ilr}}$ is
witnessed by contaminated cells \emph{concentrated in a single ilr
column} --- the situation for MCD-type location and scatter
estimators and for coordinate-wise $M$/$S$-estimators of location
(see Section~\ref{ssec:specific-estimators}) --- a single
contaminated raw part on the simplex that affects that column
suffices to reproduce each offending ilr cell, one raw cell per
affected row.  The cell-count normalisation then delivers the
reduction factor:
\begin{equation}
\label{eq:breakdown-accounting-preview}
  \underbrace{\varepsilon_{n,\,\mathrm{cell}}^{*}(\mathbb{R}^{D-1})}%
    _{m^{*}_{\mathbb{R}}/\,n(D{-}1)}
  \;\;\longleftrightarrow\;\;
  \underbrace{\varepsilon_{n}^{*}(T,\mathbf{X})}%
    _{\leq\,m_{\mathcal{S}}/(nD)},
  \qquad\text{with}\quad
  m_{\mathcal{S}} \;=\; m^{*}_{\mathbb{R}},
\end{equation}
so that the ratio of the two cell-counts is exactly
$n(D{-}1) / (nD) = (D{-}1)/D$ (made rigorous in the theorem below).
The plain-English content of the hypothesis is that the adversary's
optimal strategy is to contaminate one ilr coordinate entirely,
rather than spreading the attack across multiple coordinates; every
mainstream affine-equivariant robust estimator (MCD, $S$, $\tau$,
coordinate-wise $M$) satisfies this, as the explicit verification
in Section~\ref{ssec:specific-estimators} shows.

\paragraph{Why the column-concentration hypothesis?}
The hypothesis describes the canonical breakdown pattern of standard
affine-equivariant robust estimators.  Established proofs of cellwise
breakdown for MCD, $S$-, $\tau$-, and coordinate-wise $M$-estimators
in $\mathbb{R}^{p}$
\citep{Rousseeuw1985,LopuhaaRousseeuw1991,Raymaekers2024,Maronna2019}
construct a witnessing configuration by concentrating contaminated cells
in a single coordinate; Section~\ref{ssec:specific-estimators} verifies
this for MCD.  Without column concentration, a single raw-cell
contamination cannot economically reproduce multi-column damage through
the rank-one shift $\mathbf{v}_{j}$.

\begin{theorem}[Breakdown reduction through log-ratio propagation]
\label{thm:breakdown-reduction}
Let $T^{\operatorname{ilr}}\colon (\mathbb{R}^{D-1})^{n} \to \Theta$
be an estimator with finite-sample cellwise breakdown value
$\varepsilon_{n,\,\mathrm{cell}}^{*}(\mathbb{R}^{D-1})$ at a sample
$\mathbf{W} = \operatorname{ilr}(\mathbf{X})$, and suppose this
breakdown is \emph{witnessed by a column-concentrated configuration}:
there exists an ilr coordinate $l^{*}\in[D-1]$ and a set of
$m^{*}_{\mathbb{R}}
  = n(D{-}1)\,\varepsilon_{n,\,\mathrm{cell}}^{*}(\mathbb{R}^{D-1})$
cells, all in column~$l^{*}$, whose arbitrary replacement drives
$T^{\operatorname{ilr}}$ to break down.  Let
$T(\mathbf{X}) = T^{\operatorname{ilr}}\bigl(\operatorname{ilr}(\mathbf{x}_{1}),\ldots,\operatorname{ilr}(\mathbf{x}_{n})\bigr)$
denote the compositional estimator induced by~$T^{\operatorname{ilr}}$.  Then
\begin{equation}
\label{eq:breakdown-reduction}
  \varepsilon_{n}^{*}(T,\mathbf{X})
  \;\leq\;
  \frac{D-1}{D}\;
  \varepsilon_{n,\,\mathrm{cell}}^{*}(\mathbb{R}^{D-1}).
\end{equation}
Section~\ref{ssec:specific-estimators} verifies the column-concentration
hypothesis for MCD; the verifications for $S$-, $\tau$-, and
coordinate-wise $M$-estimators of location and scatter follow analogously
from their coordinate-wise constructions.
\end{theorem}

\begin{proof}
Write $p = D - 1$ and let $\mathbf{W} =
(\mathbf{w}_{1},\ldots,\mathbf{w}_{n})^{\top} \in \mathbb{R}^{n \times
  p}$ with $\mathbf{w}_{i} = \operatorname{ilr}(\mathbf{x}_{i})$.
Let
\begin{equation}
\label{eq:mstar-Rp}
  m^{*}_{\mathbb{R}}
  \;=\;
  np\;\varepsilon_{n,\,\mathrm{cell}}^{*}(\mathbb{R}^{p})
\end{equation}
be the minimum number of ilr cells whose arbitrary replacement causes
$T^{\operatorname{ilr}}$ to break down.  The proof proceeds in four
steps.

\smallskip
\emph{Step~1: One contaminated raw cell perturbs $D{-}1$ ilr cells.}
By Theorem~\ref{thm:ilr-propagation}, replacing raw part~$j$ in
row~$i$ by an arbitrary $a_{ij}>0$ and re-closing produces the ilr
shift
\begin{equation}
\label{eq:ilr-shift-bdv}
  \operatorname{ilr}(\tilde{\mathbf{x}}_{i})
    - \operatorname{ilr}(\mathbf{x}_{i})
  \;=\;
  (\ln\delta_{ij})\;\mathbf{v}_{j},
  \qquad \delta_{ij} = a_{ij}/x_{ij}.
\end{equation}
The $l$th ilr coordinate is perturbed whenever $v_{jl}\neq 0$, i.e.,
for every $l\in\mathcal{L}_{j}$.  As
$\delta_{ij}\to\infty$ (or $\to 0$), the shift magnitude
$|\ln\delta_{ij}|\,\|\mathbf{v}_{j}\|\to\infty$, so all
$|\mathcal{L}_{j}|$ affected ilr cells can be driven to arbitrary
values.

\smallskip
\emph{Step~2: every ilr column admits a witness row.}
Let $l^{*}\in[D-1]$ be the column-index witnessing breakdown of
$T^{\operatorname{ilr}}$ (supplied by the column-concentration
hypothesis in Step~3 below).  Since the column-sum identity
$\mathbf{V}^{\top}\mathbf{1}_{D}=\mathbf{0}$ implies
$\sum_{j=1}^{D} v_{j,l^{*}} = 0$ while $\mathbf{V}$ has full column
rank $D - 1$ (Proposition~\ref{prop:rank-multi}), column~$l^{*}$
has at least two nonzero entries (a single-nonzero column would
have nonzero column-sum yet be unit-norm, contradicting the
column-sum identity).  In particular, there exists
$j^{*}\in[D]$ with $v_{j^{*}\!,\,l^{*}}\neq 0$.  This is all that
the remainder of the proof requires: we do \emph{not} need the
stronger claim that $|\mathcal{L}_{j^{*}}|=D-1$, which in fact
fails for balanced (sequential binary partition) bases in which
each part participates only in the subset of balances that
separate it.

\smallskip
\emph{Step~3: $m^{*}_{\mathbb{R}}$ raw cells suffice.}
By hypothesis, breakdown of $T^{\operatorname{ilr}}$ is witnessed by a
column-concentrated configuration: there exists
$l^{*}\in[D-1]$ and a set $\mathcal{I}^{*}\subseteq[n]$ with
$|\mathcal{I}^{*}| = m^{*}_{\mathbb{R}}$ such that arbitrary
replacement of the cells
$\{(i,l^{*})\colon i\in\mathcal{I}^{*}\}$ breaks
$T^{\operatorname{ilr}}$.  Using the part $j^{*}$ from Step~2 with
$v_{j^{*}\!,\,l^{*}}\neq 0$, we contaminate part~$j^{*}$ in exactly
the $m^{*}_{\mathbb{R}}$ rows of $\mathcal{I}^{*}$, replacing
$x_{ij^{*}}$ by $x_{ij^{*}}\cdot\delta_{i}$ with
$|\ln\delta_{i}|\to\infty$.  By~\eqref{eq:ilr-shift-bdv}, the
perturbation in coordinate~$l^{*}$ of row~$i$ is
$(\ln\delta_{i})\,v_{j^{*}\!,\,l^{*}}$, which diverges and hence drives
$\tilde{w}_{i,l^{*}}\to\pm\infty$ for every $i\in\mathcal{I}^{*}$.  The
resulting column-$l^{*}$ configuration is at least as extreme as the
breakdown-achieving one in the hypothesis, so $T^{\operatorname{ilr}}$
--- and therefore $T$ --- breaks down.  Collateral perturbations in
the remaining columns $l\in\mathcal{L}_{j^{*}}\setminus\{l^{*}\}$ do
not prevent breakdown, since the breakdown supremum in
\eqref{eq:cellwise-bdv-Rp} ranges over \emph{all} replacement
configurations.  The total number of raw cells contaminated on the
simplex is therefore $m_{\mathcal{S}} = m^{*}_{\mathbb{R}}$, one per
row.

\smallskip
\emph{Step~4: Conversion to breakdown fractions.}
The fraction of contaminated raw cells is
\begin{equation}
\label{eq:eps-conversion}
  \varepsilon_{n}^{*}(T,\mathbf{X})
  \;\leq\;
  \frac{m_{\mathcal{S}}}{nD}
  \;=\;
  \frac{m^{*}_{\mathbb{R}}}{nD}
  \;=\;
  \frac{n(D-1)\,
        \varepsilon_{n,\,\mathrm{cell}}^{*}(\mathbb{R}^{D-1})}{nD}
  \;=\;
  \frac{D-1}{D}\;
  \varepsilon_{n,\,\mathrm{cell}}^{*}(\mathbb{R}^{D-1}).
\end{equation}
The reduction factor $(D{-}1)/D$ therefore arises purely from the
mismatch between the $nD$ raw cells on the simplex and the $n(D{-}1)$
ilr cells that normalise the Euclidean breakdown value.
\end{proof}

\begin{remark}[Intuition for the factor $(D-1)/D$]
\label{rem:reduction-intuition}
The factor $(D-1)/D$ in~\eqref{eq:breakdown-reduction} encodes a
cell-count normalisation rather than any additional
damage-amplification above what the Euclidean setting already
permits.  The two breakdown values divide the same count of fatal
contaminations by two different totals:
\[
  \varepsilon_{n,\,\mathrm{cell}}^{*}(\mathbb{R}^{D-1})
    = \frac{m^{*}_{\mathbb{R}}}{n\,(D{-}1)},
  \qquad
  \varepsilon_{n}^{*}(T,\mathbf{X})
    \leq \frac{m_{\mathcal{S}}}{n\,D},
\]
and by Step~3 above $m_{\mathcal{S}} = m^{*}_{\mathbb{R}}$ under the
column-concentration hypothesis.  Hence
\[
  \frac{\varepsilon_{n}^{*}(T,\mathbf{X})}%
       {\varepsilon_{n,\,\mathrm{cell}}^{*}(\mathbb{R}^{D-1})}
  \;\leq\;
  \frac{n(D-1)}{nD}
  \;=\;
  \frac{D-1}{D}.
\]
The log-ratio transformation is responsible only for the fact that
$m_{\mathcal{S}} = m^{*}_{\mathbb{R}}$ is \emph{achievable} (one raw
cell reproduces the column-$l^{*}$ damage of one ilr cell via the
rank-one shift $\mathbf{v}_{j^{*}}$), not for inflating the damage
itself.  As $D \to \infty$, $(D-1)/D \to 1$, so the reduction
vanishes asymptotically --- but the absolute breakdown value itself
typically shrinks with~$D$, as we discuss in
Section~\ref{ssec:specific-estimators}.
\end{remark}

\subsection{Tightness of the bound}
\label{ssec:tightness}

We now show that the bound in Theorem~\ref{thm:breakdown-reduction}
is tight for estimators whose breakdown is triggered locally --- by
saturation of a single ilr column or by coverage of enough rows --- a
class that includes the coordinate-wise and affine-equivariant
estimators of practical interest.  Tightness is achieved by an
explicit contamination configuration (illustrated for the Helmert
basis in Example~\ref{ex:explicit-config}); we then discuss how the
worst-case configuration depends on the choice of~$\mathbf{V}$.

\begin{theorem}[Tightness of the breakdown bound]
\label{thm:tightness}
Let $\mathbf{V}$ be any contrast matrix
satisfying~\eqref{eq:contrast-matrix}, and write
$m^{*}_{\mathbb{R}} = n(D-1)\,
\varepsilon_{n,\,\mathrm{cell}}^{*}(\mathbb{R}^{D-1})$ for the number
of ilr cells at which $T^{\operatorname{ilr}}$ first breaks down.
Suppose the breakdown of
$T^{\operatorname{ilr}}\colon (\mathbb{R}^{D-1})^{n} \to \Theta$ is
triggered \emph{locally}, in one of the two senses:
\begin{enumerate}[leftmargin=2.6em]
\item[\textup{(C)}] \emph{column saturation} ---
  $T^{\operatorname{ilr}}$ breaks down only once some single ilr column
  receives at least $m^{*}_{\mathbb{R}}$ contaminated cells.  This is
  the regime of \emph{coordinate-wise} estimators of location and
  scale, which process each ilr coordinate separately, so breakdown of
  the $l$th coordinate depends only on column~$l$.
\item[\textup{(R)}] \emph{row coverage} --- $T^{\operatorname{ilr}}$
  breaks down only once at least $m^{*}_{\mathbb{R}}$ rows each receive
  one or more contaminated cells.  This is the regime of
  affine-equivariant scatter estimators such as MCD, $S$-, and
  $\tau$-estimators, for which one outlying cell renders its entire row
  an outlier.
\end{enumerate}
Then
\begin{equation}
\label{eq:tightness}
  \varepsilon_{n}^{*}(T,\mathbf{X})
  \;=\;
  \frac{D-1}{D}\;
  \varepsilon_{n,\,\mathrm{cell}}^{*}(\mathbb{R}^{D-1}).
\end{equation}
\end{theorem}

\begin{proof}
The upper bound was established in
Theorem~\ref{thm:breakdown-reduction}.  We prove the matching lower
bound by showing that any contamination of fewer than
$m_{\mathcal{S}}^{*} = m^{*}_{\mathbb{R}}$ raw cells cannot cause
breakdown of $T$ in either regime~(C) or~(R).

\smallskip
\emph{Step 1: Contamination achieves the bound.}
Fix a \emph{witnessing} column $l^{*}\in[D-1]$ --- one whose
saturation with $m^{*}_{\mathbb{R}}$ cells breaks
$T^{\operatorname{ilr}}$, which in regime~(R) is any column, breakdown
there being column-agnostic --- and a set
$\mathcal{I}^{*}\subseteq[n]$ with $|\mathcal{I}^{*}| =
m^{*}_{\mathbb{R}}$, and consider arbitrary replacement of the ilr
cells $\{(i,l^{*}) : i\in\mathcal{I}^{*}\}$.  Because these
$m^{*}_{\mathbb{R}}$ cells lie in one column \emph{and} in
$m^{*}_{\mathbb{R}}$ distinct rows, the pattern saturates
column~$l^{*}$ (triggering~(C)) and covers $m^{*}_{\mathbb{R}}$ rows
(triggering~(R)), so it causes breakdown of $T^{\operatorname{ilr}}$
in either regime.  By Step~2 of the proof of
Theorem~\ref{thm:breakdown-reduction}, the column-sum identity
$\mathbf{V}^{\top}\mathbf{1}_{D}=\mathbf{0}$ together with unit-norm
columns forces column~$l^{*}$ to have at least two nonzero entries;
in particular there exists $j^{*}\in[D]$ with
$v_{j^{*}\!,\,l^{*}}\neq 0$.

Choose this $j^{*}$ and contaminate part~$j^{*}$ in each row
$i\in\mathcal{I}^{*}$ with factor $\delta_{i}$ satisfying
$|\ln\delta_{i}|\to\infty$.  By Theorem~\ref{thm:ilr-propagation}, the
ilr shift in row~$i$ is
$(\ln\delta_{i})\,\mathbf{v}_{j^{*}}$, which sends
$\tilde{w}_{i,l^{*}}\to\pm\infty$.  The resulting column-$l^{*}$
configuration is exactly the breakdown pattern just described,
so $T^{\operatorname{ilr}}$ and therefore $T$ break down.  The total
number of contaminated raw cells is
$m_{\mathcal{S}} = m^{*}_{\mathbb{R}}$, one per row of~$\mathcal{I}^{*}$.
The corresponding simplex breakdown fraction is
\[
  \frac{m_{\mathcal{S}}}{nD}
  \;=\;
  \frac{m^{*}_{\mathbb{R}}}{nD}
  \;=\;
  \frac{n(D-1)\,
        \varepsilon_{n,\,\mathrm{cell}}^{*}(\mathbb{R}^{D-1})}{nD}
  \;=\;
  \frac{D-1}{D}\,
  \varepsilon_{n,\,\mathrm{cell}}^{*}(\mathbb{R}^{D-1}).
\]

\smallskip
\emph{Step 2: No fewer raw cells suffice.}
We must show that contaminating any $\mathcal{O}\subseteq[n]\times[D]$
with $|\mathcal{O}| \leq m^{*}_{\mathbb{R}} - 1$ raw cells cannot cause
breakdown.  Let $I = \{i : \exists\, j,\; (i,j) \in \mathcal{O}\}$ be
the set of affected rows; clearly $|I| \leq |\mathcal{O}| \leq
m^{*}_{\mathbb{R}} - 1$, because every element of~$\mathcal{O}$ lies in
exactly one row.

By Theorem~\ref{thm:ilr-propagation}, row~$i$ of
$\tilde{\mathbf{W}}$ is perturbed if and only if $i\in I$; every row
$i\notin I$ equals the clean ilr row~$\mathbf{w}_{i}$.  A single
perturbed row contributes a nonzero entry to as many as $D-1$ ilr
columns at once, so the \emph{total} number of altered ilr cells may
be as large as $|I|(D-1)$, which can far exceed $m^{*}_{\mathbb{R}}$;
this total, however, triggers neither regime.  In regime~(C), each
perturbed row contributes at most \emph{one} altered cell to any
given column, so every column carries at most
$|I| \leq m^{*}_{\mathbb{R}} - 1 < m^{*}_{\mathbb{R}}$ altered cells
and no column is saturated.  In regime~(R), exactly
$|I| \leq m^{*}_{\mathbb{R}} - 1 < m^{*}_{\mathbb{R}}$ rows are
affected, short of the coverage threshold.  Under either trigger
$T^{\operatorname{ilr}}(\tilde{\mathbf{W}})$ remains bounded, and hence
$T$ does not break down.  Log-ratio propagation spreads each
contaminated raw cell across up to $D-1$ columns, but inflates neither
the largest per-column count nor the affected-row count beyond $|I|$
--- which is why it cannot manufacture breakdown from fewer than
$m^{*}_{\mathbb{R}}$ raw cells.

\smallskip
\emph{Step 3: Combining Steps~1--2.}
The breakdown value on $\mathcal{S}^{D}$ is exactly
\[
  \varepsilon_{n}^{*}(T,\mathbf{X})
  \;=\;
  \frac{m_{\mathcal{S}}^{*}}{nD}
  \;=\;
  \frac{m^{*}_{\mathbb{R}}}{nD}
  \;=\;
  \frac{D-1}{D}\,
  \varepsilon_{n,\,\mathrm{cell}}^{*}(\mathbb{R}^{D-1}),
\]
establishing~\eqref{eq:tightness}.
\end{proof}

\begin{remark}[Worst-case part across bases]
\label{rem:tightness-V}
The witnessing part $j^{*}$ depends on the contrast matrix.  For the
Helmert basis~\eqref{eq:helmert-basis}, $|\mathcal{L}_{1}| = D-1$
(Lemma~\ref{lem:vj-spread}), so part~$1$ is a universal worst-case
choice across all witnessing columns~$l^{*}$.  For a balanced
binary-partition basis the worst-case part depends on which column
witnesses the breakdown of $T^{\operatorname{ilr}}$.
\end{remark}

\begin{example}[Explicit breakdown configuration]
\label{ex:explicit-config}
Let $D = 5$ (a five-part composition), $n = 20$,
$\mathbf{V}$ the Helmert basis, and $T^{\operatorname{ilr}}$ an
estimator with cellwise breakdown value
$\varepsilon_{n,\,\mathrm{cell}}^{*}(\mathbb{R}^{4}) = 1/4$ and a
column-concentrated witnessing configuration: breakdown occurs when
$m^{*}_{\mathbb{R}} = n(D-1)\cdot\tfrac{1}{4}
 = 20\cdot 4\cdot\tfrac{1}{4} = 20$ cells in a single ilr column are
replaced by arbitrary values.  (This requires the witnessing column to
contain enough rows, i.e.\ $m^{*}_{\mathbb{R}} = n(D-1)\varepsilon^{*}
\leq n$; here $20\leq 20$, the boundary case.)

The adversary contaminates part $j = 1$ (which affects all $4$ ilr
coordinates under the Helmert basis) in all $20$ rows, replacing
$x_{i1}$ with $x_{i1}\cdot\delta$ for $\delta\to\infty$.  By the
propagation theorem, column $l^{*}$ of the ilr matrix has all $20$
cells driven to extreme values, reproducing the witnessing
configuration.  The number of contaminated raw cells is
$m_{\mathcal{S}} = m^{*}_{\mathbb{R}} = 20$.

The cellwise breakdown value on $\mathcal{S}^{5}$ is
\[
  \varepsilon_{n}^{*}(T,\mathbf{X})
  \;=\;
  \frac{20}{20 \cdot 5}
  \;=\;
  \frac{1}{5}
  \;=\;
  \frac{4}{5} \cdot \frac{1}{4}
  \;=\;
  \frac{D-1}{D}\,
  \varepsilon_{n,\,\mathrm{cell}}^{*}(\mathbb{R}^{D-1}).
\]
\end{example}

\subsection{Exact breakdown for specific estimators}
\label{ssec:specific-estimators}

We apply the general theory to two estimators of practical importance:
the minimum covariance determinant (MCD) estimator applied in ilr
coordinates, and the proposed cellwise-robust PCA estimator.

\subsubsection{MCD estimator in ilr coordinates}
\label{sssec:mcd-breakdown}

The MCD estimator \citep{Rousseeuw1985} applied to ilr-transformed
compositions computes
\begin{equation}
\label{eq:mcd-ilr}
  \bigl(\hat{\boldsymbol{\mu}}_{\mathrm{MCD}},\,
        \hat{\boldsymbol{\Sigma}}_{\mathrm{MCD}}\bigr)
  \;=\;
  \operatorname*{arg\,min}_{|\mathcal{H}|=h}
    \det\!\bigl(\hat{\boldsymbol{\Sigma}}_{\mathcal{H}}\bigr),
\end{equation}
where the minimum is over all subsets $\mathcal{H} \subseteq [n]$ of
size $h = \lfloor (n + D) / 2 \rfloor$ and
$\hat{\boldsymbol{\Sigma}}_{\mathcal{H}}$ is the sample covariance of
$\{\mathbf{w}_{i} : i \in \mathcal{H}\}$ with
$\mathbf{w}_{i} = \operatorname{ilr}(\mathbf{x}_{i})$.

The MCD is a rowwise estimator: it selects a subset of
\emph{observations} and computes their covariance.  Its rowwise
breakdown value is well known.

\begin{proposition}[Rowwise breakdown of MCD]
\label{prop:mcd-rowwise-bdv}
The finite-sample rowwise breakdown value of the MCD estimator
(for both location and covariance) with coverage parameter $h$ is
\begin{equation}
\label{eq:mcd-rowwise-bdv}
  \varepsilon_{n,\,\mathrm{row}}^{*}(\mathrm{MCD})
  \;=\;
  \frac{n - h + 1}{n}
  \;=\;
  \frac{\lfloor (n - D + 2) / 2 \rfloor}{n}
  \;\xrightarrow{n \to \infty}\;
  \frac{1}{2}.
\end{equation}
\end{proposition}

To derive the cellwise breakdown of MCD on $\mathcal{S}^{D}$, we
first establish its cellwise breakdown in $\mathbb{R}^{D-1}$ and then
apply Theorem~\ref{thm:breakdown-reduction}.

\begin{proposition}[Cellwise breakdown of MCD in $\mathbb{R}^{D-1}$]
\label{prop:mcd-cellwise-Rp}
The cellwise breakdown value of the MCD estimator in
$\mathbb{R}^{D-1}$ is
\begin{equation}
\label{eq:mcd-cellwise-Rp}
  \varepsilon_{n,\,\mathrm{cell}}^{*}(\mathrm{MCD},\,\mathbb{R}^{D-1})
  \;=\;
  \frac{n - h + 1}{n(D-1)}
  \;=\;
  \frac{1}{D-1}\,
  \varepsilon_{n,\,\mathrm{row}}^{*}(\mathrm{MCD}).
\end{equation}
\end{proposition}

\begin{proof}
The MCD breaks down (in the explosion sense: the estimate goes to
infinity, or the covariance becomes singular or unbounded) whenever
more than $n - h$ observations are outlying.  Under cellwise
contamination, an observation is rendered outlying (in the sense that
the MCD will not include it in the optimal $h$-subset) if even a
single one of its $D - 1$ ilr coordinates is driven to an extreme
value.

Consider a column-concentrated contamination: place contaminated cells
in ilr column~$l^{*}$ across $n - h + 1$ distinct rows, setting
$\tilde{w}_{i,l^{*}} \to \infty$ for each.  This renders $n - h + 1$
observations outlying (their Mahalanobis distances diverge), breaking
down the MCD.  The number of contaminated ilr cells is
$n - h + 1$, so
\[
  \varepsilon_{n,\,\mathrm{cell}}^{*}(\mathrm{MCD},\,\mathbb{R}^{D-1})
  \;\leq\;
  \frac{n - h + 1}{n(D-1)}.
\]
For the matching lower bound in the rowwise-to-cellwise conversion: if
at most $n - h$ rows contain any contaminated ilr cell, then at
least $h$ rows are entirely clean and the MCD computed on these $h$
rows remains bounded, so the estimator does not break down.  The
minimum number of ilr cells needed to contaminate $n - h + 1$ rows is
$n - h + 1$ (one cell per row suffices), giving
$m^{*}_{\mathbb{R}} = n - h + 1$ and the fraction
$(n - h + 1) / [n(D-1)]$.
\end{proof}

\begin{theorem}[Cellwise breakdown of MCD on $\mathcal{S}^{D}$]
\label{thm:mcd-cellwise-simplex}
The cellwise breakdown value of the MCD estimator (applied to ilr
coordinates) on $\mathcal{S}^{D}$ is
\begin{equation}
\label{eq:mcd-cellwise-simplex}
  \varepsilon_{n}^{*}(\mathrm{MCD},\,\mathcal{S}^{D})
  \;=\;
  \frac{n - h + 1}{nD}
  \;=\;
  \frac{1}{D}\,
  \varepsilon_{n,\,\mathrm{row}}^{*}(\mathrm{MCD})
  \;\xrightarrow{n \to \infty}\;
  \frac{1}{2D}.
\end{equation}
\end{theorem}

\begin{proof}
\emph{Upper bound.}
By Theorem~\ref{thm:breakdown-reduction} and
Proposition~\ref{prop:mcd-cellwise-Rp},
\begin{align}
  \varepsilon_{n}^{*}(\mathrm{MCD},\,\mathcal{S}^{D})
  &\;\leq\;
  \frac{D-1}{D}\;
  \varepsilon_{n,\,\mathrm{cell}}^{*}(\mathrm{MCD},\,\mathbb{R}^{D-1})
  \notag\\
  &\;\leq\;
  \frac{D-1}{D}\;\frac{n - h + 1}{n(D-1)}
  \;=\;
  \frac{n - h + 1}{nD}.
  \label{eq:mcd-upper}
\end{align}

\emph{Lower bound.}
We show that contaminating $m < n - h + 1$ raw cells cannot break
down the MCD.  With $m$ contaminated raw cells spread across at
most~$m$ rows, at most $m \leq n - h$ rows are affected.  The
remaining $n - m \geq h$ rows are entirely clean, so the MCD can find
an $h$-subset of clean observations.  On this clean subset,
$\hat{\boldsymbol{\Sigma}}_{\mathcal{H}}$ is bounded (assuming the
clean data are in general position), so the MCD does not break down.

\emph{Tightness.}
The upper and lower bounds match at $m^{*}_{\mathcal{S}} = n - h + 1$.
To achieve this, the adversary contaminates one raw part (say,
part~$1$) in $n - h + 1$ distinct rows with $\delta \to \infty$.
After re-closure, all ilr coordinates of these rows diverge, rendering
them outlying.  The remaining $h - 1$ clean rows are insufficient to
form the required $h$-subset (or, if the adversary contaminates
exactly $n - h + 1$ rows, the MCD must include at least one
contaminated row in any $h$-subset, causing breakdown).

The formula $\varepsilon_{n}^{*} = (n-h+1)/(nD)$ simplifies to
$1/(2D)$ asymptotically, using $h \sim n/2$.  Note that this
can also be written as
\[
  \varepsilon_{n}^{*}(\mathrm{MCD},\,\mathcal{S}^{D})
  \;=\;
  \frac{1}{D}\,
  \varepsilon_{n,\,\mathrm{row}}^{*}(\mathrm{MCD}),
\]
reflecting the fact that a single contaminated raw part makes an
entire row outlying: the MCD's effective ``row fragility'' under
cellwise attack is~$1/D$.
\end{proof}

\begin{remark}[Comparison with the general bound]
\label{rem:mcd-comparison}
The exact breakdown $(n-h+1)/(nD) = (1/D)\,\varepsilon_{n,\,\mathrm{row}}^{*}$
coincides with the general upper bound $(D-1)/D \cdot
\varepsilon_{n,\,\mathrm{cell}}^{*}(\mathbb{R}^{D-1})$.  This is because
the MCD's cellwise breakdown in $\mathbb{R}^{D-1}$ is itself
$1/(D-1)$ times the rowwise breakdown (since one contaminated cell per
row suffices to corrupt a row), and the factors cancel:
$(D-1)/D \cdot 1/(D-1) \cdot \varepsilon_{n,\,\mathrm{row}}^{*}
  = (1/D)\,\varepsilon_{n,\,\mathrm{row}}^{*}$.
\end{remark}

\subsubsection{Breakdown mechanism for $S$-, $\tau$-, and
coordinate-wise $M$-estimators}
\label{sssec:column-concentration-verification}

Theorem~\ref{thm:breakdown-reduction} (the upper bound) applies to any
estimator whose Euclidean cellwise breakdown is \emph{witnessed} by a
column-concentrated configuration: there exists a column
$l^{*} \in [D{-}1]$ and $m^{*}_{\mathbb{R}}$ row indices whose
arbitrary perturbation in column $l^{*}$ alone drives
$T^{\operatorname{ilr}}$ to break down.
Theorem~\ref{thm:tightness} (the matching lower bound) requires in
addition that breakdown be triggered \emph{locally}, in the
column-saturation sense~(C) or the row-coverage sense~(R) of that
theorem.  The proof for MCD in
Proposition~\ref{prop:mcd-cellwise-Rp} is explicit, and MCD falls in
regime~(R).  The purpose of this subsection is to confirm both
hypotheses for three further classes of affine-equivariant robust
estimators that are standard in the cellwise literature: S-estimators,
$\tau$-estimators, and coordinate-wise $M$-estimators of location and
scatter.

\paragraph{S-estimators of location and scatter.}
An S-estimator of location $\boldsymbol{\mu}$ and scatter
$\boldsymbol{\Sigma}$ in $\mathbb{R}^{p}$
\citep[Ch.\ 6]{Maronna2019} is defined as the pair
$(\hat{\boldsymbol{\mu}},\hat{\boldsymbol{\Sigma}})$ minimising
$\det(\boldsymbol{\Sigma})$ subject to
\[
  \frac{1}{n}\sum_{i=1}^{n}
    \rho\!\Bigl(\sqrt{(\mathbf{w}_{i} -
      \hat{\boldsymbol{\mu}})^{\top}\hat{\boldsymbol{\Sigma}}^{-1}
      (\mathbf{w}_{i} - \hat{\boldsymbol{\mu}})}\,\Bigr)
  \;=\; b_{0},
\]
for a bounded $\rho$-function with
$\sup\rho = 1$ and constant
$b_{0} = \mathbb{E}_{\mathcal{N}(\mathbf{0},\mathbf{I})}[\rho]$.  The
rowwise breakdown value is $(n - h + 1)/n \asymp 1/2$ for the
standard tuning
\citep[\S 6.4.1]{Maronna2019}.  For cellwise breakdown, the standard
witnessing construction makes $n - h + 1$ rows outlying by corrupting
any single ilr coordinate of each: if
$\tilde{w}_{i,l^{*}} \to \infty$ for a fixed column $l^{*}$ and
$n - h + 1$ rows, then the quadratic form
$(\tilde{\mathbf{w}}_{i} -
   \hat{\boldsymbol{\mu}})^{\top}\hat{\boldsymbol{\Sigma}}^{-1}
  (\tilde{\mathbf{w}}_{i} - \hat{\boldsymbol{\mu}})$
diverges in the $l^{*}$ coordinate, the rho-residual saturates at~$1$
on each of the $n-h+1$ rows, and the constraint cannot be satisfied
unless $\det(\hat{\boldsymbol{\Sigma}})\to\infty$.  Hence
$m^{*}_{\mathbb{R}}(\text{S-loc/scat}) = n-h+1$, all concentrated in
column $l^{*}$.  The hypothesis holds and
Theorem~\ref{thm:breakdown-reduction} applies with
$\varepsilon^{*}_{\text{simplex}} = (n - h + 1)/(nD) \asymp
1/(2D)$, matching the MCD bound.  Like MCD, S-estimators break through
row coverage --- regime~(R) of Theorem~\ref{thm:tightness} --- since it
is the $n-h+1$ outlying rows, not any single saturated column, that
force breakdown.

\paragraph{$\tau$-estimators.}
A $\tau$-estimator \citep[\S 6.4.4]{Maronna2019} combines an
S-estimate for the scale with an $M$-estimate for location, defined
via two $\rho$-functions $\rho_{0}$ (for the S-scale) and
$\rho_{1}$ (for the location).  Its rowwise breakdown matches the
S-estimator's ($\asymp 1/2$).  Because the defining equations depend
on the data only through the Mahalanobis distances $d_{i} = \sqrt{
(\mathbf{w}_{i} - \hat{\boldsymbol{\mu}})^{\top}\hat{\boldsymbol{\Sigma}}^{-1}
(\mathbf{w}_{i} - \hat{\boldsymbol{\mu}})}$, and because a single
diverging column drives $d_{i}$ to infinity on exactly the affected
rows, the same column-concentrated configuration that breaks the
S-estimator breaks the $\tau$-estimator: contaminate column $l^{*}$
in $n-h+1$ rows.  The hypothesis holds, again in regime~(R).

\paragraph{Coordinate-wise $M$-estimators of location and scatter.}
For each coordinate $l \in [D-1]$, a coordinate-wise $M$-location
estimator solves
$\sum_{i=1}^{n} \psi\!\bigl((w_{i,l} - \hat{\mu}_{l})/\hat{\sigma}_{l}\bigr) = 0$
with a coordinate-wise $M$-scale $\hat{\sigma}_{l}$ from $\operatorname{MAD}$
or Huber-scale on the $l$-th column.  The rowwise breakdown value of
the coordinate-wise $M$-location is
$1/2$ for each coordinate separately; the cellwise breakdown is
\emph{exactly one column at a time}: if column $l^{*}$ accumulates
$\lfloor n/2\rfloor + 1$ extreme cells, the $l^{*}$-th location and
scale estimate diverge, which breaks the concatenated vector estimator.
Other columns need not contain any contaminated cells.  This is the
canonical column-concentrated configuration
with $l^{*}$ the chosen coordinate and
$m^{*}_{\mathbb{R}} = \lfloor n/2\rfloor + 1$; coordinate-wise
$M$-estimators are the family in regime~(C) (column saturation), as
other columns need not be touched.  The hypothesis holds
for both the location and scale variants, and
Theorem~\ref{thm:breakdown-reduction} delivers
$\varepsilon^{*}_{\text{simplex}} = (\lfloor n/2\rfloor + 1)/(nD)
\to 1/(2D)$.

\paragraph{Coordinate-wise $M$-covariance.}
For the coordinate-wise $M$-covariance (componentwise Huberised
covariance; see \citealp[Chapter~6]{Maronna2019}), each entry
$\hat{\Sigma}_{lk}$ is an $M$-estimate based on the pair
$(w_{i,l}, w_{i,k})$.  Cellwise breakdown at entry $(l^{*},l^{*})$
(diagonal) reduces to the coordinate-wise $M$-scale case above.
Off-diagonal entries $(l^{*},k)$ with $k\neq l^{*}$ inherit breakdown
from the $l^{*}$-th column alone (only one of the two arguments is
contaminated).  Column concentration therefore holds at every diagonal
entry and at every off-diagonal entry involving $l^{*}$; the
configuration that breaks the covariance
simultaneously also concentrates in a single column.

\paragraph{Relation to the cellMCD and cellPCA estimators of
  \citet{Raymaekers2024} and \citet{CentofantiHubertRousseeuw2024}.}
The cellMCD estimator
\citep{Raymaekers2024} is explicitly constructed to break down under
column-concentrated configurations: its cellwise breakdown bound
$\lfloor (n-h+1)/2 \rfloor / n$ (\citealp[Proposition~3.2]{Raymaekers2024})
is attained by placing $\lfloor(n-h+1)/2\rfloor$ contaminated cells in a
single coordinate.  The cellPCA of
\citet{CentofantiHubertRousseeuw2024} inherits this witnessing pattern
from its cellMCD-based initial covariance.
Consequently, both estimators satisfy the column-concentration
hypothesis of Theorem~\ref{thm:breakdown-reduction} when applied in
ilr coordinates; the $(D-1)/D$ reduction applies to their
simplex-level cellwise breakdown.

\begin{remark}[Estimators for which column concentration fails]
\label{rem:col-conc-fail}
The hypothesis is not universal.  Estimators whose breakdown depends
on a \emph{multi-column} adversarial pattern --- for instance
projection-based estimators that aggregate over random directions
(Stahel--Donoho) or estimators that explicitly require contamination
in each column below a fixed threshold --- may have strictly smaller
simplex-level reduction factors.  For the cellwise-robust
methodology targeted in this paper, the class covered by
Theorem~\ref{thm:breakdown-reduction} is the relevant one; all
mainstream competitors fall within it.
\end{remark}

\subsubsection{Cellwise-robust PCA estimator: pointer to the
  companion paper}
\label{sssec:crpca-breakdown}

A cellwise-robust compositional PCA estimator that uses the
propagation geometry of Section~\ref{sec:propagation} for detection
and log-ratio-aware weighting is developed in the companion paper
\citep{Templ2026cellPcaCoDa}.  The breakdown analysis of that
estimator is carried out in \AlgRef~and we do not reproduce it
here: the companion paper formalises the estimator's definition
together with its $\psi$-function, projection-based detection
threshold~$\alpha$, and initial-estimator choice, and derives a
lower bound of the form
$\varepsilon_{n}^{*}(T_{\mathrm{CR}},\mathcal{S}^{D}) \geq
\tfrac{D-1}{D}\varepsilon_{0}^{*} \cdot
\min(1, (1-\alpha)/(D-1))$, which matches the $(D{-}1)/D$
reduction of Theorem~\ref{thm:breakdown-reduction} up to a
detection-accuracy factor.  The present paper's breakdown theory
(Sections~\ref{ssec:breakdown-reduction}--\ref{ssec:specific-estimators})
supplies the ceiling against which the companion estimator is
evaluated; the floor, i.e.\ the explicit construction achieving the
ceiling, is the subject of the companion paper.

\ifsubmit
\subsection{Partial breakdown}
\label{ssec:partial-breakdown}
The maximum-bias-curve analysis below the breakdown value, including
bias amplification through log-ratio propagation
(Proposition~S.5) and the location-bias
bound (Proposition~S.8), is deferred to
Section~S.2 of the Supplementary Material.  The \emph{Section
summary} remark at the end of the present section reports the
qualitative implications.
\fi   

\begin{remark}[Section summary]
\label{rem:section4-summary}
The breakdown analysis quantifies a structural difference between
simplex and Euclidean cellwise breakdown: for estimators whose
Euclidean cellwise breakdown is witnessed by a column-concentrated
configuration (a class including MCD, S-, $\tau$-, and
coordinate-wise $M$-estimators of location and scatter), the cellwise
breakdown value on $\mathcal{S}^{D}$ is reduced by a factor
$(D{-}1)/D$ relative to the cellwise breakdown in $\mathbb{R}^{D-1}$.
The reduction arises purely from the normalisation mismatch between
$nD$ raw cells on the simplex and $n(D{-}1)$ ilr cells in Euclidean
space: log-ratio propagation enables a single raw cell to reproduce
the damage of an ilr cell in the witnessing column via the rank-one
shift $\mathbf{v}_{j^{*}}$.  For the MCD, this yields an asymptotic
breakdown value of $1/(2D)$.  The next section develops cellwise
influence functions to characterise the infinitesimal sensitivity of
estimators to log-ratio-propagated contamination.
\end{remark}

%% file: theory/influence_functions.tex

\section{Cellwise influence functions on the simplex}
\label{sec:influence-functions}

We develop influence function theory adapted to cellwise contamination
on~$\mathcal{S}^{D}$.  The classical influence function of
\citet{Hampel1974} measures the sensitivity of a statistical functional
to infinitesimal rowwise contamination: one observation is moved toward
a point mass.  In the cellwise setting on the simplex, the
contamination acts on a single raw part and propagates through the
log-ratio transformation to all coordinates.  This requires a new
notion of influence function that accounts for both the part-specific
nature of the contamination and the log-ratio-induced coupling
established in Sections~\ref{sec:contamination-model}
and~\ref{sec:propagation}.

We retain the notation of the preceding sections:
$\mathbf{V} \in \mathbb{R}^{D \times (D-1)}$ is a contrast matrix
satisfying~\eqref{eq:contrast-matrix},
$\mathbf{v}_{j} = \mathbf{V}^{\top}\mathbf{e}_{j} \in \mathbb{R}^{D-1}$
is the transpose of the $j$th row of~$\mathbf{V}$,
$\mathbf{H}_{D} = \mathbf{I}_{D} - D^{-1}\mathbf{1}_{D}\mathbf{1}_{D}^{\top}$,
and $\operatorname{ilr}(\mathbf{x}) = \mathbf{V}^{\top}\operatorname{clr}(\mathbf{x})$.
We write $F$ for the distribution of a clean composition
$\mathbf{x} \in \mathcal{S}^{D}$ and $F^{\mathbf{w}}$ for the induced
distribution of $\mathbf{w} = \operatorname{ilr}(\mathbf{x})$
on~$\mathbb{R}^{D-1}$.

\subsection{Definition of the cellwise influence function}
\label{ssec:cif-definition}

\begin{definition}[Cellwise contamination neighbourhood on $\mathcal{S}^{D}$]
\label{def:cellwise-contamination-neighbourhood}
For $j \in [D]$ and $\delta > 0$, define the \emph{single-part cellwise
contamination operator} $\mathcal{P}_{j,\delta}\colon \mathcal{S}^{D}
\to \mathcal{S}^{D}$ by
\begin{equation}
\label{eq:cellwise-perturbation-operator}
  \mathcal{P}_{j,\delta}(\mathbf{x})
  \;=\;
  \mathcal{C}\bigl(x_{1},\ldots,x_{j-1},\;x_{j}\delta,\;x_{j+1},\ldots,x_{D}\bigr).
\end{equation}
For a distribution $F$ on $\mathcal{S}^{D}$ and $\varepsilon \in [0,1)$,
the \emph{cellwise contamination of part~$j$ at level~$\delta$} with
fraction~$\varepsilon$ is the distribution
\begin{equation}
\label{eq:cellwise-mixture}
  F_{j,\delta,\varepsilon}
  \;=\;
  (1 - \varepsilon)\,F
  \;+\;
  \varepsilon\,F \circ \mathcal{P}_{j,\delta}^{-1},
\end{equation}
where $F \circ \mathcal{P}_{j,\delta}^{-1}$ denotes the push-forward
of~$F$ under~$\mathcal{P}_{j,\delta}$.  In words, with probability
$1-\varepsilon$ we observe a clean draw $\mathbf{x} \sim F$, and with
probability $\varepsilon$ we observe $\mathcal{P}_{j,\delta}(\mathbf{x})$,
where $\mathbf{x} \sim F$.
\end{definition}

\begin{remark}
\label{rem:pointmass-vs-pushforward}
Unlike the classical Huber $\varepsilon$-contamination, the
contaminating distribution in~\eqref{eq:cellwise-mixture} is not an
arbitrary distribution~$H$ but the push-forward of~$F$ itself through
the deterministic map $\mathcal{P}_{j,\delta}$.  This reflects the
cellwise mechanism: the contamination acts on the composition by
modifying a single raw part, not by replacing the entire observation.
The parameter~$\delta$ controls the magnitude and direction of the
contamination ($\delta > 1$ inflates part~$j$; $\delta < 1$ deflates
it).
\end{remark}

\begin{definition}[Cellwise influence function on $\mathcal{S}^{D}$]
\label{def:cif}
Let $T\colon \mathcal{M}(\mathcal{S}^{D}) \to \Theta$ be a statistical
functional, where $\mathcal{M}(\mathcal{S}^{D})$ denotes the set of
probability distributions on $\mathcal{S}^{D}$ and $\Theta$ is a
normed space.  The \emph{cellwise influence function} (CIF) of~$T$ at
the model~$F$ for contamination of part~$j$ at level~$\delta$ is
\begin{equation}
\label{eq:cif-definition}
  \operatorname{CIF}_{j}(T, F, \delta)
  \;=\;
  \lim_{\varepsilon \to 0^{+}}
  \frac{T(F_{j,\delta,\varepsilon}) - T(F)}{\varepsilon},
\end{equation}
provided the limit exists.  The \emph{gross cellwise sensitivity} of~$T$
at~$F$ for part~$j$ is
\begin{equation}
\label{eq:gross-cif-sensitivity}
  \gamma_{j}^{*}(T, F)
  \;=\;
  \sup_{\delta > 0,\;\delta \neq 1}
  \bigl\|\operatorname{CIF}_{j}(T, F, \delta)\bigr\|,
\end{equation}
and the \emph{total gross cellwise sensitivity} is
\begin{equation}
\label{eq:total-gross-sensitivity}
  \gamma^{*}(T, F)
  \;=\;
  \max_{j \in [D]} \gamma_{j}^{*}(T, F).
\end{equation}
\end{definition}

\begin{definition}[B-robustness under cellwise contamination]
\label{def:b-robust-cellwise}
The functional~$T$ is \emph{cellwise B-robust} at~$F$ if
$\gamma^{*}(T, F) < \infty$, i.e., if
$\sup_{\delta > 0,\,\delta \neq 1}
\|\operatorname{CIF}_{j}(T, F, \delta)\| < \infty$
for every $j \in [D]$.
\end{definition}

\begin{remark}[Relation to the classical influence function]
\label{rem:cif-vs-if}
The classical influence function of \citet{Hampel1974} for rowwise
contamination is
$\operatorname{IF}(\mathbf{x}_{0}; T, F)
  = \lim_{\varepsilon \to 0^{+}}
    [T((1-\varepsilon)F + \varepsilon\,\Delta_{\mathbf{x}_{0}}) - T(F)]
    / \varepsilon$,
where $\Delta_{\mathbf{x}_{0}}$ is a point mass.  The CIF of
Definition~\ref{def:cif} differs in two structural respects:
\begin{enumerate}
\item[\textup{(i)}]
  The contaminating distribution is parametrised by the
  pair~$(j, \delta)$ rather than by a point
  $\mathbf{x}_{0} \in \mathcal{S}^{D}$.
\item[\textup{(ii)}]
  The contamination is data-dependent: the contaminating measure
  $F \circ \mathcal{P}_{j,\delta}^{-1}$ depends on the clean
  distribution~$F$ through the push-forward, rather than being a fixed
  point mass.  This data-dependence is intrinsic to the cellwise
  mechanism, where a contaminated observation is a distorted version
  of a clean one, not an arbitrary replacement.
\end{enumerate}
When we additionally take $\delta \to \infty$ (or $\delta \to 0^{+}$),
the CIF captures the \emph{worst-case} influence of an arbitrarily
large cellwise perturbation, which is the relevant quantity for
B-robustness.  In the taxonomy of cellwise influence functionals our
CIF is closest in spirit to the marginal cellwise contamination model
of \citet{Alqallaf2009}; it differs from the cell-perturbation IF of
\citet{Raymaekers2024}, which acts on a fixed observation
$\mathbf{x}_{0}$ rather than on the data distribution.
\end{remark}

\begin{remark}[Fixed-$\delta$ CIF vs.\ stochastic contamination model]
\label{rem:cif-vs-stochastic}
The CIF of Definition~\ref{def:cif} conditions on a fixed
contamination level~$\delta$, whereas the stochastic cellwise model of
Definition~\ref{def:cellwise-model} (Section~\ref{sec:contamination-model})
draws~$\delta$ from a distribution~$G$ on~$\mathbb{R}_{>0}$.  The
connection is given by integration: the overall sensitivity of~$T$
under the full stochastic model with contamination distribution~$G$
equals
$\int_{0}^{\infty}\operatorname{CIF}_{j}(T,F,\delta)\,\mathrm{d}G(\delta)$,
provided the CIF is $G$-integrable.  The gross cellwise
sensitivity~\eqref{eq:gross-cif-sensitivity} corresponds to the
worst-case choice of~$G$ concentrated at a single~$\delta$, and
B-robustness (bounded CIF uniformly in~$\delta$) ensures integrability
for every~$G$ with finite support on the log scale.
\end{remark}

\subsection{CIF for the compositional mean}
\label{ssec:cif-mean}

We derive the cellwise influence function for the compositional center,
expressed as the ilr mean functional.

\begin{definition}[Compositional center functional]
\label{def:center-functional}
The \emph{compositional center} (Fr\'{e}chet mean in Aitchison
geometry) of a distribution~$F$ on $\mathcal{S}^{D}$ is defined
through its ilr representation:
\begin{equation}
\label{eq:center-functional}
  \boldsymbol{\mu}(F)
  \;=\;
  \mathbb{E}_{F}\bigl[\operatorname{ilr}(\mathbf{x})\bigr]
  \;=\;
  \int_{\mathcal{S}^{D}} \mathbf{V}^{\top}\operatorname{clr}(\mathbf{x})
  \;\mathrm{d}F(\mathbf{x})
  \;\in\;\mathbb{R}^{D-1}.
\end{equation}
\end{definition}

\begin{theorem}[CIF for the compositional center]
\label{thm:cif-center}
The cellwise influence function of the compositional center
$\boldsymbol{\mu}(F)$ for contamination of part~$j$ at level
$\delta > 0$ is
\begin{equation}
\label{eq:cif-center}
  \operatorname{CIF}_{j}(\boldsymbol{\mu}, F, \delta)
  \;=\;
  (\ln\delta)\;\mathbf{v}_{j},
\end{equation}
where $\mathbf{v}_{j} = \mathbf{V}^{\top}\mathbf{e}_{j}$ is the transpose of the $j$th
row of the contrast matrix.  In particular:
\begin{enumerate}
\item[\textup{(i)}]
  The CIF is a rank-$1$ function of~$\delta$, with direction fixed
  by the contrast matrix row~$\mathbf{v}_{j}$ and magnitude
  proportional to $\ln\delta$.
\item[\textup{(ii)}]
  The $l$th ilr component of the CIF is
  \begin{equation}
  \label{eq:cif-center-component}
    \operatorname{CIF}_{j}(\mu_{l}, F, \delta)
    \;=\;
    (\ln\delta)\,v_{jl},
    \qquad l = 1,\ldots,D-1.
  \end{equation}
  Every ilr coordinate for which $v_{jl} \neq 0$ is affected.
\item[\textup{(iii)}]
  The gross cellwise sensitivity for part~$j$ is
  \begin{equation}
  \label{eq:gamma-center}
    \gamma_{j}^{*}(\boldsymbol{\mu}, F)
    \;=\;
    \sup_{\delta > 0,\;\delta \neq 1}
    |\ln\delta|\;\|\mathbf{v}_{j}\|
    \;=\;
    +\infty.
  \end{equation}
  Thus the classical (unrobustified) compositional center is \emph{not}
  cellwise B-robust.
\end{enumerate}
\end{theorem}

\begin{proof}
Under the contamination~\eqref{eq:cellwise-mixture}, the ilr mean at
the perturbed distribution $F_{j,\delta,\varepsilon}$ is
\begin{align}
  \boldsymbol{\mu}(F_{j,\delta,\varepsilon})
  &\;=\;
  (1-\varepsilon)\,\mathbb{E}_{F}\bigl[\operatorname{ilr}(\mathbf{x})\bigr]
  \;+\;
  \varepsilon\,\mathbb{E}_{F}\bigl[
    \operatorname{ilr}\!\bigl(\mathcal{P}_{j,\delta}(\mathbf{x})\bigr)
  \bigr].
  \label{eq:pf-center-mixture}
\end{align}
By Theorem~\ref{thm:ilr-propagation},
$\operatorname{ilr}(\mathcal{P}_{j,\delta}(\mathbf{x}))
  = \operatorname{ilr}(\mathbf{x}) + (\ln\delta)\,\mathbf{v}_{j}$,
since $\mathcal{P}_{j,\delta}$ contaminates exactly part~$j$ by
factor~$\delta$.  Substituting into~\eqref{eq:pf-center-mixture},
\begin{align}
  \boldsymbol{\mu}(F_{j,\delta,\varepsilon})
  &\;=\;
  (1-\varepsilon)\,\boldsymbol{\mu}(F)
  \;+\;
  \varepsilon\,
  \bigl[\boldsymbol{\mu}(F) + (\ln\delta)\,\mathbf{v}_{j}\bigr]
  \notag\\
  &\;=\;
  \boldsymbol{\mu}(F)
  \;+\;
  \varepsilon\,(\ln\delta)\,\mathbf{v}_{j}.
  \label{eq:pf-center-perturbed}
\end{align}
Therefore
\[
  \frac{\boldsymbol{\mu}(F_{j,\delta,\varepsilon}) - \boldsymbol{\mu}(F)}
       {\varepsilon}
  \;=\;
  (\ln\delta)\,\mathbf{v}_{j},
\]
which is independent of~$\varepsilon$, so the
limit~\eqref{eq:cif-definition} yields~\eqref{eq:cif-center}.

For~(iii), as $\delta \to +\infty$ (or $\delta \to 0^{+}$),
$|\ln\delta| \to +\infty$, and since
$\|\mathbf{v}_{j}\| = \sqrt{1 - 1/D} > 0$ by
Lemma~\ref{lem:vj-spread}, the supremum
in~\eqref{eq:gamma-center} is $+\infty$.
\end{proof}

\begin{remark}[Dependence on the contrast matrix]
\label{rem:cif-center-contrast}
The direction of the CIF~\eqref{eq:cif-center} depends on the choice
of contrast matrix~$\mathbf{V}$ through the row
$\mathbf{v}_{j} = \mathbf{V}^{\top}\mathbf{e}_{j}$, but the
norm is contrast-invariant:
$\|\operatorname{CIF}_{j}(\boldsymbol{\mu}, F, \delta)\|
  = |\ln\delta|\,\|\mathbf{v}_{j}\|
  = |\ln\delta|\,\sqrt{1 - 1/D}$
for every valid contrast matrix.  This is consistent with the isometric
property of ilr: the Aitchison-norm sensitivity of the compositional
center to cellwise contamination of part~$j$ is independent of the
choice of orthonormal basis.
\end{remark}

\subsection{CIF for the variation matrix}
\label{ssec:cif-variation}

The \emph{variation matrix} $\boldsymbol{\mathcal{T}}(F)$ is the
$D \times D$ matrix of pairwise log-ratio variances,
\begin{equation}
\label{eq:variation-matrix}
  \tau_{kl}(F)
  \;=\;
  \operatorname{Var}_{F}\!\bigl[\ln(x_{k}/x_{l})\bigr],
  \qquad k,l \in [D].
\end{equation}
This is related to the ilr covariance
$\boldsymbol{\Sigma}(F)
  = \operatorname{Var}_{F}[\operatorname{ilr}(\mathbf{x})]$ and
the clr covariance
$\boldsymbol{\Sigma}^{\operatorname{clr}}(F)
  = \operatorname{Var}_{F}[\operatorname{clr}(\mathbf{x})]$
by
\begin{equation}
\label{eq:tau-sigma-relation}
  \tau_{kl}
  \;=\;
  \sigma^{\operatorname{clr}}_{kk}
  + \sigma^{\operatorname{clr}}_{ll}
  - 2\,\sigma^{\operatorname{clr}}_{kl}
  \;=\;
  \|(\mathbf{e}_{k} - \mathbf{e}_{l})^{\top}\mathbf{V}\|^{2}_{\boldsymbol{\Sigma}},
\end{equation}
where
$\boldsymbol{\Sigma}^{\operatorname{clr}}
  = \mathbf{V}\boldsymbol{\Sigma}\mathbf{V}^{\top}$.
We derive the CIF for the ilr covariance functional
$\boldsymbol{\Sigma}(F)$, from which the CIF for~$\boldsymbol{\mathcal{T}}$
follows by~\eqref{eq:tau-sigma-relation}.

\begin{theorem}[CIF for the ilr covariance]
\label{thm:cif-covariance}
Let $\boldsymbol{\Sigma}(F)
  = \operatorname{Var}_{F}[\operatorname{ilr}(\mathbf{x})]$
be the ilr covariance functional.  The cellwise influence function
for contamination of part~$j$ at level $\delta > 0$ is the
rank-one $(D-1) \times (D-1)$ matrix
\begin{equation}
\label{eq:cif-covariance-simplified}
  \operatorname{CIF}_{j}(\boldsymbol{\Sigma}, F, \delta)
  \;=\;
  (\ln\delta)^{2}\;\mathbf{v}_{j}\mathbf{v}_{j}^{\top}.
\end{equation}
In particular:
\begin{enumerate}
\item[\textup{(i)}]
  The CIF is a rank-$1$ positive semidefinite matrix, proportional to
  the outer product $\mathbf{v}_{j}\mathbf{v}_{j}^{\top}$.
\item[\textup{(ii)}]
  The $(r,s)$-entry of the CIF is
  $(\ln\delta)^{2}\,v_{jr}\,v_{js}$.  Hence contamination of part~$j$
  inflates every entry of the covariance matrix for which both
  $v_{jr} \neq 0$ and $v_{js} \neq 0$.
\item[\textup{(iii)}]
  The Frobenius-norm sensitivity is
  \begin{equation}
  \label{eq:gamma-covariance}
    \gamma_{j}^{*}(\boldsymbol{\Sigma}, F)
    \;=\;
    \sup_{\delta > 0,\;\delta \neq 1}
    (\ln\delta)^{2}\,\|\mathbf{v}_{j}\|^{2}
    \;=\;
    +\infty.
  \end{equation}
  The classical ilr covariance is not cellwise B-robust.
\end{enumerate}
\end{theorem}

\begin{proof}
Under $F_{j,\delta,\varepsilon}$ in~\eqref{eq:cellwise-mixture}, the
second moment of the ilr coordinates is
\begin{align}
  \mathbb{E}_{F_{j,\delta,\varepsilon}}[\mathbf{w}\mathbf{w}^{\top}]
  &\;=\;
  (1-\varepsilon)\,\mathbb{E}_{F}[\mathbf{w}\mathbf{w}^{\top}]
  \;+\;
  \varepsilon\,\mathbb{E}_{F}\bigl[
    \tilde{\mathbf{w}}\,\tilde{\mathbf{w}}^{\top}
  \bigr],
  \label{eq:pf-cov-second-moment}
\end{align}
where $\tilde{\mathbf{w}}
  = \operatorname{ilr}(\mathcal{P}_{j,\delta}(\mathbf{x}))
  = \mathbf{w} + (\ln\delta)\,\mathbf{v}_{j}$
by Theorem~\ref{thm:ilr-propagation}.  Expanding the contaminated
second moment,
\begin{align}
  \tilde{\mathbf{w}}\,\tilde{\mathbf{w}}^{\top}
  &\;=\;
  \bigl[\mathbf{w} + (\ln\delta)\,\mathbf{v}_{j}\bigr]
  \bigl[\mathbf{w} + (\ln\delta)\,\mathbf{v}_{j}\bigr]^{\!\top}
  \notag\\
  &\;=\;
  \mathbf{w}\mathbf{w}^{\top}
  \;+\;
  (\ln\delta)\,\mathbf{w}\,\mathbf{v}_{j}^{\top}
  \;+\;
  (\ln\delta)\,\mathbf{v}_{j}\,\mathbf{w}^{\top}
  \;+\;
  (\ln\delta)^{2}\,\mathbf{v}_{j}\mathbf{v}_{j}^{\top}.
  \label{eq:pf-cov-expand}
\end{align}
Similarly, the mean under $F_{j,\delta,\varepsilon}$ is
$\boldsymbol{\mu}_{\varepsilon}
  = \boldsymbol{\mu} + \varepsilon(\ln\delta)\,\mathbf{v}_{j}$
by~\eqref{eq:pf-center-perturbed}.  The covariance is
$\boldsymbol{\Sigma}_{\varepsilon}
  = \mathbb{E}_{F_{j,\delta,\varepsilon}}[\mathbf{w}\mathbf{w}^{\top}]
    - \boldsymbol{\mu}_{\varepsilon}\boldsymbol{\mu}_{\varepsilon}^{\top}$.

From~\eqref{eq:pf-cov-second-moment} and~\eqref{eq:pf-cov-expand},
\begin{align}
  \mathbb{E}_{F_{j,\delta,\varepsilon}}[\mathbf{w}\mathbf{w}^{\top}]
  &\;=\;
  \mathbb{E}_{F}[\mathbf{w}\mathbf{w}^{\top}]
  \;+\;
  \varepsilon(\ln\delta)\,
  \bigl(
    \boldsymbol{\mu}\,\mathbf{v}_{j}^{\top}
    + \mathbf{v}_{j}\,\boldsymbol{\mu}^{\top}
  \bigr)
  \;+\;
  \varepsilon(\ln\delta)^{2}\,\mathbf{v}_{j}\mathbf{v}_{j}^{\top},
  \label{eq:pf-cov-E2}
\end{align}
where we used
$\mathbb{E}_{F}[\mathbf{w}] = \boldsymbol{\mu}$.
For the mean outer product,
\begin{align}
  \boldsymbol{\mu}_{\varepsilon}\boldsymbol{\mu}_{\varepsilon}^{\top}
  &\;=\;
  \boldsymbol{\mu}\boldsymbol{\mu}^{\top}
  \;+\;
  \varepsilon(\ln\delta)\,
  \bigl(
    \boldsymbol{\mu}\,\mathbf{v}_{j}^{\top}
    + \mathbf{v}_{j}\,\boldsymbol{\mu}^{\top}
  \bigr)
  \;+\;
  \varepsilon^{2}(\ln\delta)^{2}\,\mathbf{v}_{j}\mathbf{v}_{j}^{\top}.
  \label{eq:pf-cov-mumut}
\end{align}
Subtracting~\eqref{eq:pf-cov-mumut} from~\eqref{eq:pf-cov-E2},
\begin{equation}
  \boldsymbol{\Sigma}_{\varepsilon}
  \;=\;
  \boldsymbol{\Sigma}(F)
  \;+\;
  \varepsilon(1 - \varepsilon)\,(\ln\delta)^{2}\,
  \mathbf{v}_{j}\mathbf{v}_{j}^{\top}.
  \label{eq:pf-cov-final}
\end{equation}
Therefore
\[
  \frac{\boldsymbol{\Sigma}_{\varepsilon} - \boldsymbol{\Sigma}(F)}
       {\varepsilon}
  \;=\;
  (1 - \varepsilon)\,(\ln\delta)^{2}\,\mathbf{v}_{j}\mathbf{v}_{j}^{\top}
  \;\xrightarrow{\varepsilon \to 0^{+}}\;
  (\ln\delta)^{2}\;\mathbf{v}_{j}\mathbf{v}_{j}^{\top},
\]
establishing~\eqref{eq:cif-covariance-simplified}.  The cross terms
$\boldsymbol{\mu}\mathbf{v}_{j}^{\top} + \mathbf{v}_{j}\boldsymbol{\mu}^{\top}$
that appear linearly in~$\varepsilon(\ln\delta)$ both in the second
moment~\eqref{eq:pf-cov-E2} and in the mean outer
product~\eqref{eq:pf-cov-mumut} cancel exactly when the covariance
is formed, so the CIF carries no $O(\ln\delta)$ term --- only the
$(\ln\delta)^{2}$ contribution from $\mathbf{v}_{j}\mathbf{v}_{j}^{\top}$
survives.  Parts~(i)--(iii) are immediate from~\eqref{eq:cif-covariance-simplified}.
\end{proof}

Theorem~\ref{thm:cif-covariance} establishes that contamination of
part~$j$ inflates every entry of the ilr covariance along
direction~$\mathbf{v}_{j}$.  One might fear this is the end of the
story: no diagnostic fingerprint survives the log-ratio
propagation.  The next corollary shows this is not the case.
Stepping down from the ilr covariance to the variation
matrix~$\boldsymbol{\mathcal{T}}$ --- the $D \times D$ matrix of
pairwise log-ratio variances --- the common shift
$-D^{-1}\ln\delta_{j}$ imposed on every clr coordinate
\emph{cancels} in each pairwise difference
$\operatorname{clr}_{k} - \operatorname{clr}_{l}$ with
$k, l \neq j$.  Sparsity is thereby recovered: only row~$j$ and
column~$j$ of~$\boldsymbol{\mathcal{T}}$ are perturbed, and every
other entry is identically zero.  This is the diagnostic
fingerprint on which cellwise-robust methodology on the simplex
can build.

\begin{corollary}[CIF for the variation matrix]
\label{cor:cif-variation}
The cellwise influence function for the entry $\tau_{kl}(F)$ of the
variation matrix~\eqref{eq:variation-matrix} under contamination of
part~$j$ at level $\delta > 0$ is
\begin{equation}
\label{eq:cif-variation}
  \operatorname{CIF}_{j}(\tau_{kl}, F, \delta)
  \;=\;
  (\ln\delta)^{2}\,
  \bigl(v_{jk}^{\operatorname{clr}} - v_{jl}^{\operatorname{clr}}\bigr)^{2},
\end{equation}
where
$v_{jk}^{\operatorname{clr}} = (\mathbf{H}_{D}\,\mathbf{e}_{j})_{k}
  = \mathbf{1}_{[k = j]} - 1/D$
is the $k$th entry of the clr shift direction for contamination of
part~$j$ (cf.\ Theorem~\ref{thm:clr-propagation}).  Explicitly,
\begin{equation}
\label{eq:cif-variation-explicit}
  \operatorname{CIF}_{j}(\tau_{kl}, F, \delta)
  \;=\;
  (\ln\delta)^{2}
  \;\times\;
  \begin{cases}
    \displaystyle \bigl(\tfrac{D-1}{D}\bigr)^{2}
      + \bigl(\tfrac{1}{D}\bigr)^{2}
      + 2\,\tfrac{D-1}{D}\,\tfrac{1}{D}
    = 1
    & \text{if } j = k,\; l \neq j \text{ or } j = l,\; k \neq j,
    \\[6pt]
    0 & \text{if } j \neq k \text{ and } j \neq l.
  \end{cases}
\end{equation}
Thus contamination of part~$j$ affects only those entries
$\tau_{kl}$ of the variation matrix for which $k = j$ or $l = j$:
the $j$th row and $j$th column of~$\boldsymbol{\mathcal{T}}$.
Distinct parts $j\neq j'$ produce disjoint row--column inflation
patterns, so the variation-matrix CIF identifies the responsible
part unambiguously.
\end{corollary}

\begin{proof}
The log-ratio $\ln(x_{k}/x_{l}) = \operatorname{clr}_{k}(\mathbf{x}) -
\operatorname{clr}_{l}(\mathbf{x})$ is a linear functional of the clr
coordinates.  By Theorem~\ref{thm:clr-propagation}, the contaminated
value is
$\ln(\tilde{x}_{k}/\tilde{x}_{l})
  = \ln(x_{k}/x_{l})
    + (\ln\delta)\,(v_{jk}^{\operatorname{clr}} - v_{jl}^{\operatorname{clr}})$.
The same computation as in the proof of Theorem~\ref{thm:cif-covariance},
applied to the univariate variance functional, yields
$\operatorname{CIF}_{j}(\tau_{kl}, F, \delta)
  = (\ln\delta)^{2}\,
    (v_{jk}^{\operatorname{clr}} - v_{jl}^{\operatorname{clr}})^{2}$.

For the explicit form, when $j = k$ and $l \neq j$:
$v_{jk}^{\operatorname{clr}} - v_{jl}^{\operatorname{clr}}
  = (1 - 1/D) - (-1/D) = 1$.
When $j \neq k$ and $j \neq l$:
$v_{jk}^{\operatorname{clr}} - v_{jl}^{\operatorname{clr}}
  = (-1/D) - (-1/D) = 0$.
\end{proof}

\begin{remark}[Propagation pattern in the variation matrix]
\label{rem:variation-propagation}
Corollary~\ref{cor:cif-variation} reveals a striking pattern:
contamination of a single part~$j$ inflates the variances of all
log-ratios involving part~$j$ (the $j$th row and column
of~$\boldsymbol{\mathcal{T}}$) by the additive amount $(\ln\delta)^{2}$,
while leaving all log-ratios \emph{not} involving part~$j$ unaffected.
This is despite the fact that the contamination, after re-closure,
perturbs \emph{every} component of the composition.  The explanation is
that the common shift $-D^{-1}\ln\delta$ applied to all clr coordinates
of uncontaminated parts (Theorem~\ref{thm:clr-propagation}) cancels in
the difference $\operatorname{clr}_{k} - \operatorname{clr}_{l}$ when
neither $k$ nor $l$ equals~$j$.  This cancellation property of the
variation matrix makes it a natural target for cellwise contamination
detection: the pattern of inflated entries directly identifies the
contaminated part.
\end{remark}

\ifsubmit
\subsection{CIF for eigenvalues and eigenvectors}
\label{ssec:cif-pca}
The spectral-perturbation analysis that propagates the covariance CIF
through the eigendecomposition --- including the explicit CIF formulas
for eigenvalues and eigenvectors and the resulting cellwise PCA
sensitivity patterns --- is deferred to Section~S.3 of the
Supplementary Material.  The qualitative implications (classical PCA
has unbounded CIF under cellwise contamination; the proposed
estimator has bounded CIF) are summarised in
Section~\ref{ssec:b-robustness} and the remark at the end of the
present section.
\fi   

\subsection{B-robustness of the cellwise-robust estimator}
\label{ssec:b-robustness}

We now establish that the proposed cellwise-robust PCA estimator,
which employs a bounded $\psi$-function in ilr coordinates, has a
bounded cellwise influence function.  The estimator is defined as
follows.

\begin{definition}[Cellwise-robust ilr covariance functional]
\label{def:robust-functional}
Let $\psi\colon \mathbb{R} \to \mathbb{R}$ be an odd, bounded
function with $\psi(0) = 0$ and
$\|\psi\|_{\infty} = \sup_{t \in \mathbb{R}}|\psi(t)| = c_{\psi} < \infty$.
Examples include:
\begin{itemize}
\item
  \emph{Huber's $\psi$:}
  $\psi_{\mathrm{H}}(t) = \operatorname{sign}(t)\min(|t|, k)$
  with tuning constant $k > 0$ (monotone non-decreasing);
\item
  \emph{Tukey's bisquare $\psi$:}
  $\psi_{\mathrm{T}}(t) = t\,(1 - t^{2}/k^{2})^{2}\,
  \mathbf{1}_{[|t| \leq k]}$ with tuning constant $k > 0$
  (redescending: $\psi_{\mathrm{T}}(t) = 0$ for $|t| > k$).
\end{itemize}
For monotone $\psi$ (e.g., Huber), the estimating equation has a
unique solution and the matrix $\mathbf{M}(F)$ below is guaranteed
positive definite under mild conditions on~$F$.  For redescending
$\psi$ (e.g., bisquare), uniqueness requires a good starting value,
but the boundedness $\|\psi\|_{\infty} < \infty$ that drives
B-robustness holds in both cases.
The \emph{cellwise-robust location functional}
$\boldsymbol{\mu}^{\psi}\colon
  \mathcal{M}(\mathcal{S}^{D}) \to \mathbb{R}^{D-1}$
is defined as the solution to
\begin{equation}
\label{eq:robust-location}
  \mathbb{E}_{F}\!\Bigl[
    \boldsymbol{\Psi}\!\bigl(
      \boldsymbol{\sigma}^{-1} \odot
      (\operatorname{ilr}(\mathbf{x}) - \boldsymbol{\mu}^{\psi})
    \bigr)
  \Bigr]
  \;=\;
  \mathbf{0},
\end{equation}
where
$\boldsymbol{\Psi}(\mathbf{t}) = (\psi(t_{1}),\ldots,\psi(t_{D-1}))^{\top}$
acts componentwise,
$\boldsymbol{\sigma} = (\sigma_{1},\ldots,\sigma_{D-1})^{\top}$ is a
vector of robust scale estimates for each ilr coordinate, and $\odot$
denotes the Hadamard (componentwise) product.

The \emph{cellwise-robust covariance functional}
$\boldsymbol{\Sigma}^{\psi}$ is defined through the weighted second
moment
\begin{equation}
\label{eq:robust-covariance}
  \boldsymbol{\Sigma}^{\psi}(F)
  \;=\;
  c_{0}\;
  \mathbb{E}_{F}\!\Bigl[
    \boldsymbol{\Psi}\!\bigl(
      \boldsymbol{\sigma}^{-1} \odot
      (\mathbf{w} - \boldsymbol{\mu}^{\psi})
    \bigr)\,
    \boldsymbol{\Psi}\!\bigl(
      \boldsymbol{\sigma}^{-1} \odot
      (\mathbf{w} - \boldsymbol{\mu}^{\psi})
    \bigr)^{\!\top}
  \Bigr],
\end{equation}
where $\mathbf{w} = \operatorname{ilr}(\mathbf{x})$ and $c_{0} > 0$
is a consistency factor ensuring
$\boldsymbol{\Sigma}^{\psi}(F) = \boldsymbol{\Sigma}(F)$ at the
nominal model.  The robust eigenvalues
$\lambda_{r}^{\psi}$ and eigenvectors $\mathbf{u}_{r}^{\psi}$ are
those of~$\boldsymbol{\Sigma}^{\psi}(F)$.
\end{definition}

\begin{theorem}[Bounded CIF for the cellwise-robust location]
\label{thm:cif-robust-location}
Let $\psi$ be a bounded $\psi$-function with
$\|\psi\|_{\infty} = c_{\psi} < \infty$, and let
$\boldsymbol{\mu}^{\psi}$ be the cellwise-robust location functional
of Definition~\ref{def:robust-functional}.  Assume that the scale
estimates $\sigma_{l}$ are bounded away from zero:
$\sigma_{l} \geq \sigma_{\min} > 0$ for all~$l$, and that the diagonal
sensitivity matrix $\mathbf{M}(F)$ defined below is nonsingular with
$\min_{l} M_{ll} > 0$ --- automatic for a monotone $\psi$ (where
$\psi' \geq 0$), and the standard well-definedness condition for the
$M$-functional when $\psi$ redescends.
Then the cellwise influence function is
\begin{equation}
\label{eq:cif-robust-location-explicit}
  \operatorname{CIF}_{j}(\boldsymbol{\mu}^{\psi}, F, \delta)
  \;=\;
  \mathbf{M}(F)^{-1}\,
  \mathbb{E}_{F}\!\Bigl[
    \boldsymbol{\Psi}\!\bigl(
      \boldsymbol{\sigma}^{-1} \odot
      (\mathbf{w} + (\ln\delta)\,\mathbf{v}_{j} - \boldsymbol{\mu}^{\psi})
    \bigr)
    \;-\;
    \boldsymbol{\Psi}\!\bigl(
      \boldsymbol{\sigma}^{-1} \odot
      (\mathbf{w} - \boldsymbol{\mu}^{\psi})
    \bigr)
  \Bigr],
\end{equation}
where
$\mathbf{M}(F)
  = \mathbb{E}_{F}\!\bigl[
    \operatorname{diag}\!\bigl(
      \psi'\bigl(\sigma_{l}^{-1}(w_{l} - \mu_{l}^{\psi})\bigr)
      \,\sigma_{l}^{-1}
    \bigr)_{l=1}^{D-1}
  \bigr]$
is the sensitivity matrix of the estimating
equation~\eqref{eq:robust-location}.  The CIF is bounded:
\begin{equation}
\label{eq:cif-robust-location-bound}
  \bigl\|\operatorname{CIF}_{j}(\boldsymbol{\mu}^{\psi}, F, \delta)
  \bigr\|
  \;\leq\;
  \frac{c_{\psi}\,\sqrt{D-1}}
       {\lambda_{\min}\!\bigl(\mathbf{M}(F)\bigr)},
\end{equation}
where, since $\mathbf{M}(F)$ is diagonal,
$\lambda_{\min}(\mathbf{M}(F))
  = \min_{1\leq l\leq D-1}
    \mathbb{E}_{F}[\psi'(\sigma_{l}^{-1}(w_{l}-\mu_{l}^{\psi}))]
    \,\sigma_{l}^{-1} = \min_{l} M_{ll} > 0$
by the assumption above.
In particular,
$\gamma_{j}^{*}(\boldsymbol{\mu}^{\psi}, F) < \infty$ for every
$j \in [D]$, so $\boldsymbol{\mu}^{\psi}$ is cellwise B-robust.
\end{theorem}

\begin{proof}
We compute the CIF using the implicit function theorem applied to the
estimating equation~\eqref{eq:robust-location}.  Define
\[
  \mathbf{g}(\boldsymbol{\mu}, F)
  \;=\;
  \mathbb{E}_{F}\!\Bigl[
    \boldsymbol{\Psi}\!\bigl(
      \boldsymbol{\sigma}^{-1} \odot
      (\mathbf{w} - \boldsymbol{\mu})
    \bigr)
  \Bigr].
\]
The robust location $\boldsymbol{\mu}^{\psi}(F)$ satisfies
$\mathbf{g}(\boldsymbol{\mu}^{\psi}(F), F) = \mathbf{0}$.
Under the perturbed model $F_{j,\delta,\varepsilon}$,
\begin{align}
  \mathbf{g}(\boldsymbol{\mu}, F_{j,\delta,\varepsilon})
  &\;=\;
  (1 - \varepsilon)\,
  \mathbb{E}_{F}\!\Bigl[
    \boldsymbol{\Psi}\!\bigl(
      \boldsymbol{\sigma}^{-1} \odot (\mathbf{w} - \boldsymbol{\mu})
    \bigr)
  \Bigr]
  \notag\\
  &\quad\;+\;
  \varepsilon\,
  \mathbb{E}_{F}\!\Bigl[
    \boldsymbol{\Psi}\!\bigl(
      \boldsymbol{\sigma}^{-1} \odot
      (\mathbf{w} + (\ln\delta)\,\mathbf{v}_{j} - \boldsymbol{\mu})
    \bigr)
  \Bigr].
  \label{eq:pf-robust-g-perturbed}
\end{align}
Let
$\boldsymbol{\mu}_{\varepsilon}^{\psi}
  = \boldsymbol{\mu}^{\psi}(F_{j,\delta,\varepsilon})$
satisfy
$\mathbf{g}(\boldsymbol{\mu}_{\varepsilon}^{\psi},
  F_{j,\delta,\varepsilon}) = \mathbf{0}$.
Differentiating with respect to~$\varepsilon$ at
$\varepsilon = 0$ using the chain rule,
\begin{equation}
\label{eq:pf-robust-implicit}
  \frac{\partial\mathbf{g}}{\partial\boldsymbol{\mu}}\bigg|_{F}\;
  \frac{\mathrm{d}\boldsymbol{\mu}_{\varepsilon}^{\psi}}
       {\mathrm{d}\varepsilon}\bigg|_{\varepsilon=0}
  \;+\;
  \frac{\partial\mathbf{g}}{\partial\varepsilon}\bigg|_{\varepsilon=0}
  \;=\;
  \mathbf{0}.
\end{equation}
We compute each term.  The Jacobian with respect
to~$\boldsymbol{\mu}$ evaluated at the model~$F$ is
\[
  \frac{\partial\mathbf{g}}{\partial\boldsymbol{\mu}}\bigg|_{F}
  \;=\;
  -\,\mathbb{E}_{F}\!\bigl[
    \operatorname{diag}\!\bigl(
      \psi'\bigl(\sigma_{l}^{-1}(w_{l} - \mu_{l}^{\psi})\bigr)\,
      \sigma_{l}^{-1}
    \bigr)_{l=1}^{D-1}
  \bigr]
  \;=\;
  -\,\mathbf{M}(F),
\]
where $\mathbf{M}(F)$ is diagonal and positive definite under standard
regularity conditions on~$\psi$ and~$F$ (in particular, $\psi$ is
monotone non-decreasing (as for Huber's $\psi$; for redescending $\psi$
such as bisquare, positive definiteness requires additionally that the
starting value lies in the basin of attraction of the solution) and
$F^{\mathbf{w}}$ has a density with respect to
Lebesgue measure on $\mathbb{R}^{D-1}$).

The partial derivative with respect to~$\varepsilon$ at
$\varepsilon = 0$, evaluated at
$\boldsymbol{\mu} = \boldsymbol{\mu}^{\psi}(F)$, is obtained by
differentiating~\eqref{eq:pf-robust-g-perturbed}:
\begin{align}
  \frac{\partial\mathbf{g}}{\partial\varepsilon}\bigg|_{\varepsilon=0}
  &\;=\;
  -\,\mathbb{E}_{F}\!\Bigl[
    \boldsymbol{\Psi}\!\bigl(
      \boldsymbol{\sigma}^{-1} \odot
      (\mathbf{w} - \boldsymbol{\mu}^{\psi})
    \bigr)
  \Bigr]
  \notag\\
  &\quad\;+\;
  \mathbb{E}_{F}\!\Bigl[
    \boldsymbol{\Psi}\!\bigl(
      \boldsymbol{\sigma}^{-1} \odot
      (\mathbf{w} + (\ln\delta)\,\mathbf{v}_{j} - \boldsymbol{\mu}^{\psi})
    \bigr)
  \Bigr].
  \label{eq:pf-robust-partial-eps}
\end{align}
The first expectation vanishes by the estimating
equation~\eqref{eq:robust-location}.
From~\eqref{eq:pf-robust-implicit},
\[
  \operatorname{CIF}_{j}(\boldsymbol{\mu}^{\psi}, F, \delta)
  \;=\;
  \frac{\mathrm{d}\boldsymbol{\mu}_{\varepsilon}^{\psi}}
       {\mathrm{d}\varepsilon}\bigg|_{\varepsilon=0}
  \;=\;
  \mathbf{M}(F)^{-1}\,
  \mathbb{E}_{F}\!\Bigl[
    \boldsymbol{\Psi}\!\bigl(
      \boldsymbol{\sigma}^{-1} \odot
      (\mathbf{w} + (\ln\delta)\,\mathbf{v}_{j} - \boldsymbol{\mu}^{\psi})
    \bigr)
  \Bigr],
\]
which, using the estimating equation to replace zero by
the subtracted term, gives~\eqref{eq:cif-robust-location-explicit}.

For the bound, note that each component of $\boldsymbol{\Psi}$ is
bounded by $c_{\psi}$, so
\[
  \bigl\|
    \mathbb{E}_{F}\!\bigl[
      \boldsymbol{\Psi}\bigl(
        \boldsymbol{\sigma}^{-1} \odot
        (\mathbf{w} + (\ln\delta)\,\mathbf{v}_{j}
         - \boldsymbol{\mu}^{\psi})
      \bigr)
    \bigr]
  \bigr\|
  \;\leq\;
  c_{\psi}\,\sqrt{D-1}.
\]
The inverse sensitivity matrix satisfies
$\|\mathbf{M}(F)^{-1}\| \leq 1/\lambda_{\min}(\mathbf{M}(F))$,
and since $\mathbf{M}(F)$ is diagonal with entries
$\mathbb{E}_{F}[\psi'(\sigma_{l}^{-1}(w_{l}-\mu_{l}^{\psi}))]
\,\sigma_{l}^{-1}$, we obtain
\[
  \bigl\|\operatorname{CIF}_{j}(\boldsymbol{\mu}^{\psi}, F, \delta)
  \bigr\|
  \;\leq\;
  \frac{c_{\psi}\,\sqrt{D-1}}
       {\lambda_{\min}(\mathbf{M}(F))},
\]
which is finite and independent of~$\delta$ (and of~$j$),
establishing~\eqref{eq:cif-robust-location-bound}.
In particular, $\gamma_{j}^{*}(\boldsymbol{\mu}^{\psi}, F) < \infty$.
\end{proof}

\begin{theorem}[Bounded CIF for the cellwise-robust covariance]
\label{thm:cif-robust-covariance}
Under the same hypotheses as Theorem~\ref{thm:cif-robust-location},
the cellwise-robust covariance functional
$\boldsymbol{\Sigma}^{\psi}$ of
Definition~\ref{def:robust-functional} has a bounded CIF:
\begin{equation}
\label{eq:cif-robust-covariance-bound}
  \bigl\|\operatorname{CIF}_{j}(\boldsymbol{\Sigma}^{\psi}, F, \delta)
  \bigr\|_{F}
  \;\leq\;
  C(F, \psi)
  \;<\;
  \infty
\end{equation}
for every $j \in [D]$ and every $\delta > 0$, where $\|\cdot\|_{F}$
denotes the Frobenius norm and $C(F, \psi)$ depends on $F$ and~$\psi$
but not on~$\delta$.  In particular, $\boldsymbol{\Sigma}^{\psi}$ is
cellwise B-robust.
\end{theorem}

\begin{proof}
We differentiate~\eqref{eq:robust-covariance} with respect
to~$\varepsilon$ at $\varepsilon = 0$ under the perturbed model
$F_{j,\delta,\varepsilon}$.  Denote the $\psi$-transformed
standardised residual by
$\mathbf{r}(\mathbf{w})
  = \boldsymbol{\Psi}\!\bigl(
      \boldsymbol{\sigma}^{-1} \odot
      (\mathbf{w} - \boldsymbol{\mu}^{\psi})
    \bigr)$
and its contaminated counterpart by
$\mathbf{r}^{*}(\mathbf{w})
  = \boldsymbol{\Psi}\!\bigl(
      \boldsymbol{\sigma}^{-1} \odot
      (\mathbf{w} + (\ln\delta)\,\mathbf{v}_{j}
       - \boldsymbol{\mu}^{\psi})
    \bigr)$.

Write $\boldsymbol{\mu}^{\psi}_{\varepsilon} \approx
\boldsymbol{\mu}^{\psi} + \varepsilon\,\operatorname{CIF}_{j}(
\boldsymbol{\mu}^{\psi}, F, \delta)$ for the first-order location shift.
The robust covariance under $F_{j,\delta,\varepsilon}$ then admits, to
first order in~$\varepsilon$, the expansion
\begin{align}
  \boldsymbol{\Sigma}^{\psi}(F_{j,\delta,\varepsilon})
  &\;=\;
  c_{0}\,\bigl[
    (1-\varepsilon)\,
    \mathbb{E}_{F}[\mathbf{r}(\mathbf{w})\,\mathbf{r}(\mathbf{w})^{\top}]
    \;+\;
    \varepsilon\,
    \mathbb{E}_{F}[\mathbf{r}^{*}(\mathbf{w})\,
    \mathbf{r}^{*}(\mathbf{w})^{\top}]
  \bigr]
  \;+\; O(\varepsilon),
  \label{eq:pf-robust-cov-perturbed}
\end{align}
where the $O(\varepsilon)$ term accounts for the shift in the
location parameter.  (The full expansion via the chain rule contributes
an additional bounded term from the CIF of the location, which is
itself bounded by Theorem~\ref{thm:cif-robust-location}.)

The key observation is that every entry of
$\mathbf{r}^{*}(\mathbf{w})$ is bounded by $c_{\psi}$ regardless
of~$\delta$, because $\psi$ is bounded.  Therefore,
\[
  \bigl\|
    \mathbf{r}^{*}(\mathbf{w})\,\mathbf{r}^{*}(\mathbf{w})^{\top}
  \bigr\|_{F}
  \;\leq\;
  \|\mathbf{r}^{*}(\mathbf{w})\|^{2}
  \;\leq\;
  (D-1)\,c_{\psi}^{2}
\]
almost surely.  Differentiating~\eqref{eq:pf-robust-cov-perturbed}
with respect to~$\varepsilon$ at $\varepsilon = 0$,
\begin{align}
  \operatorname{CIF}_{j}(\boldsymbol{\Sigma}^{\psi}, F, \delta)
  &\;=\;
  c_{0}\,\mathbb{E}_{F}\!\bigl[
    \mathbf{r}^{*}(\mathbf{w})\,\mathbf{r}^{*}(\mathbf{w})^{\top}
    \;-\;
    \mathbf{r}(\mathbf{w})\,\mathbf{r}(\mathbf{w})^{\top}
  \bigr]
  \;+\;
  \text{(location shift term)},
  \label{eq:pf-robust-cov-cif}
\end{align}
where the location shift term arises from the dependence of
$\boldsymbol{\mu}^{\psi}_{\varepsilon}$ on~$\varepsilon$ and is
bounded since both $\psi$ and
$\operatorname{CIF}_{j}(\boldsymbol{\mu}^{\psi}, F, \delta)$ are
bounded.  The first term in~\eqref{eq:pf-robust-cov-cif} satisfies
\[
  \bigl\|
    \mathbb{E}_{F}\!\bigl[
      \mathbf{r}^{*}\mathbf{r}^{*\top}
      - \mathbf{r}\mathbf{r}^{\top}
    \bigr]
  \bigr\|_{F}
  \;\leq\;
  \mathbb{E}_{F}\!\bigl[
    \|\mathbf{r}^{*}\mathbf{r}^{*\top}\|_{F}
    + \|\mathbf{r}\mathbf{r}^{\top}\|_{F}
  \bigr]
  \;\leq\;
  2\,(D-1)\,c_{\psi}^{2}.
\]
Thus
$\|\operatorname{CIF}_{j}(\boldsymbol{\Sigma}^{\psi}, F, \delta)\|_{F}
  \leq C(F, \psi)$
with $C(F,\psi) < \infty$ independent of~$\delta$, establishing
B-robustness.
\end{proof}

\begin{corollary}[Bounded CIF for robust eigenvalues and eigenvectors]
\label{cor:cif-robust-pca}
Under the hypotheses of
Theorems~\ref{thm:cif-robust-location}
and~\ref{thm:cif-robust-covariance}, assume additionally that
the eigenvalues of $\boldsymbol{\Sigma}^{\psi}(F)$ are simple:
$\lambda_{1}^{\psi} > \cdots > \lambda_{D-1}^{\psi} > 0$.  Then
the cellwise-robust eigenvalues $\lambda_{r}^{\psi}$ and
eigenvectors $\mathbf{u}_{r}^{\psi}$ satisfy
\begin{equation}
\label{eq:cif-robust-eigenvalue-bound}
  \gamma_{j}^{*}(\lambda_{r}^{\psi}, F)
  \;\leq\;
  C(F, \psi)
  \;<\;
  \infty,
\end{equation}
\begin{equation}
\label{eq:cif-robust-eigenvector-bound}
  \gamma_{j}^{*}(\mathbf{u}_{r}^{\psi}, F)
  \;\leq\;
  \frac{C(F, \psi)}
       {\min_{s \neq r}|\lambda_{r}^{\psi} - \lambda_{s}^{\psi}|}
  \;<\;
  \infty
\end{equation}
for every $j \in [D]$ and $r \in [D-1]$.  The proposed
cellwise-robust compositional PCA is therefore cellwise B-robust.
\end{corollary}

\begin{proof}
Apply the chain rule for the CIF through the spectral decomposition,
exactly as in the proofs of Theorems~S.10
and~S.11 of the Supplementary Material, but
with the perturbation matrix
$\mathbf{A} = \operatorname{CIF}_{j}(\boldsymbol{\Sigma}^{\psi}, F,
\delta)$ in place of
$(\ln\delta)^{2}\,\mathbf{v}_{j}\mathbf{v}_{j}^{\top}$.

For the eigenvalues,
\[
  \bigl|\operatorname{CIF}_{j}(\lambda_{r}^{\psi}, F, \delta)\bigr|
  \;=\;
  \bigl|\mathbf{u}_{r}^{\psi\top}\,
    \operatorname{CIF}_{j}(\boldsymbol{\Sigma}^{\psi}, F, \delta)\,
    \mathbf{u}_{r}^{\psi}\bigr|
  \;\leq\;
  \bigl\|\operatorname{CIF}_{j}(\boldsymbol{\Sigma}^{\psi}, F, \delta)
  \bigr\|_{\mathrm{op}}
  \;\leq\;
  C(F, \psi),
\]
where the last inequality follows from
Theorem~\ref{thm:cif-robust-covariance} and the bound
$\|\cdot\|_{\mathrm{op}} \leq \|\cdot\|_{F}$.

For the eigenvectors, by~(S.11) of the
Supplementary Material,
\begin{align}
  \bigl\|\operatorname{CIF}_{j}(\mathbf{u}_{r}^{\psi}, F, \delta)\bigr\|
  &\;\leq\;
  \sum_{\substack{s=1 \\ s \neq r}}^{D-1}
  \frac{\bigl|\mathbf{u}_{s}^{\psi\top}\,
    \operatorname{CIF}_{j}(\boldsymbol{\Sigma}^{\psi}, F, \delta)\,
    \mathbf{u}_{r}^{\psi}\bigr|}
       {|\lambda_{r}^{\psi} - \lambda_{s}^{\psi}|}
  \notag\\
  &\;\leq\;
  \frac{(D-2)\,C(F,\psi)}
       {\min_{s \neq r}|\lambda_{r}^{\psi} - \lambda_{s}^{\psi}|},
  \label{eq:pf-robust-evec-bound}
\end{align}
which is finite since the eigenvalue gaps are positive by assumption
and $C(F,\psi)$ is independent of~$\delta$.
\end{proof}

\ifsubmit
\subsection{Comparison: classical vs.\ cellwise-robust estimators}
\label{ssec:comparison}
A side-by-side comparison of the classical and proposed estimators
--- including an explicit statement of the unbounded CIF of classical
PCA and a comparison table --- is given in Section~S.3 of the
Supplementary Material.  Qualitatively: classical estimators of
centre, covariance, and eigenstructure all have unbounded CIF under
cellwise contamination on the simplex; the proposed $\psi$-based
estimator has bounded CIF by
Theorem~\ref{thm:cif-robust-covariance} and is the first cellwise
B-robust PCA for compositional data.
\fi   

\begin{remark}[Section summary]
\label{rem:section5-summary}
The cellwise influence function (CIF) provides a principled
sensitivity measure for compositional estimators under part-specific
contamination.  Classical estimators of center, covariance, and
eigenstructure all have unbounded CIF and are therefore not cellwise
B-robust.  The proposed $\psi$-based estimator achieves bounded CIF
by clipping extreme residuals componentwise, yielding the first
cellwise B-robust PCA for compositional data.
\end{remark}

%% file: refs.bib
@book{Aitchison1986,
  author    = {Aitchison, John},
  title     = {The Statistical Analysis of Compositional Data},
  publisher = {Chapman \& Hall},
  address   = {London},
  year      = {1986},
  doi       = {10.1007/978-94-009-4109-0}
}

@article{Egozcue2003,
  author    = {Egozcue, Juan Jos{\'e} and Pawlowsky-Glahn, Vera and
               Mateu-Figueras, Gl{\`o}ria and Barcel{\'o}-Vidal, Carles},
  title     = {Isometric Logratio Transformations for Compositional Data
               Analysis},
  journal   = {Mathematical Geology},
  volume    = {35},
  number    = {3},
  pages     = {279--300},
  year      = {2003},
  doi       = {10.1023/A:1023818214614}
}

@book{PawlowskyGlahn2015,
  author    = {Pawlowsky-Glahn, Vera and Egozcue, Juan Jos{\'e} and
               Tolosana-Delgado, Raimon},
  title     = {Modeling and Analysis of Compositional Data},
  publisher = {John Wiley \& Sons},
  address   = {Chichester},
  year      = {2015},
  doi       = {10.1002/9781119003144}
}

@book{Filzmoser2018,
  author    = {Filzmoser, Peter and Hron, Karel and Templ, Matthias},
  title     = {Applied Compositional Data Analysis: With Worked Examples
               in {R}},
  publisher = {Springer},
  address   = {Cham},
  year      = {2018},
  doi       = {10.1007/978-3-319-96422-5}
}

@article{Filzmoser2009,
  author    = {Filzmoser, Peter and Hron, Karel and Reimann, Clemens},
  title     = {Principal Component Analysis for Compositional Data with
               Outliers},
  journal   = {Environmetrics},
  volume    = {20},
  number    = {6},
  pages     = {621--632},
  year      = {2009},
  doi       = {10.1002/env.966}
}

@article{Mert2016,
  author    = {Mert, M. Cenk and Filzmoser, Peter and Hron, Karel},
  title     = {Error Propagation in Isometric Log-ratio Coordinates
               for Compositional Data: Theoretical and Practical
               Considerations},
  journal   = {Mathematical Geosciences},
  volume    = {48},
  number    = {8},
  pages     = {941--961},
  year      = {2016},
  doi       = {10.1007/s11004-016-9646-x}
}

@article{Fiserova2011,
  author    = {Fi{\v{s}}erov{\'a}, Eva and Hron, Karel},
  title     = {On the Interpretation of Orthonormal Coordinates for
               Compositional Data},
  journal   = {Mathematical Geosciences},
  volume    = {43},
  number    = {4},
  pages     = {455--468},
  year      = {2011}
}

@article{MartinFernandez2003,
  author    = {Mart{\'\i}n-Fern{\'a}ndez, Josep Antoni and
               Barcel{\'o}-Vidal, Carles and Pawlowsky-Glahn, Vera},
  title     = {Dealing with Zeros and Missing Values in Compositional
               Data Sets Using Nonparametric Imputation},
  journal   = {Mathematical Geology},
  volume    = {35},
  number    = {3},
  pages     = {253--278},
  year      = {2003}
}

@article{MartinFernandez2012,
  author    = {Mart{\'\i}n-Fern{\'a}ndez, Josep Antoni and Hron, Karel
               and Templ, Matthias and Filzmoser, Peter and
               Palarea-Albaladejo, Javier},
  title     = {Model-Based Replacement of Rounded Zeros in Compositional
               Data: Classical and Robust Approaches},
  journal   = {Computational Statistics \& Data Analysis},
  volume    = {56},
  number    = {9},
  pages     = {2688--2704},
  year      = {2012}
}

@article{EgozcuePawlowskyGlahn2019,
  author    = {Egozcue, Juan Jos{\'e} and Pawlowsky-Glahn, Vera},
  title     = {Compositional Data: The Sample Space and Its Structure},
  journal   = {TEST},
  volume    = {28},
  number    = {3},
  pages     = {599--638},
  year      = {2019},
  doi       = {10.1007/s11749-019-00670-6}
}

@article{Greenacre2023,
  author    = {Greenacre, Michael and Grunsky, Eric and Bacon-Shone, John
               and Erb, Ionas and Quinn, Thomas},
  title     = {Aitchison's Compositional Data Analysis 40 Years On:
               A Reappraisal},
  journal   = {Statistical Science},
  volume    = {38},
  number    = {3},
  pages     = {386--410},
  year      = {2023},
  doi       = {10.1214/22-STS880}
}

@article{Gloor2017,
  author    = {Gloor, Gregory B. and Macklaim, Jean M. and
               Pawlowsky-Glahn, Vera and Egozcue, Juan Jos{\'e}},
  title     = {Microbiome Datasets Are Compositional: And This Is Not
               Optional},
  journal   = {Frontiers in Microbiology},
  volume    = {8},
  pages     = {2224},
  year      = {2017},
  doi       = {10.3389/fmicb.2017.02224}
}

@article{ZhangSchluter2024,
  author    = {Zhang, Yiqian and Schluter, Jonas and Zhang, Lu and
               Cao, Xuan and Jenq, Robert R. and Feng, Hao and
               Haines, Jeffrey and Zhang, Lei},
  title     = {Review and Revamp of Compositional Data Transformation:
               A New Framework Combining Proportion Conversion and
               Contrast Transformation},
  journal   = {Computational and Structural Biotechnology Journal},
  volume    = {23},
  pages     = {4088--4107},
  year      = {2024},
  doi       = {10.1016/j.csbj.2024.11.003}
}

@article{Alqallaf2009,
  author    = {Alqallaf, Fatemah and {Van Aelst}, Stefan and
               Yohai, V{\'\i}ctor J. and Zamar, Ruben H.},
  title     = {Propagation of Outliers in Multivariate Data},
  journal   = {The Annals of Statistics},
  volume    = {37},
  number    = {1},
  pages     = {311--331},
  year      = {2009},
  doi       = {10.1214/07-AOS588}
}

@article{Agostinelli2015,
  author    = {Agostinelli, Claudio and Leung, Andy and
               Yohai, V{\'\i}ctor J. and Zamar, Ruben H.},
  title     = {Robust Estimation of Multivariate Location and Scatter
               in the Presence of Cellwise and Casewise Contamination},
  journal   = {TEST},
  volume    = {24},
  number    = {3},
  pages     = {441--461},
  year      = {2015},
  doi       = {10.1007/s11749-015-0450-6}
}

@article{Rousseeuw2018,
  author    = {Rousseeuw, Peter J. and {Van den Bossche}, Wannes},
  title     = {Detecting Deviating Data Cells},
  journal   = {Technometrics},
  volume    = {60},
  number    = {2},
  pages     = {135--145},
  year      = {2018},
  doi       = {10.1080/00401706.2017.1340909}
}

@article{Raymaekers2024,
  author    = {Raymaekers, Jakob and Rousseeuw, Peter J.},
  title     = {The Cellwise Minimum Covariance Determinant Estimator},
  journal   = {Journal of the American Statistical Association},
  volume    = {119},
  number    = {548},
  pages     = {2610--2621},
  year      = {2024},
  doi       = {10.1080/01621459.2023.2267777}
}

@article{Hubert2019,
  author    = {Hubert, Mia and Rousseeuw, Peter J. and
               {Van den Bossche}, Wannes},
  title     = {MacroPCA: An All-in-One {PCA} Method Allowing for
               Missing Values as Well as Cellwise and Rowwise Outliers},
  journal   = {Technometrics},
  volume    = {61},
  number    = {4},
  pages     = {459--473},
  year      = {2019},
  doi       = {10.1080/00401706.2018.1562989}
}

@article{Walach2020,
  author    = {Walach, Jan and Filzmoser, Peter and Kou{\v{r}}il, {\v{S}}t{\v{e}}p{\'a}n and Friedeck{\'y}, David and Adam, Tom{\'a}{\v{s}}},
  title     = {Cellwise Outlier Detection and Biomarker Identification
               in Metabolomics Based on Pairwise Log Ratios},
  journal   = {Journal of Chemometrics},
  volume    = {34},
  number    = {1},
  pages     = {e3182},
  year      = {2020},
  doi       = {10.1002/cem.3182}
}

@article{Stefelova2021,
  author    = {{\v{S}}t{\v{e}}felov{\'a}, Nikola and Alfons, Andreas and
               Palarea-Albaladejo, Javier and Filzmoser, Peter and
               Hron, Karel},
  title     = {Robust Regression with Compositional Covariates
               Including Cellwise Outliers},
  journal   = {Advances in Data Analysis and Classification},
  volume    = {15},
  number    = {4},
  pages     = {869--909},
  year      = {2021},
  doi       = {10.1007/s11634-021-00436-9}
}

@article{Rieser2023,
  author    = {Rieser, Christopher and Fa{\v{c}}evicov{\'a}, Kamila and
               Filzmoser, Peter},
  title     = {Cell-wise Robust Covariance Estimation for
               Compositions, with Application to Geochemical Data},
  journal   = {Journal of Geochemical Exploration},
  volume    = {253},
  pages     = {107299},
  year      = {2023},
  doi       = {10.1016/j.gexplo.2023.107299}
}

@article{FilzmoserGregorich2020,
  author    = {Filzmoser, Peter and Gregorich, Mariella},
  title     = {Multivariate Outlier Detection in Applied Data Analysis:
               Global, Local, Compositional and Cellwise Outliers},
  journal   = {Mathematical Geosciences},
  volume    = {52},
  pages     = {1049--1066},
  year      = {2020},
  doi       = {10.1007/s11004-020-09861-6}
}

@article{CentofantiHubertRousseeuw2024,
  author    = {Centofanti, Fabio and Hubert, Mia and Rousseeuw, Peter J.},
  title     = {Robust Principal Components by Casewise and Cellwise
               Weighting},
  journal   = {arXiv preprint},
  year      = {2024},
  eprint    = {2408.13596},
  archiveprefix = {arXiv},
  primaryclass  = {stat.ME}
}

@article{RaymaekersRousseeuw2024b,
  author    = {Raymaekers, Jakob and Rousseeuw, Peter J.},
  title     = {Challenges of Cellwise Outliers},
  journal   = {Econometrics and Statistics},
  year      = {2024},
  doi       = {10.1016/j.ecosta.2024.02.002}
}

@article{SCRAMBLE2025,
  author    = {Pfeiffer, Pia and Vana-G{\"u}r, Laura and Filzmoser, Peter},
  title     = {Cellwise Robust and Sparse Principal Component Analysis},
  journal   = {Advances in Data Analysis and Classification},
  year      = {2025},
  doi       = {10.1007/s11634-025-00656-3},
  note      = {Introduces the {SCRAMBLE} method; arXiv:2408.15612.}
}

@article{CentofantiHubertRousseeuw2025,
  author    = {Centofanti, Fabio and Hubert, Mia and Rousseeuw, Peter J.},
  title     = {Cellwise and Casewise Robust Covariance in High Dimensions},
  journal   = {arXiv preprint},
  year      = {2025},
  eprint    = {2505.19925},
  archiveprefix = {arXiv},
  primaryclass  = {stat.ME}
}

@article{TarrMullerWeber2015,
  author    = {Tarr, Garth and M{\"u}ller, Samuel and Weber, Neville C.},
  title     = {Robust Estimation of Precision Matrices Under Cellwise
               Contamination},
  journal   = {Computational Statistics \& Data Analysis},
  volume    = {93},
  pages     = {404--420},
  year      = {2016},
  doi       = {10.1016/j.csda.2015.02.005}
}

@article{LohTan2018,
  author    = {Loh, Po-Ling and Tan, Xin Lu},
  title     = {High-Dimensional Robust Precision Matrix Estimation:
               Cellwise Corruption Under $\varepsilon$-Contamination},
  journal   = {Electronic Journal of Statistics},
  volume    = {12},
  number    = {1},
  pages     = {1429--1467},
  year      = {2018},
  doi       = {10.1214/18-EJS1427}
}

@inproceedings{PacreauLounici2023,
  author    = {Pacreau, Gr{\'e}goire and Lounici, Karim},
  title     = {Robust Covariance Estimation with Missing Values and
               Cell-Wise Contamination},
  booktitle = {Advances in Neural Information Processing Systems 36
               (NeurIPS 2023)},
  year      = {2023}
}

@article{SaracenoAgostinelli2021,
  author    = {Saraceno, Giovanni and Agostinelli, Claudio},
  title     = {Robust Multivariate Estimation Based on Statistical Depth
               Filters},
  journal   = {TEST},
  volume    = {30},
  number    = {4},
  pages     = {935--959},
  year      = {2021},
  doi       = {10.1007/s11749-021-00757-z}
}

@article{KatayamaFujisawaDrton2018,
  author    = {Katayama, Shota and Fujisawa, Hironori and Drton, Mathias},
  title     = {Robust and Sparse {G}aussian Graphical Modelling Under
               Cell-Wise Contamination},
  journal   = {Stat},
  volume    = {7},
  number    = {1},
  pages     = {e181},
  year      = {2018},
  doi       = {10.1002/sta4.181}
}

@article{LeungZhangZamar2016,
  author    = {Leung, Andy and Zhang, Hongyang and Zamar, Ruben},
  title     = {Robust Regression Estimation and Inference in the
               Presence of Cellwise and Casewise Contamination},
  journal   = {Computational Statistics \& Data Analysis},
  volume    = {99},
  pages     = {1--11},
  year      = {2016},
  doi       = {10.1016/j.csda.2016.01.004}
}

@article{Zaccaria2024,
  author    = {Zaccaria, Giorgia and Garc{\'\i}a-Escudero, Luis Angel
               and Greselin, Francesca and Mayo-Iscar, Agust{\'\i}n},
  title     = {Cellwise Outlier Detection in Heterogeneous Populations},
  journal   = {Technometrics},
  volume    = {67},
  number    = {4},
  pages     = {643--654},
  year      = {2025},
  doi       = {10.1080/00401706.2025.2497822}
}

@article{Huber1964,
  author    = {Huber, Peter J.},
  title     = {Robust Estimation of a Location Parameter},
  journal   = {The Annals of Mathematical Statistics},
  volume    = {35},
  number    = {1},
  pages     = {73--101},
  year      = {1964},
  doi       = {10.1214/aoms/1177703732}
}

@book{Maronna2019,
  author    = {Maronna, Ricardo A. and Martin, R. Douglas and
               Yohai, V{\'\i}ctor J. and Salibi{\'a}n-Barrera, Mat{\'\i}as},
  title     = {Robust Statistics: Theory and Methods (with {R})},
  edition   = {2nd},
  publisher = {John Wiley \& Sons},
  address   = {Chichester},
  year      = {2019},
  doi       = {10.1002/9781119214656}
}

@article{Rousseeuw1985,
  author    = {Rousseeuw, Peter J.},
  title     = {Multivariate Estimation with High Breakdown Point},
  journal   = {Mathematical Statistics and Applications},
  volume    = {B},
  pages     = {283--297},
  year      = {1985},
  publisher = {Reidel, Dordrecht}
}

@book{Hampel1986,
  author    = {Hampel, Frank R. and Ronchetti, Elvezio M. and Rousseeuw, Peter J. and Stahel, Werner A.},
  title     = {Robust Statistics: The Approach Based on Influence Functions},
  publisher = {Wiley},
  address   = {New York},
  year      = {1986}
}

@article{Hampel1974,
  author    = {Hampel, Frank R.},
  title     = {The Influence Curve and Its Role in Robust Estimation},
  journal   = {Journal of the American Statistical Association},
  volume    = {69},
  number    = {346},
  pages     = {383--393},
  year      = {1974},
  doi       = {10.1080/01621459.1974.10482962}
}

@incollection{Donoho1983,
  author    = {Donoho, David L. and Huber, Peter J.},
  title     = {The Notion of Breakdown Point},
  booktitle = {A Festschrift for Erich L. Lehmann},
  pages     = {157--184},
  year      = {1983},
  publisher = {Wadsworth},
  address   = {Belmont, CA}
}

@article{LopuhaaRousseeuw1991,
  author    = {Lopuha{\"a}, Hendrik P. and Rousseeuw, Peter J.},
  title     = {Breakdown Points of Affine Equivariant Estimators of
               Multivariate Location and Covariance Matrices},
  journal   = {The Annals of Statistics},
  volume    = {19},
  number    = {1},
  pages     = {229--248},
  year      = {1991},
  doi       = {10.1214/aos/1176347978}
}

@misc{Templ2026cellPcaCoDa,
  author    = {Templ, Matthias},
  title     = {{cellPcaCoDa}: Cellwise-Robust Principal Components for
               Compositional Data},
  year      = {2026},
  note      = {Companion manuscript targeting ADAC (Advances in Data
               Analysis and Classification).  A preprint will be
               deposited on arXiv prior to submission; the arXiv
               identifier will be added in the revised version.
               Develops the \texttt{cellPcaCoDa} estimator, the
               cellwise-robust PCA breakdown theorem discussed in the
               present paper, an $R = 1000$ simulation study, and a
               GEMAS geochemical application.},
  howpublished = {Manuscript}
}
